\documentclass[acmsmall]{acmart}

\acmJournal{PACMPL}
\acmVolume{1}
\acmNumber{CONF} 
\acmArticle{1}
\acmYear{2018}
\acmMonth{1}
\acmDOI{} 
\startPage{1}

\setcopyright{none}
\usepackage{booktabs}
\usepackage{subcaption}
\usepackage{xspace}
\usepackage{listings,lstautogobble}
\usepackage{tikz}
\usepackage{caption}
\usepackage{pict2e}
\usepackage{algorithm}
\usepackage{enumitem}
\usepackage{multirow}
\usepackage{adjustbox}
\usepackage{mathtools}	
\usepackage{mdframed}
\usepackage{syntax}
\usepackage{xcolor}
\usepackage[skins]{tcolorbox}
\usepackage{array}
\usepackage{pgfplots}
\pgfplotsset{
	compat=1.4,
	small,
	legend style={
		at={(0.99,0.99)},
		anchor=north east,
		font=\bfseries,
	},
	label style={font=\sffamily\small\bfseries}
}
\usepackage[normalem]{ulem}
\let\UnderScore_

\newcommand\redout{\bgroup\markoverwith
	{\textcolor{red}{\rule[.45ex]{1.5pt}{1.pt}}}\ULon}

\global\mdfdefinestyle{default}{%
	usetwoside=false,%
	linecolor=black,linewidth=1pt,%
	innerleftmargin=5pt,innerrightmargin=5,%
	everyline=true
}

\newcommand{\tabincell}[2]{\begin{tabular}{@{}#1@{}}#2\end{tabular}}

\theoremstyle{definition}
\newtheorem{definition}{Definition}[section]
\newtheorem{theorem}{Theorem}[section]

\makeatletter
\newenvironment{btHighlight}[1][]
{\begingroup\tikzset{bt@Highlight@par/.style={#1}}\begin{lrbox}{\@tempboxa}}
	{\end{lrbox}\bt@HL@box[bt@Highlight@par]{\@tempboxa}\endgroup}

\newcommand\btHL[1][]{%
	\begin{btHighlight}[#1]\bgroup\aftergroup\bt@HL@endenv%
	}
	\def\bt@HL@endenv{%
	\end{btHighlight}%
	\egroup
}
\newcommand{\bt@HL@box}[2][]{%
	\tikz[#1]{%
		\pgfpathrectangle{\pgfpoint{1pt}{0pt}}{\pgfpoint{\wd #2}{\ht #2}}%
		\pgfusepath{use as bounding box}%
		\node[anchor=base west, fill=orange!25,outer sep=.5pt,inner xsep=0.5pt, inner ysep=0.15pt, rounded corners=1pt, minimum height=\ht\strutbox-.1pt,#1]{\raisebox{.01pt}{\strut}\strut\usebox{#2}};
	}%
}
\makeatother

\definecolor{add-line}{rgb}{0.84, 0.89, 0.97}
\definecolor{add-word}{rgb}{0.5, 0.7, 0.98}
\definecolor{ForestGreen}{RGB}{63,147,88}
\definecolor{remove-line}{rgb}{1.0, 0.93, 0.73}
\definecolor{remove-word}{rgb}{1.0, 0.8, 0.2}
\definecolor{Gray}{gray}{0.9}

\lstdefinestyle{mystyle}{  
	frame=single, 
	framexleftmargin=0pt,
	commentstyle=\color{ForestGreen},
	keywordstyle=\color{blue}\bfseries,
	numberstyle=\tiny\color{gray},
	stringstyle=\color{purple},
	basicstyle=\tiny\ttfamily\bfseries,
	breakatwhitespace=false,         
	breaklines=false,                 
	captionpos=b,                    
	keepspaces=true,     
	numbers=none,                    
	numbersep=4pt,               
	showspaces=false,                
	showstringspaces=false,
	showtabs=false,                  
	tabsize=2,  
	language=Java,  
	escapechar=\%,  
	moredelim=**[is][{\btHL[fill=red!40]}]{@}{@},
	moredelim=**[is][{\btHL[fill=remove-line]}]{|*}{*|},
	moredelim=**[is][{\btHL[fill=remove-word]}]{|^}{^|},
	moredelim=**[is][{\btHL[fill=add-line]}]{||^}{^||},	
	moredelim=**[is][{\btHL[fill=add-word]}]{||*}{*||},	
}

\usepackage{algpseudocode}
\renewcommand*\Call[2]{\textproc{#1}(#2)}
\algnewcommand{\LineComment}[1]{\Statex #1}
\algrenewcommand\algorithmicindent{1.0em}

\newcommand*\circled[1]{\tikz[baseline=(char.base)]{
            \node[shape=circle,draw,inner sep=0.1pt] (char) {#1};}}

\usetikzlibrary{shapes,positioning,arrows.meta, calc, decorations.pathreplacing, matrix}
\usepackage{colortbl}
\definecolor{Gray}{gray}{0.85}
\tikzset{draw half paths/.style 2 args={%
  decoration={show path construction,
    lineto code={
      \draw [#1] (\tikzinputsegmentfirst) -- 
         ($(\tikzinputsegmentfirst)!0.5!(\tikzinputsegmentlast)$);
      \draw [#2] ($(\tikzinputsegmentfirst)!0.5!(\tikzinputsegmentlast)$)
        -- (\tikzinputsegmentlast);
    }
  }, decorate
}}
\usepackage{tikz-qtree}

\newcommand{\textmtte}[1]{{\fontsize{9}{11}\fontfamily{txtt}\selectfont #1}}

\renewcommand{\texttt}[1]{\textmtte{#1}}

\newcommand{\etal}{\hbox{et al.}\xspace}
\newcommand{\eg}{\hbox{\emph{e.g.}}\xspace}
\newcommand{\ie}{\hbox{\emph{i.e.}}\xspace}

\newcommand{\wrt}{\hbox{\emph{w.r.t.}}\xspace}
\newcommand{\etc}{\hbox{\emph{etc.}}\xspace}
\newcommand{\resp}{\hbox{\emph{resp.}}\xspace}

\usepackage{xr}

\newcommand{\tool}{\textsf{HuoYan}\xspace}
\newcommand{\method}{\textsf{WheaCha}\xspace}
\newcommand{\whet}{\textit{wheat}\xspace}
\newcommand{\Whet}{\textit{Wheat}\xspace}
\newcommand{\chaf}{\textit{chaff}\xspace}
\newcommand{\motimodel}{seq-GNN\xspace}

\begin{document}
\title[Short Title]{\method: A Method for Explaining the Predictions of Models of Code}

\author{Yu Wang}
\affiliation{
	\department{State Key Laboratory for Novel Software Technology Department of Computer Science and Technology}              
	\institution{Nanjing University}            
	\city{Nanjing}
	\state{Jiangsu}
	\postcode{210023}
	\country{China}                    
}
\email{yuwang\_cs@smail.nju.edu.cn} 

\author{Ke Wang}
\affiliation{
	\institution{Visa Research}            
	\city{Palo Alto}
	\state{CA}
	\country{USA}    
}
\email{kewang@visa.com}          

\author{Linzhang Wang}
\affiliation{
	\department{State Key Laboratory for Novel Software Technology Department of Computer Science and Technology}              
	\institution{Nanjing University}            
	\city{Nanjing}
	\state{Jiangsu}
	\postcode{210023}
	\country{China}                    
}
\email{lzwang@nju.edu.cn}          

\begin{abstract}


Attribution methods have emerged as a popular approach to interpreting model predictions based on the relevance of input features. Although the feature importance ranking can provide insights of how models arrive at a prediction from a raw input, they do not give a clear-cut definition of the key features models use for the prediction. 
In this paper, we present a new method, called \method, for explaining the predictions of code models. 
Although \method employs the same mechanism of tracing model predictions back to the input features, it differs from all existing attribution methods in crucial ways. Specifically, \method divides an input program into ``\whet'' (\ie, the defining features that are the reason for which models predict the label that they predict) and the rest ``\chaf''
for any prediction of a learned code model.
We realize \method in a tool, \tool, and use it to explain four prominent code models: code2vec, 
\motimodel, GGNN, and CodeBERT. Results show (1) \tool is efficient --- taking on average under twenty seconds to compute the \whet for an input program in an end-to-end fashion (\ie, including model prediction time); (2) the \whet that all models use to predict input programs is made of simple syntactic or even lexical properties (\ie, identifier names);
(3) Based on \whet, we present a novel approach to explaining the predictions of code models through the lens of training data. 


\end{abstract}

\maketitle
	
\section{Introduction}

Riding on the major breakthroughs in deep learning methods along with the ever-increasing public datasets and computation power, modern machine learning (ML) models, such as neural networks, have been increasingly applied to solve programming language tasks, and achieved remarkable success in a variety of problem domains: method name prediction~\cite{Alon2019code2vec, alon2018code2seq,fernandes2018structured}, program repair~\cite{chen2019sequencer,dinella2019hoppity}, and program verification~\cite{Jianan,Xujie3}.
Despite those accomplishments, neural networks mostly operate in a black-box manner, making it difficult to get insight into their internal mechanism of work. This lack of transparency has become an impediment to the use of learning-based program analysis tools, especially in security-critical settings (\eg, malware detection) as the degree to which their predictions can be trusted is rather unclear.
From a scientific standpoint, improving the transparency of neural models is also essential to the soundness of science.
Because their inability to provide explanations for their decisions not only weakens the validity but also hinders the openness of 
scientific discovery.

\vspace*{3pt}
\noindent
\textbf{\textit{A Review of Attribution Methods.}}\,
In the past few years, significant progress has been made in explaining the predictions of ML models. A prominent class of explainability techniques, called attribution methods, has sparked a lot of interest in the ML community. The idea is to assign an attribution score to each input feature \wrt a particular output of a network. In general, attribution methods can be classified into two categories: perturbation-based and backpropagation-based. The former refers to those that make perturbations to input features and observe the impact on later neurons in the network. \citet{zeiler2011adaptive} is a typical example in the image classification domain. In particular, it occludes different segments of an input image and visualizes the change in the activation of later layers. The strength of perturbation-based methods lies in their visibility, that is, one can directly visualize the marginal effect of any input feature via perturbation. Backpropagation-based methods exceed the perturbation-based in terms of efficiency. In particular, they can compute the attributions for all input features in a single forward and backward pass through the network.
As an early attempt, \citet{Simonyan14deepinside} proposed using the gradient of the output \wrt pixels of an input image to compute a saliency map of the image. A key drawback of this approach is its inability to address the saturation problem, namely, gradients can be tiny at certain inputs that may yield significant activations within the network~\cite{pmlrv70shrikumar17a,pmlrv9glorot10a}. 
Integrated Gradients \cite{Sundararajan2017}, a well-known explainability technique, offers a solution. Instead of computing the gradients at only the current value of the input,~\citet{Sundararajan2017} propose to integrate the gradients as the inputs are scaled up from a pre-set starting value to their current value. We defer a detailed survey on the attribution methods to Section~\ref{sec:rel}.

\begin{figure}[t!]
	\captionsetup{skip=5pt}

	\centering
	\begin{subfigure}{0.372\textwidth}
		\lstset{style=mystyle}
\begin{lstlisting}[basicstyle=\scriptsize\ttfamily\bfseries,  morekeywords={var, public, String, Object}]
void f(int position) {
  List<Obj> mItems = retQueue();
  if (position > mItems.size()) 
    return;
  mItems.add(position, genItem());
  notifyItemInserted(position);
  log("Add item;");
}
\end{lstlisting}  
	\end{subfigure}   
	\,\,
	\begin{subfigure}{0.1815\textwidth}
		\centering
		\raisebox{1.2mm}{\includegraphics[width=\textwidth]{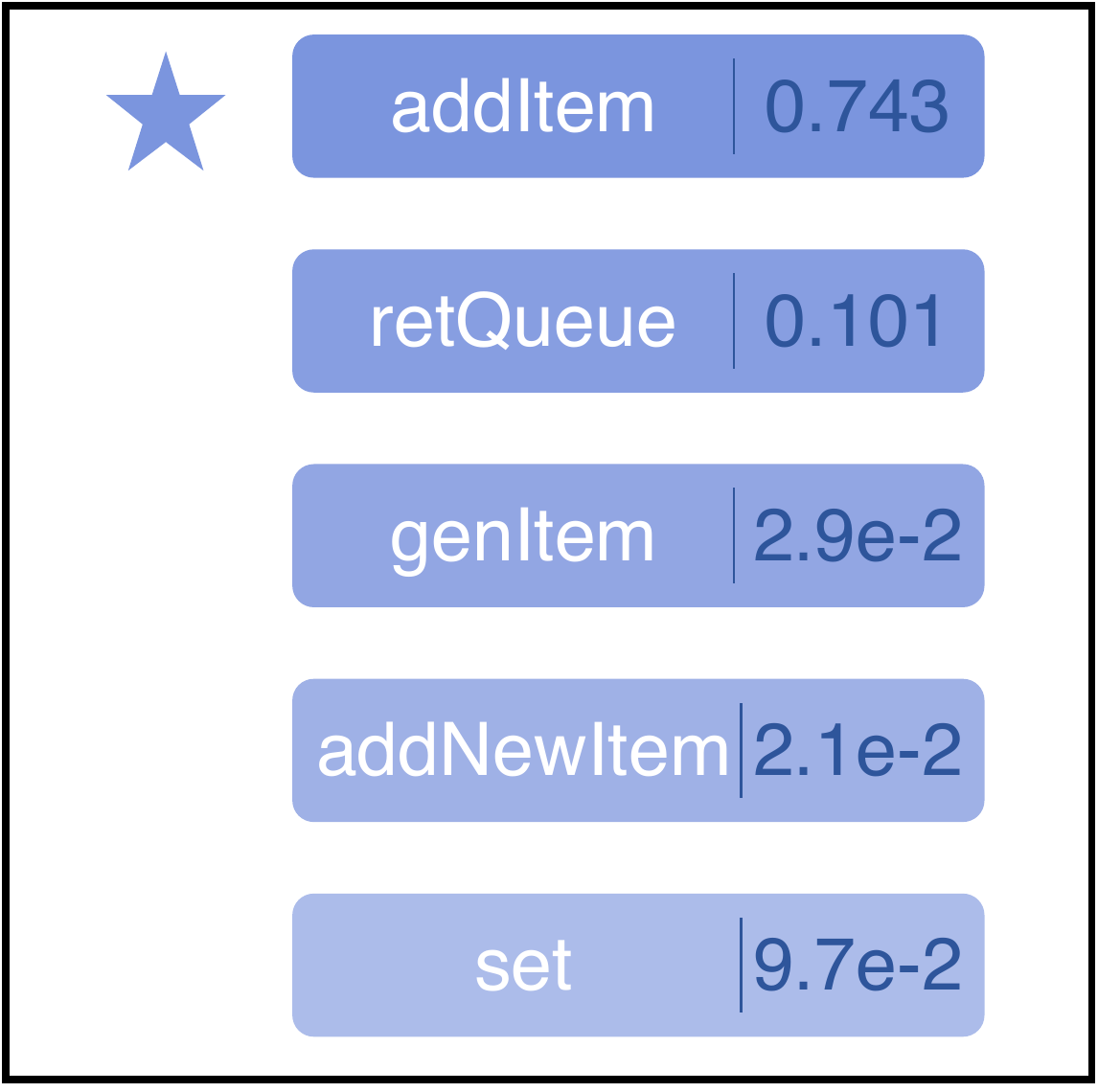}}
	\end{subfigure}
\caption{An example method whose name is correctly predicted by \motimodel~\cite{fernandes2018structured}. The top-five predictions made by \motimodel are shown on the right. Throughout the paper, we will use this program as our running example.}
	\label{fig:met}
\vspace*{-5pt}
\end{figure}

\vspace*{3pt}
\noindent
\textbf{\textit{Insights of \method.}}\, 
While the attribution methods can be readily applied to explain the predictions of learning-based program analysis tools, such as producing a ranking on the importance of different parts of the input program,
they do not give a clear-cut definition of the critical features that models use for that prediction. Considering the ideal case in which there might exist a ``sweet-spot'' in the ranking that separates the critical features from the rest, then how to determine the ``sweet-spot'' is rather unclear. To address this shortcoming, this paper presents a new explanation method, \method (\textit{\underline{Whea}t} and \textit{\underline{Cha}ff}), for interpreting the predictions of models of code, a broad class of models that predict the properties of source code. Figure~\ref{fig:met} shows one such task called method name prediction. We will use the program as our running example throughout the paper. Unlike the existing attribution methods that assign a relevance score to each input feature, \method classifies an entire input into two kinds of features: the defining features, \whet, that are the reason models predict the label that they predict, and the remaining features, \chaf. 
A natural question arises: how do we define \whet for an input program given a particular model prediction? Our insight is to observe how models react to a pair of complementary prediction samples derived from the original input. Technically, we formulate the following two constraints to quantify the influence of \whet. That is, the very same model must (1) preserve its prediction when the original input becomes \whet; and (2) change its prediction when the original input becomes \chaf. 

Below we illustrate these two constraints with our running example. Figure~\ref{fig:exasuf} shows that the expression, \texttt{mItems.add();}, alone preserves the prediction, \texttt{addItem}, \motimodel makes for the original input, thus the first constraint is satisfied; Figure~\ref{fig:exanec} shows that removing \texttt{mItems.add();} from the input program changes \motimodel's prediction to \texttt{retQueue}, therefore the second constraint is also met. Compared to the existing attribution methods, \method's explanations allow end users to know precisely and definitively the features that models use for a prediction they make, and in turn evaluate the trustworthiness of the prediction more accurately. Considering the running example, \method lets end users know exactly \texttt{mItems.add();} are the features \motimodel uses to predict the label \texttt{addItem}, whereas, a ranking of the attribution score for each feature would have been vague (\ie, unclear how small an attribution score indicates feature irrelevance) and redundant (\ie, unnecessary to compare the attribution scores of features in the \whet, such as \texttt{mItems} \textit{vs.} \texttt{add}, or those in \chaf such as \texttt{retQueue} \textit{vs.} \texttt{log}).



\begin{figure}[tb!]
	\centering
	\captionsetup{skip=5pt}
	\begin{subfigure}{0.41\textwidth}
	\centering
		\begin{subfigure}{0.45\textwidth}
		\lstset{style=mystyle}
\begin{lstlisting}[basicstyle=\scriptsize\ttfamily\bfseries, morekeywords={var, public, String, Object}]
void addItem(int 
    position) {


  mItems.add();
  
}
\end{lstlisting}
		\end{subfigure}
		\!\, 
		\begin{subfigure}{2.24cm} 
			\centering
            \raisebox{1.15mm}{\includegraphics[width=\textwidth]{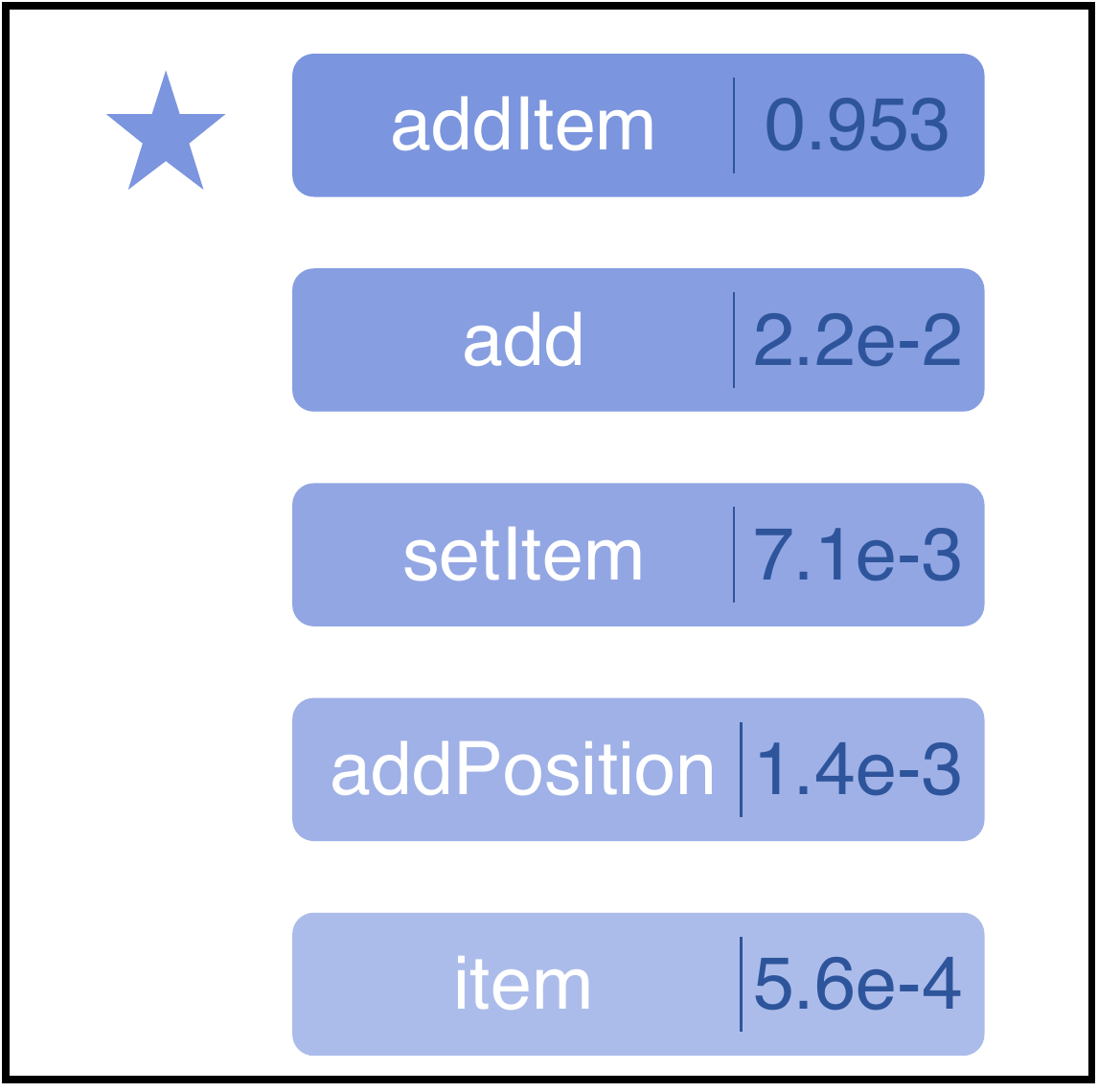}}
		\end{subfigure}	
		\caption{} 
		\label{fig:exasuf}
	\end{subfigure}  
	\,\, 
	\begin{subfigure}{0.54\textwidth}
	\centering
		\begin{subfigure}{0.66\textwidth} 
		\lstset{style=mystyle}
\begin{lstlisting}[basicstyle=\scriptsize\ttfamily\bfseries,morekeywords={var, public, String, Object}]
void retQueue(int position) {
  List<Obj> mItems = retQueue();
  if (position > mItems.size()) 
    return;
  %\sout{mItems.add(}%position; genItem()%\sout{)}%;
  notifyItemInserted(position);
  log("Add item;");}
\end{lstlisting} 
		\end{subfigure}
				\!\,
		\begin{subfigure}{2.24cm} 
			\centering
            \raisebox{1.15mm}{\includegraphics[width=\textwidth]{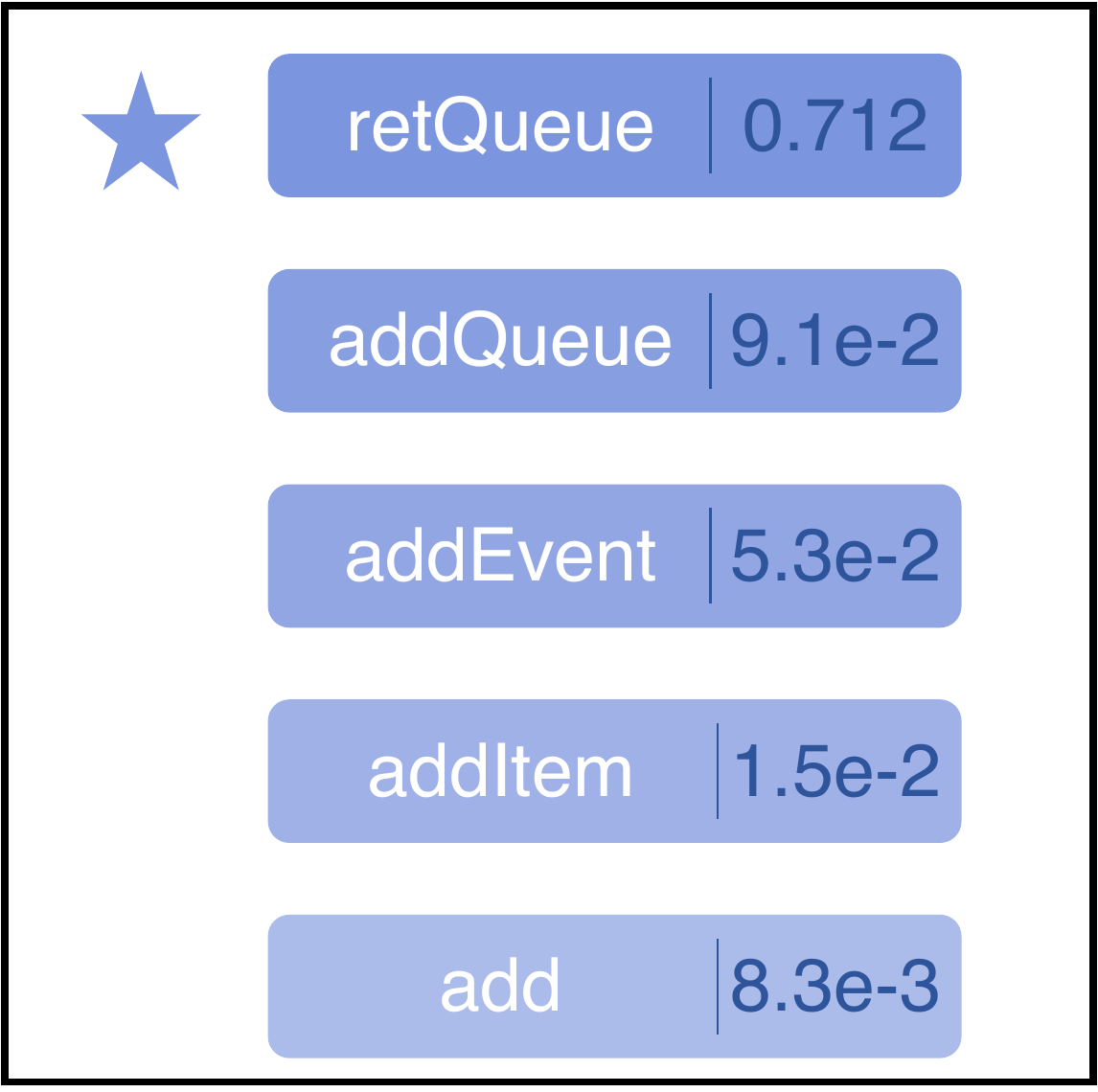}}
		\end{subfigure}	
		\caption{}
		\label{fig:exanec}
	\end{subfigure} 
	\caption{A statement in the example program that satisfies the two constraints.\protect\footnotemark }
\vspace{-5pt}
\end{figure}
\footnotetext{The reason that \whet does not compile is that \motimodel, like most code models, does not require input programs to compile. In fact, \method is general, capable of explaining models with or without compilation requirements.}

Using \method's explanations as the ground-truth,
we find that the ranking produced by some of the most prominent attribution methods routinely underestimates the importance of \whet. Take \motimodel's prediction for the example program as an instance, Figure~\ref{fig:ig} shows that none of the top-five important tokens computed by Integrated Gradients~\cite{Sundararajan2017} is part of the \whet, the underlined expression in the program. This finding rebuts the aforementioned ``sweet-spot'' hypothesis, and confirms that existing attribution methods are unsuitable for \whet detection. 

\begin{figure}[h!]
	\centering 
	\captionsetup{skip=5pt}
	\hspace{1pt} 
	\begin{subfigure}{0.372\textwidth} 
		\lstset{style=mystyle} 
\begin{lstlisting}[basicstyle=\scriptsize\ttfamily\bfseries,morekeywords={var, public, String, Object}]
void addItem(int position) {
  @List@<Obj> mItems = retQueue();
  if (@position@ > @mItems@.@size@()) 
    return;
  %\underbar{mItems.add(}%position, genItem()%\underbar{)}%;
  notifyItemInserted(position);
  @log@("Add item;");
}
\end{lstlisting} 
	\end{subfigure} 
	\,\, 
	\begin{subfigure}{0.1442\textwidth} 
		\raisebox{1.2mm}{\includegraphics[width=\textwidth]{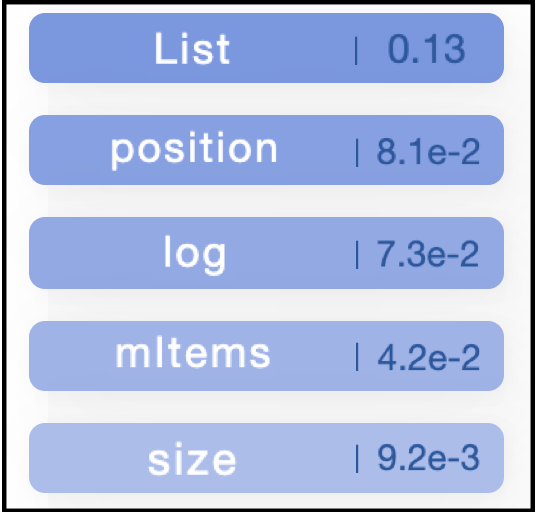}} 
	\end{subfigure} 
	
	\caption{Regarding \motimodel's prediction for the running example, the left shows the five most important features (highlighted within the shadow boxes) computed by Integrated Gradients and the \whet (underlined). The ranking of the attribution scores of the five most important features is shown on the right.}
	\label{fig:ig}
\vspace{-5pt}
\end{figure} 


Identifying the \whet from an input program is a challenging task. 
With a token-based representation, a program composed of $n$ tokens will yield a search space of $2^n$ candidates (\ie, every token may or may not be part of the \whet). Even after taking into account the syntactic and semantic constraints, the search space will not shrink dramatically. As a result, a brute-force search would be computationally intractable for any non-trivial program.

Our solution is based on the finding from~\citet{RABIN2021106552} --- 
few simple code edits, albeit preserving the property of interest for input programs, frequently cause models to alter their predictions --- we hypothesize that models heavily utilize small, local program features for predictions.
To confirm our hypothesis, we conduct a preliminary study in which we test how models respond to new programs obtained via systematic reduction of the input programs.
Quite surprisingly, we find that almost always an input program can be reduced to very few statements (\ie, $\leq$ 3) for which models make the same prediction as they do for the original program. This observation indicates that a small code fragment may already contain the \whet that models need for making the correct prediction.
Therefore, a more efficient approach is to first detect such code fragments, and then locate the \whet within them. At the technical level, we present ``\textit{Reduce} and \textit{Mutate}'', a coarse-to-fine method for identifying the \whet that code models use for prediction. Given a program $P$ for which a model $M$ makes a prediction. First, we find a global minimum program $\bar{P}$ for ${P}$ that satisfies the first two above-mentioned constraints (\textit{coarse-grained} search).
Second, we mutate the expressions 
in $\bar{P}$'s statements to pinpoint the program properties $\tilde{P}$ that led to the satisfaction of the two constraints (\textit{fine-grained} search). Although it looks like, at this point, we have found the \whet $\tilde{P}$ that $M$ uses to predict $P$, 
a core issue must be resolved: are $M$'s predictions for all intermediate programs that \method uses to query $M$ 
(\eg, $\bar{P}$, $\tilde{P}$, \etc) still valid given the seemingly distribution shift between intermediate programs and original programs on which $M$ is trained (\eg, $P$). Inspired by ROAR~\cite{ROAR}, we present a similar approach to tackle the potential out-of-distribution issue. The key idea is to retrain $M$ with additional data that resembles the intermediate programs. If \method finds the same \whet $\tilde{P}$ --- which are now clearly in distribution thanks to the retraining of $M$ --- that retrained $M$ uses to predict $P$, we have confirmed that $\tilde{P}$ is indeed the \whet of $P$ that $M$ uses for prediction.

We realize our approach in a tool, called \tool, and use it to evaluate four pioneering code models: code2vec~\cite{Alon2019code2vec}, \motimodel~\cite{fernandes2018structured}, GGNN~\cite{allamanis2018learning}, CodeBERT~\cite{feng-etal-2020-codebert}. We choose them not only because they are prominent representatives of models of code, but more importantly, they reflect a wide variety of both model architectures and downstream tasks so that the effectiveness and generality of \tool can be thoroughly tested. Our evaluation
results show (1) first, \tool is efficient --- taking on average less than twenty seconds to compute the \whet for each evaluated model; (2) the \whet that all evaluated models use for prediction are simple as they never exceed fifteen tokens, many times down to the name of a single variable; (3) After training models on programs coming from a similar distribution to those with which \method queried the models before, \method almost always finds the same \whet that original and retrained models use for prediction. This is strong evidence that \method's approach is valid in light of the out-of-distribution issues; 
(4) Integrated Gradients and SHAP~\cite{LundbergNIPS2017}, two among the most well-known attribution methods, do not precisely identify the \whet because those they assign higher scores frequently miss out on the \whet.


As an example use of \whet, we present a novel approach to explaining the predictions of code models. At the high-level, our approach answers a more fundamental question: which programs in the training set are most responsible for a model prediction? Ultimately, the weights of a model are derived from the training data. Hence our approach identifies the root cause for a prediction rather than interpreting the internal states of a model after its weights have been learned. At the technical level, given an unseen program $P$ for which a model $M$ predicts a label $L$, we rank $M$'s training programs with the label $L$ based on their distance to $P$, measured using their respective feature properties, and then present the top-$k$ closest programs as the explanations for the prediction $M$ made for $P$.

This paper makes the following contributions:

\begin{itemize}
	\item A method for explaining predictions of models of code.

	\item A definition of the defining features (\ie, \whet), the reason for which code models predict the label that they predict.
	
	\item A method for identifying the \whet that code models use for predictions. 

	\item An implementation, \tool, which we use to evaluate code2vec, \motimodel, GGNN, and CodeBERT. Results show that (1) \tool is efficient: taking on average less than twenty seconds to compute the \whet; (2) all models use simple syntactic or even lexical properties for prediction; (3) Integrated Gradients and SHAP, two attribution methods do not precisely identify the \whet for any evaluated model.
	
	\item An approach to explaining predictions of code models through the lens of training data.
\end{itemize}

The rest of the paper is structured as follows. Section~\ref{sec:form} gives the formalization. Section~\ref{sec:over} describes in detail our method including the discussion on the out-of-distribution issues and 
an example application of \whet. Next, we present our evaluation of \tool (Section~\ref{sec:eva}). Finally, we survey related work (Section~\ref{sec:rel}) and conclude (Section~\ref{sec:conc}).

\section{Formalization}
\label{sec:form}


We consider a machine learning model $M$ as a program (with semantics $\llbracket M \rrbracket$) which executes on an input (\eg, an image,
%
a document, a piece of code, \etc) and returns a prediction as output. Given a prediction $L$ that $M$ makes for an input program $P$
(\ie, $\llbracket M \rrbracket \big(P\big) = L$), the problem \method aims to solve is to identify the \whet $M$ uses for this prediction. Below, we give a formal definition of \whet, which is at the core of \method. 

\begin{definition}(\textit{\Whet})\label{def:keyfea}
The \whet that $M$  uses for predicting the label of $P$ is a set of statements $\tilde{P}$ such that (1) $\tilde{P}$ is a \textit{constituent} of $P$: $\tilde{P}$'s token sequence, denoted by $(t_{n}^{\tilde{P}})_{n\in\mathbb{N}}$, is a \textit{subsequence} of $P$'s denoted by $(t_{m}^{P})_{m\in\mathbb{N}}$. Formally, $(t_{n}^{\tilde{P}})_{n\in\mathbb{N}} = (t_{m_{k}}^{P})_{k\in\mathbb{N}}$ where $(m_k)_{k\in\mathbb{N}}$ is a strictly increasing sequence of positive integers. 
(2) $\tilde{P}$ is \textit{sufficient}: $M$ makes the same prediction for $\tilde{P}$ as it does for $P$ (\ie, $\llbracket M \rrbracket \big(P\big) = \llbracket M \rrbracket \big(\tilde{P}\big) = L$); 
(3) $\tilde{P}$ is \textit{necessary}: $M$ makes a different prediction for $P \setminus \tilde{P}$ than it does for $P$ (\ie, $\llbracket M \rrbracket \big(P \setminus \tilde{P}\big) \neq \llbracket M \rrbracket \big(P\big)$) where $P \setminus \tilde{P}$ denotes the operation that subtracts program $\tilde{P}$ from $P$ (Definition~\ref{def:subtra}).
Finally (4) $\tilde{P}$ is \textit{minimum}: there does not exist $\tilde{P}^\prime$ such that $\tilde{P}^\prime$ satisfies the above three requirements, and $\lvert\tilde{P}^\prime\rvert < \lvert\tilde{P}\rvert$ where $\lvert \cdot \rvert$ denotes the length of a sequence.
\end{definition}

Here we discuss the intuition behind Definition~\ref{def:keyfea}. First, the \textit{constituent} requirement captures an obvious intuition, that is, 
the \whet must be part of an input program. Second, the \textit{sufficient} requirement is also quite intuitive --- as the \whet, they must single-handedly lead models to predict the same label as before without the rest of the input. However, the satisfaction of \textit{sufficient} requirement alone does not suffice features to be the \whet. Consider the following example in Figure~\ref{fig:notnec}. Even though Figure~\ref{fig:notnecel} shows the statement, \texttt{log("Add item;");}, manages to preserve the prediction, \texttt{addItem}, the model makes for the original input (\ie, the \textit{sufficient} requirement is satisfied), it is \textit{not} the reason for which \motimodel makes this prediction because removing it from the input program barely influences \motimodel. As displayed in Figure~\ref{fig:notnecer}, the probabilities at which \motimodel predicts the top-five labels remain almost unchanged. Clearly, this prediction result suggests that \motimodel does not even need \texttt{log("Add item;");} let alone use it as the \whet to predict the example program. This is precisely the reason that~\citet{rabin2021understanding}, which seeks exclusively the sufficient features as model interpretations, is flawed. In fact, our large-scale study shows that removing the features discovered by~\citet{rabin2021understanding} from input programs always never makes models alter their original predictions.

To address this issue, we design the \textit{necessary} requirement: the removal of the \whet from the input program, which we call the \textit{subtraction} operation (Definition~\ref{def:subtra}), must lead models to predict different labels than they did for the original input. 

\begin{figure}[tb!]
	\centering
	\captionsetup{skip=5pt}

	\begin{subfigure}{0.25\textwidth}
	\centering
		\lstset{style=mystyle}
\begin{lstlisting}[basicstyle=\scriptsize\ttfamily\bfseries,morekeywords={var, public, String, Object},xleftmargin=.05\textwidth, xrightmargin=.05\textwidth]
void addItem(int 
        position) {


  log("Add item;");


}
\end{lstlisting} 
		\caption{A statement that satisfies the \textit{sufficient} requirement. }
		\label{fig:notnecel}
	\end{subfigure}
	\,
	\begin{subfigure}{0.6\textwidth}
		\centering
		\begin{subfigure}{0.65\textwidth}
			\lstset{style=mystyle}
\begin{lstlisting}[basicstyle=\scriptsize\ttfamily\bfseries, morekeywords={var, public, String, Object}]
void addItem(int position) {
  List<Obj> mItems = retQueue();
  if (position > mItems.size()) 
    return;
  mItems.add(position, genItem());
  notifyItemInserted(position);
  %\sout{log("Add item;");}% 
}
\end{lstlisting}
		\end{subfigure}
		\!\, 
		\begin{subfigure}{0.303\textwidth}
			\centering
            \raisebox{1.15mm}{\includegraphics[width=\textwidth]{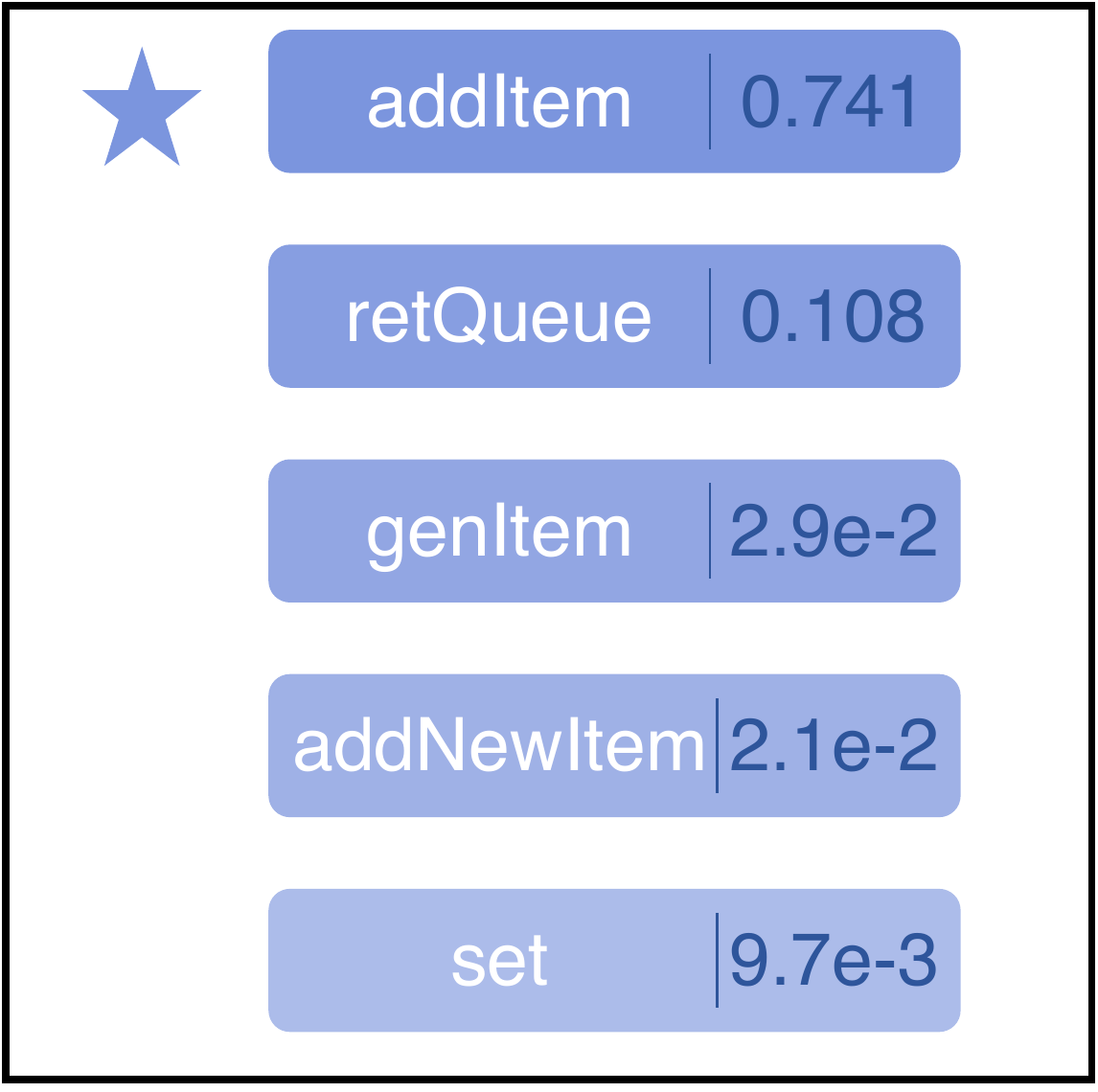}}
		\end{subfigure}	
		\caption{The same statement does not satisfy the \textit{necessary} requirement. \motimodel's prediction probabilities of the top-five labels remain almost unchanged.}
		\label{fig:notnecer}
	\end{subfigure} 	
	\caption{A statement that satisfies the \textit{sufficient} but not \textit{necessary} requirement.}
	\label{fig:notnec}
\vspace{-5pt}
\end{figure} 
 
%

\begin{definition}(\textit{Subtraction})\label{def:subtra}
	Given a program $P$ and a set of statements $\hat{P}$, $P \setminus \hat{P}$ means for each statement $S$ in $\hat{P}$, first locate the subtree, which is equivalent to the Abstract Syntax Tree (AST) of $S$, from the AST of $P$; then remove the located subtree from the AST of $P$. Finally, serialize the resultant AST back to source code.
\end{definition}

Using only the \textit{constituent}, \textit{sufficient}, and \textit{necessary} requirement,
every input program is the \whet for itself --- an uninteresting explanation for any model
prediction. To address this issue, we impose the last constraint to ensure the precision of \whet. 
Specifically, we argue that $\tilde{P}$ should be the globally minimum sequence of tokens since it's the most precise --- covering the pattern that models have learned with the least amount of features (\ie, having the highest signal-to-noise ratio).


%

%

Given Definition~\ref{def:keyfea}, it is obvious that the \whet always exists for any input program for which models make a prediction. We left the proof to the supplemental material. 
\begin{theorem}[Existence of \Whet]\label{the:exis}
	Given a prediction $L$ that $M$ makes for an input program $P$, the \whet $\tilde{P}$ that models use to predict the label of $P$ always exists.
\end{theorem}
%

\section{Methodology}
\label{sec:over}

In this section, we first briefly describe a crucial weakness of models of code, which motivates a key idea of finding the \whet. Then, we give a detailed presentation of
\textit{Reduce} and \textit{Mutate}, a simple yet efficient method for identifying the defining
features. Next, we address the out-of-distribution issues, a potential validity concern about \method's approach. Finally, we present an application of the \whet. 

\subsection{Background}

While deep neural networks have been gaining increasing levels of interest in programming languages research,~\citet{RABIN2021106552} cautioned that some eminent code models are surprisingly unstable with their predictions. Simple, natural, and semantically-preserving transformations frequently cause models to alter their predictions. Here we use an example to demonstrate their finding. Figure~\ref{fig:before} depicts the original method which is correctly predicted by code2vec to be \textit{factorial}; Figure~\ref{fig:after} depicts the transformed method, albeit semantically equivalent, is totally mishandled. None of the top-five predictions even remotely resemble the ground truth considering that all we changed is the order of the operands in a multiplication expression.

\begin{figure}[tb]
	\centering
    \captionsetup{skip=5pt}
	\begin{subfigure}{0.45\textwidth}
		\centering
		\begin{subfigure}{0.45\textwidth}
			\lstset{style=mystyle}
\begin{lstlisting}[linewidth=2.8cm,morekeywords={var, public, String}]
int f(int n) {
  if (n == 0)
  {
    return 1;
  }
  else
  {
    return @n * f(n-1)@;
  }
}
\end{lstlisting}
		\end{subfigure}
		\!\,
		\begin{subfigure}{0.4375\textwidth}
			\centering
			\raisebox{0.9mm}{\includegraphics[width=\textwidth]{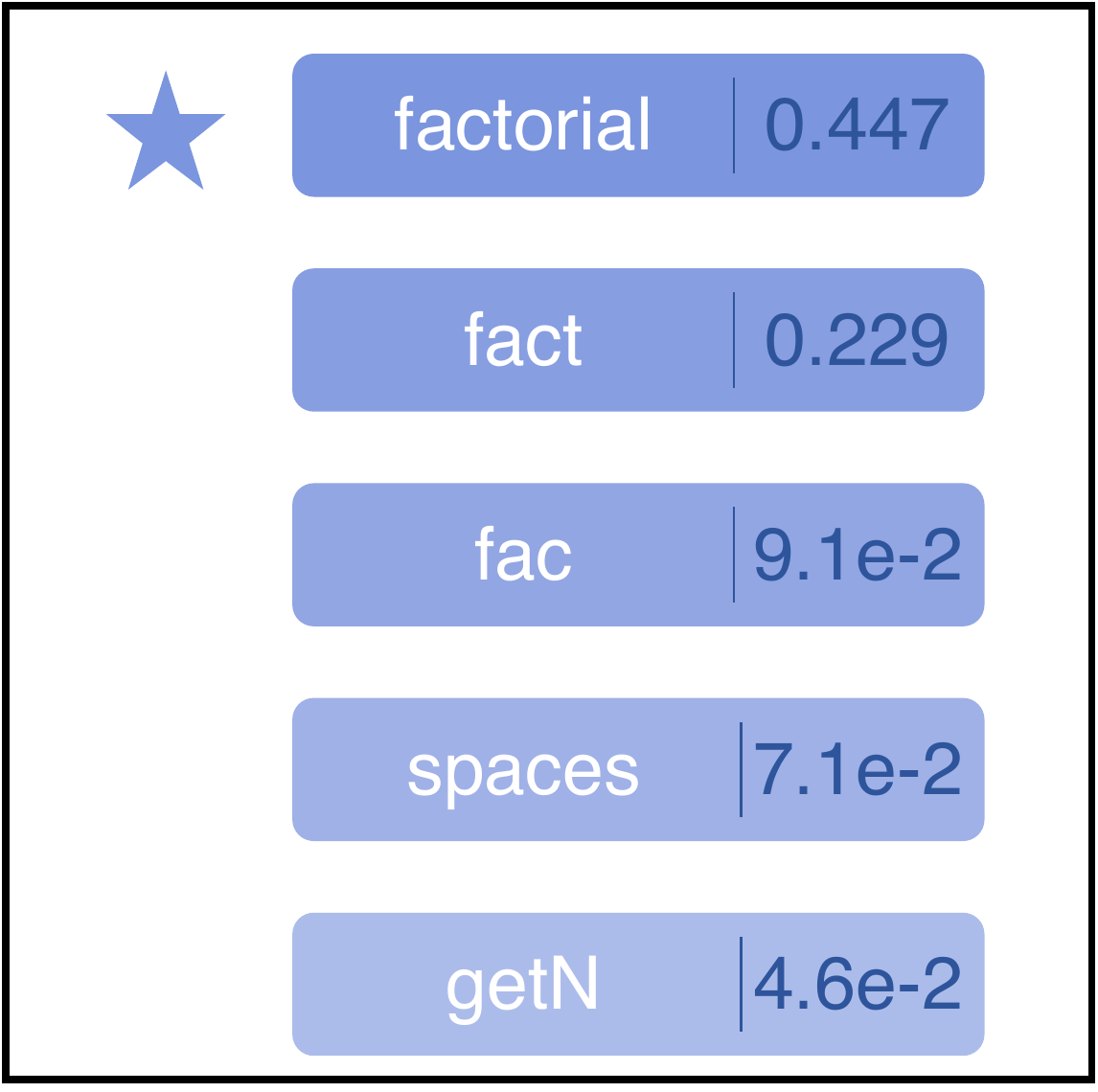}}
		\end{subfigure}	
		\caption{Prediction for the original method.}
		\label{fig:before}
	\end{subfigure} 
	\!\,
	\begin{subfigure}{0.45\textwidth}  
		\centering
		\setlength{\belowcaptionskip}{-1pt}  
		\begin{subfigure}{0.45\textwidth}
			\lstset{style=mystyle}
\begin{lstlisting}[linewidth=2.8cm,morekeywords={var, public, String}]
int f(int n) {
  if (n == 0)
  {
    return 1;
  }
  else
  {
    return @f(n-1) * n@;
  }
}
\end{lstlisting}
		\end{subfigure}
		\!\,
		\begin{subfigure}{0.4375\textwidth}
			\centering
			\vspace{.19pt}
			\raisebox{0.9mm}{\includegraphics[width=\textwidth]{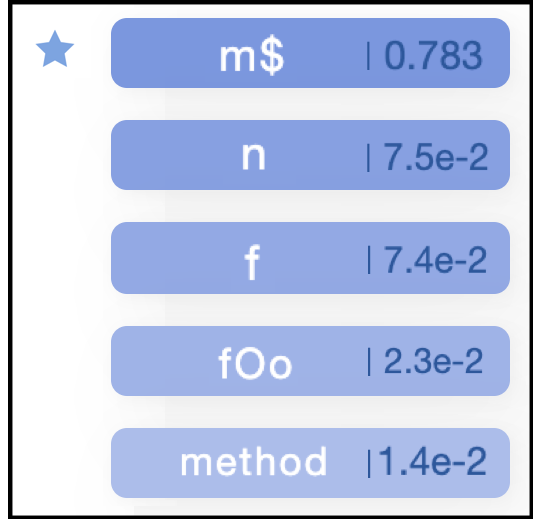}}
		\end{subfigure}	
		\caption{Prediction for the transformed method.}
		\label{fig:after}	
	\end{subfigure}
	\caption{A simple, natural, semantically-preserving transformation causes code2vec to change its prediction. Note that the probability of the top-one prediction is even higher on the transformed method. }
	\label{fig:fac}
	\vspace{-4pt} 
\end{figure}

\begin{table}[b]
	\captionsetup{skip=1pt}	
	\caption{Reducing the running example using DD. $\Delta_i$ denotes partitions and $\nabla_i$ is the complement of $\Delta_i$. For simplicity, we use tokens to represent programs that are tested against the \textit{sufficient} and \textit{necessary} requirement at each step. The last column shows the requirement that partitions do not satisfy.}
	\centering
	\adjustbox{max width=\linewidth}{
		\begin{tabular}{c|c|ccccc|c}
			\hline
			\multirow{2}{*}{\tabincell{c}{Step}} & \multirow{2}{*}{\tabincell{c}{Partition}} & \multicolumn{5}{c|}{Tokens} & \multirow{2}{*}{\tabincell{c}{Unsatisfied}} \\ 
			& & size  &  return &  mItems &  add &  position  &  \\\hline		
			1 & $\Delta_{1}$  & $\checkmark$  &  &  &  &  & Both  \\\cline{1-8}	
			2 & $\Delta_{2}$  &   & $\checkmark$ &  &  &   & Both \\\cline{1-8}	
			3 & $\Delta_{3}$  &   &  & $\checkmark$ &  &   &   Both  \\\cline{1-8}
			4 & $\Delta_{4}$  &   &   &   & $\checkmark$ &  &   Both  \\\cline{1-8}
			5 & $\Delta_{5}$  &   &   &   &  & $\checkmark$  &  Both  \\\cline{1-8}
			6 & $\nabla_{1}$  &   &  $\checkmark$  &  $\checkmark$  & $\checkmark$ & $\checkmark$  & \textit{Sufficient} \\\cline{1-8}
			7 & $\nabla_{2}$  &  $\checkmark$   &  &  $\checkmark$  & $\checkmark$ & $\checkmark$  &\textit{Sufficient}  \\\cline{1-8}
			8 & $\nabla_{3}$  &  $\checkmark$   &  $\checkmark$  &  & $\checkmark$ & $\checkmark$  &   \textit{Sufficient}  \\\cline{1-8}
			9 & $\nabla_{4}$  &   $\checkmark$  & $\checkmark$  & $\checkmark$ &  & $\checkmark$ &  \textit{Sufficient}  \\\cline{1-8}
			10 & $\nabla_{5}$  &   $\checkmark$  & $\checkmark$  & $\checkmark$ &   $\checkmark$ &  &  \textit{Sufficient}  \\\hline
		\end{tabular}
	}
	\label{tab:expmono}
\end{table}

\vspace*{3pt}
\noindent
\textbf{\textit{Applicability of Delta Debugging.}}\,
Their finding suggests that models don't evenly distribute their attention across the entire structure of input programs; instead, they focus on a small fragment of code. At first glance, Delta Debugging (DD)~\cite{zeller2002simplifying} seems to be a perfect approach to finding such code fragment from input programs. In particular, we can apply DD to remove the irrelevant code in reference to the \textit{sufficient} and \textit{necessary} requirement.
It is worth mentioning a property that has to be satisfied in order for DD to be applicable is \textit{Monotony} (Definition 6 in~\citet{zeller1999yesterday}), meaning, in the software debugging setting, whenever an input causes a program to fail a test, others that include this input will also cause the program to fail. This property is crucial because it allows DD to remove partitions of the input that are unrelated to the cause of the failure. As a result, the search space still shrinks even when the failure-inducing faults can not be precisely located from the input. Conversely, without the monotony property, an irrelevant partition may not be removed from the input because its complement, albeit including the failure-inducing faults, can still have the program pass the test. We can easily redefine monotony in our settings to be whenever $p$ satisfies both the \textit{sufficient} and \textit{necessary} requirement, then every super-sequence of $p$  satisfies the two requirements as well, formally, 
$\forall p \preceq P \, (\llbracket M \rrbracket  \big(p\big) =  L \land \llbracket M \rrbracket\big(P \setminus  p\big) \neq  L \land p \preceq p') \Rightarrow \llbracket M \rrbracket \big(p'\big) =  L \land \llbracket M \rrbracket\big(P \setminus  p'\big) \neq L $
where $p\preceq P$ denotes $p$ is a sub-sequence of $P$.
It is rather clear that machine learning models do not guarantee to satisfy this property. 
Below, we use the running example to demonstrate \motimodel's non-satisfaction of the monotony property (Table~\ref{tab:expmono}), and in turn the ramifications of applying DD for \whet detection. 

We skip the steps in which DD successfully reduces the example into the code fragment consisting of \texttt{size();}, \texttt{return;} and \texttt{mItems.add(position);}. Since then DD can not make further reductions even though the fragment is clearly not the \whet. In fact, DD had an opportunity to remove either the statement \texttt{size();} (at Step 6), the statement \texttt{return;} (at Step 7), or the parameter \texttt{position} (at Step 10) had the model satisfied the monotony property. That is, in any of the three steps, the reduced program, which includes the \whet, would have satisfied the \textit{sufficient} and \textit{necessary} requirement and allowed DD to proceed. However, since predictions of \motimodel are not monotonic, the resultant programs at all three steps turn out to violate the \textit{sufficient} requirement. As a result, DD gets stalled and the final output is not minimal. For interested readers, we left DD's complete procedure in reducing the example program to Section~\ref{app-app:dd}  in the supplemental material.

Although refining DD to deal with non-monotony subjects is a pathway forward, it is out of the scope of this work. Instead, we present a simple, efficient search strategy that first aggressively prunes the search space of the \whet (\ie, the goal of \textit{Reduce}), and then precisely locate its constituent program properties from the remaining code fragment (\ie, the goal of \textit{Mutate}).

\setlength\fboxrule{0.6pt}
\lstset{style=mystyle}	
\begin{figure*}[tbp]
		\captionsetup{skip=10pt}
\setlength{\belowcaptionskip}{-10pt}
	\centering
	\begin{tikzpicture}[font=\small,thick,node distance=0.5cm]
	
	\matrix[matrix of nodes, shape=rectangle,draw,
	nodes={anchor=center},column sep=-1.5em, row sep=1pt, inner sep=1.65pt,outer sep=0pt, ampersand replacement=\&](blockStep1){%
		\begin{minipage}[t]{0.32\textwidth}
		\vspace{-0.9pt}
		
		\textbf{\scriptsize{\textit{Step} \circled{1}}}
		
		\vspace{1pt}
		
\begin{lstlisting}[frame=none, morekeywords={var, public, String}]
void @addItem@(int position) {
  List<Obj> mItems = retQueue();
  if (position > mItems.size()) 
    return;
  mItems.add(position, genItem());
  notifyItemInserted(position);
  log("Add item;");
}
\end{lstlisting}
		\end{minipage}
		\&
		\hspace{0.90em}
		\includegraphics[width=0.03\textwidth]{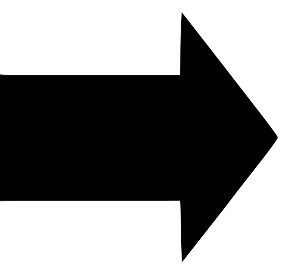}
		\&
		\hspace{1.15em}
		\fbox{
			\begin{minipage}[t]{0.316\textwidth}
			\centering
			\vspace{1pt}
			
			\hspace{1pt}\textbf{\scriptsize{List of statements}}
			
			\vspace{3pt} 
			
			\adjustbox{max width=1\textwidth}{
				\begin{tabular}{>{\columncolor{Gray}}c | c}
				\hline
				Stmt. 1 & \texttt{List<Obj> mItems = retQueue(); }  \\\hline
				Stmt. 2 & \texttt{if(position > mItems.size());} \\\hline
				Stmt. 3 & \texttt{return;} \\\hline
				Stmt. 4 & \texttt{\tabincell{c}{mItems.add(position,\\genItem());}} \\\hline
				Stmt. 5 & \texttt{\tabincell{c}{notifyItemInserted(\\position); }} \\\hline
				Stmt. 6 & \texttt{log("Add item;"); } \\\hline
				\end{tabular}
			}
			\vspace{1pt}
			
			\end{minipage}
		}
		
		\\
	};
	
	
	\node[draw,
	right=of blockStep1] (blockStep2) { 
		\begin{minipage}[t]{0.2\textwidth} 
		\textbf{\scriptsize{\textit{Step} \circled{2}}}
\begin{lstlisting}[frame=none,  morekeywords={var, public, String}]
void @retQueue@(
    int position) {
  List<Obj> mItems = 
    retQueue();
  


}
\end{lstlisting}
		\label{fig:motireduce22}	
		\end{minipage}
	};  
	
	\node[draw,
	below=of blockStep2] (blockStep21) {  
		\begin{minipage}[t]{0.2\textwidth}
		\textbf{\scriptsize{\textit{Step} \circled{3}}}
\begin{lstlisting}[frame=none, morekeywords={var, public, String}]
void @onItemClick@(
    int position) {
  if (position > 
    mItems.size());

    
}
\end{lstlisting}
		\label{fig:motireduce21}		
		\end{minipage}
	}; 
	
	\node[draw,
	left=of blockStep21] (blockStep3) {  
		\begin{minipage}[t]{0.32\textwidth}
		\textbf{\scriptsize{\textit{Step} \circled{4}}}
\begin{lstlisting}[frame=none, morekeywords={var, public, String}]
void @addItem@(int position) {



  mItems.add(position, genItem());

}
\end{lstlisting}
		\label{fig:motireduce23}		
		\end{minipage}
	};

	\node[draw,
	left=of blockStep3] (blockStep4) {  
		\begin{minipage}[t]{0.32\textwidth}
		\textbf{\scriptsize{\textit{Step} \circled{5}}}
\begin{lstlisting}[frame=none, morekeywords={var, public, String, Object}]
void @retQueue@(int position) {
  List<Obj> mItems = retQueue();
  if (position > mItems.size()) 
    return;
  %\sout{mItems.add(position, genItem());}%
  notifyItemInserted(position);
  log("Add item;");}
\end{lstlisting}
		\end{minipage}
	};

	
	\draw[-{Triangle[width=10pt,length=4pt]}, line width=5pt](blockStep1) -- (blockStep2); 
	
	
	\draw[-{Triangle[width=10pt,length=4pt]}, line width=5pt](blockStep2)-- (blockStep21);
	
	\path [draw half paths={dotted,line width=2pt }{line width=5pt,-{Triangle[width=10pt,length=4pt]},solid}]  (blockStep21) -- (blockStep3);

	
	

	\draw[-{Triangle[width=10pt,length=4pt]}, line width=5pt](blockStep3)-- (blockStep4);


	\end{tikzpicture}
	\caption{High-level steps of \textit{Reduce}. The label that \motimodel predicts at each step is highlighted in shadowbox.}
	\label{fig:highreduce}
\end{figure*}

\subsection{Reduce} 
Because fragments that contain the \whet are usually small, the search can be made very efficient by testing out smaller code fragments first. Figure~\ref{fig:highreduce} illustrates the high-level steps of \textit{Reduce} on the running example.

First, we flatten the input method into a list of statements (\textit{Step }\textbf{\circled{1}}).
We then traverse the fragments (\ie, combination of statements) in the ascending order of their size, starting with those that contain only one statement.
We pick \textit{Stmt. 1}: \texttt{List<Obj> mItems=retQueue();} as the first statement to check against the \textit{sufficient} and \textit{necessary} requirement. As depicted in \textit{Step }\textbf{\circled{2}}, \motimodel makes a different prediction for this statement\footnote{To facilitate readers' understanding, we highlight the predicted label within the shadowbox at each step in Figure~\ref{fig:highreduce} and~\ref{fig:mutateprocedure}.}, thus, the \textit{sufficient} requirement is not met. 
We move on to the other statements. As shown at \textit{Step }\textbf{\circled{3}}, \textit{Stmt. 2} also does not satisfy the \textit{sufficient} requirement. We omit \textit{Stmt. 3} for its non-satisfaction of the \textit{sufficient} requirement either. Next, we arrive at \textit{Stmt. 4}: \texttt{mItems.add(position, genItem());}. In particular, we find that \textit{Stmt. 4} satisfies the \textit{sufficient} requirement at \textit{Step }\textbf{\circled{4}} and \textit{necessary} requirement at \textit{Step }\textbf{\circled{5}}, we declare \textit{Stmt. 4} to be the fragment that contains the \whet.
To avoid missing other candidates, we continue the traversal until reaching the last statement \textit{Stmt. 7} \texttt{log("Add item;");}, which violates the \textit{necessary} requirement. Because future fragments to be explored will be non-minimum compared to \textit{Stmt. 4}, the \textit{Reduce} step completes with the lone output of \textit{Stmt. 4}.

\begin{algorithm}[!tbp]

	\algnewcommand\Sz{size\xspace}
	\algnewcommand\Md{method\xspace}
	\algnewcommand\Mn{method.name\xspace}		
	\algnewcommand\Mb{method.body\xspace}	
	\algnewcommand\Mp{method.parameters\xspace}	
	\algnewcommand\Mh{method.head\xspace}	
	\algnewcommand\Mr{method\_header\xspace}		
	\algnewcommand\Sd{new\xspace}	
	\algnewcommand\Sr{new Method(remaining\_statements, method.head)\xspace}		
	\algnewcommand\Ml{model\xspace}
	\algnewcommand\Lb{label\xspace}
	\algnewcommand\Pd{result\xspace}	
	\algnewcommand\K{k\xspace}
	\algnewcommand\I{i\xspace}
	\algnewcommand\S{s\xspace}
	\algnewcommand\St{selected\_statements\xspace}		
	\algnewcommand\Rs{remaining\_statements\xspace}
	\algnewcommand\Ts{statements\xspace}
	\algnewcommand\Ti{statements[i]\xspace}	
	\algnewcommand\Tz{statements.size\xspace}
	\algnewcommand\To{statements.size-1\xspace}		
	\algnewcommand\Tr{traversed\_statements\xspace}
	\algnewcommand\Se{fragment\xspace}	
	\algnewcommand\Rm{region[-1]\xspace}	
	\algnewcommand\Mm{m\xspace}
	\algnewcommand\Ic{is\_continue\xspace}
	\algnewcommand\Nm{m\xspace}
	\algnewcommand\Ci{index\xspace}
	\algnewcommand\Cd{depth\xspace}
	\algnewcommand\Ip{index++\xspace}
	\algnewcommand\Dp{depth++\xspace}
	\algnewcommand\Me{max\_depth\xspace}
	\algnewcommand\Sa{selected\xspace}			
	\algnewcommand\Lz{selected.size\xspace}
	\algnewcommand\rec{Reconstruct\xspace}
	\algnewcommand\vef{VerifyCodeFragment\xspace}
	\algblockdefx[Foreach]{Foreach}{EndForeach}[1]{\textbf{foreach} #1 \textbf{do}}{\textbf{end foreach}}

	\caption{Find the minimum fragment of code.}
	\label{alg:findregionsimp}
 \begin{algorithmic}[1]
\Procedure{FindMinimumFragment}{$program, model$}
 \State statements $\gets$ \textit{Flatten}\! ({program}) \label{line:flat}
 \LineComment{\ \ \ \ \(\triangleright\)  traversing all subsets in the ascending of the }
 \LineComment{\ \ \ \ \hspace{\algorithmicindent}    cardinality; `-1' excludes the entire set.}
\For{k $\gets 1$  \textbf{to} statements.size - 1 \label{line:loop1}}
  	\LineComment{\ \ \ \  \ \ \ \  \(\triangleright\) getting combinations of size {k}}
  	\State sets $\gets$ \textit{CombineKStmt}\! (\Ts, \K) \label{line:combk}  
	\State fragments $\gets$ $\varnothing $
	\Foreach{set $\in$ sets}
			\State suff\_mth, nec\_mth $\gets$ \textit{\rec}\! (set, program)  \label{line:reconsimp} 
			\If{\textit{\vef}\! ({suff\_mth, nec\_mth, \Ml})} \label{line:verifysimp}	
				\State fragments $\gets$ fragments $\cup$ set \label{line:addtoregsimp}
			\EndIf						
	\EndForeach
\If{fragments $\neq$ $\varnothing$} \label{line:ifregionsimp}
	\State  \Return{fragments} \label{line:retregionsimp}
\EndIf 
\EndFor
\EndProcedure
\end{algorithmic}
\end{algorithm}

Algorithm~\ref{alg:findregionsimp} gives the details of the \textit{Reduce} step. The goal of the \texttt{Flatten} function is to reduce an input program into a list of statements (line~\ref{line:flat}). The \texttt{for} loop at line~\ref{line:loop1} gradually increases the size of the subsets drawn from the  \texttt{Flatten} function (line~\ref{line:flat}) while searching for the minimum code fragments. The \texttt{CombineKStmt} function at line~\ref{line:combk} creates subsets of statements with size \texttt{k}. Then, for each subset, the algorithm invokes \texttt{Reconstruct} function (line~\ref{line:reconsimp}) to return two new programs: one for verifying against the \textit{sufficient} requirement (\ie, \texttt{suff\_mth}) and the other for verifying against the \textit{necessary} requirement (\ie, \texttt{nec\_mth}). In particular, \texttt{suff\_mth} is simply the current subset being explored, whereas \texttt{nec\_mth} is resulted from the subtraction of \texttt{suff\_mth} from \texttt{program} (Definition~\ref{def:subtra}).
Once the created subset is found out to be valid by \texttt{VerifyCodeFragment} 
function (line~\ref{line:verifysimp}), we add it to the collection of all minimum subsets \texttt{fragments} (line~\ref{line:addtoregsimp}). After we have traversed 
all combinations of \texttt{k} statements, we return \texttt{fragments} if it's not an empty set (line~\ref{line:retregionsimp}), otherwise, we will continue to explore the code fragments with the size of \texttt{k+1}.

\vspace*{3pt}
\noindent
\textbf{\textit{Reduce vs. DD Regarding Non-monotony.}}\, Since the code fragments that
\textit{Reduce} identifies also include extra code that is irrelevant to the \whet, \textit{Reduce} is affected by models' non-monotony, the same way DD is affected. However, the degree to which non-monotony impacts \textit{Reduce} is in general less than that impacts DD. This is because the extra features included always stem from the same statement where the \whet is, as a result, models, by and large, still behave monotonically (\eg, \motimodel makes the same prediction for \texttt{mItems.add(position,genItem());} and \texttt{mItems.add()}). On the contrary, what DD encounters 
is that the extra code included along with the \whet is significantly less constrained (\ie, it can be any statement or expression in the program) and come in larger quantities, making models far more likely to be non-monotonic. 

 \lstset{style=mystyle}	
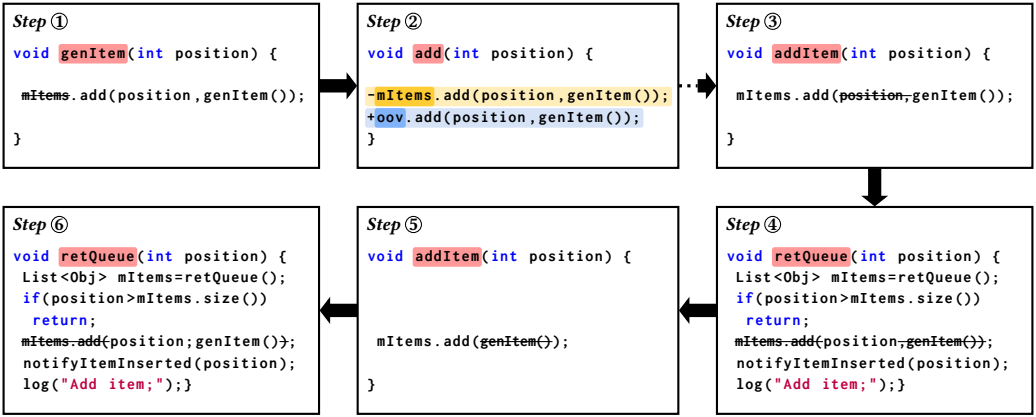
\begin{figure*}[tbp]
		\captionsetup{skip=10pt}
\setlength{\belowcaptionskip}{-10pt}
	\centering
	\begin{tikzpicture}[font=\small,thick, node distance=0.5cm]
	
	\node[draw](blockStep1) {
		\begin{minipage}[t]{0.285\textwidth}
		\textbf{\scriptsize{\textit{Step} \circled{1}}}
\begin{lstlisting}[frame=none, morekeywords={var, public, String}]
void @genItem@(int position) {

 %\sout{mItems}%.add(position,genItem());

}
\end{lstlisting}
		\end{minipage}
	};

	\node[draw,
	right=of blockStep1] (blockStep2) {  
		\begin{minipage}[t]{0.29\textwidth}
		\textbf{\scriptsize{\textit{Step} \circled{2}}}
\begin{lstlisting}[frame=none,  morekeywords={var, public, String}]
void @add@(int position) {

|*-|^mItems^|.add(position,genItem());*|
||^+||*oov*||.add(position,genItem());^||
}
\end{lstlisting}
		\label{fig:motireduce1}
		\end{minipage} 
	};  
	
	\node[draw,
	right=of blockStep2] (blockStep3) {  
		\begin{minipage}[t]{0.285\textwidth}
		\textbf{\scriptsize{\textit{Step} \circled{3}}}
\begin{lstlisting}[frame=none,  morekeywords={var, public, String}]
void @addItem@(int position) {

 mItems.add(%\sout{position,}%genItem());

}
\end{lstlisting}
		\label{fig:motireduce2}	
		\end{minipage}
	}; 
	
	\node[draw,
	below=of blockStep3] (blockStep4) {   
		\begin{minipage}[t]{0.285\textwidth}
		\textbf{\scriptsize{\textit{Step} \circled{4}}}
\begin{lstlisting}[frame=none, morekeywords={var, public, String}]
void @retQueue@(int position) {
 List<Obj> mItems=retQueue();
 if(position>mItems.size()) 
  return;
 %\sout{mItems.add(}%position%\sout{,genItem())}%;
 notifyItemInserted(position);
 log("Add item;");}
\end{lstlisting}
		\label{fig:motireduce3}		
		\end{minipage}
	}; 
	
	\node[draw,
	left=of blockStep4] (blockStep5) {  
		\begin{minipage}[t]{0.29\textwidth}
		\textbf{\scriptsize{\textit{Step} \circled{5}}}
\begin{lstlisting}[frame=none, morekeywords={var, public, String}]
void @addItem@(int position) {



 mItems.add(%\sout{genItem()}%);

}
\end{lstlisting} 
		\label{fig:motireduce4}		
		\end{minipage}
	};

	\node[draw,
	left=of blockStep5] (blockStep6) {   
		\begin{minipage}[t]{0.285\textwidth}
		\textbf{\scriptsize{\textit{Step} \circled{6}}}
\begin{lstlisting}[frame=none, morekeywords={var, public, String}]
void @retQueue@(int position) {
 List<Obj> mItems=retQueue();
 if(position>mItems.size()) 
  return;
 %\sout{mItems.add(}%position;genItem()%\sout{)}%;
 notifyItemInserted(position);
 log("Add item;");}
\end{lstlisting}
		\label{fig:motireduce5}		
		\end{minipage}
	}; 
	
	\draw[-{Triangle[width=10pt,length=4pt]}, line width=5pt](blockStep1)-- (blockStep2);
	
	\path [draw half paths={dotted,line width=1.2pt }{line width=5pt,-{Triangle[width=10pt,length=4pt]},solid}]  (blockStep2) -- (blockStep3);
	
	\draw[-{Triangle[width=10pt,length=4pt]}, line width=5pt](blockStep3) -- (blockStep4);
	
	\draw[-{Triangle[width=10pt,length=4pt]}, line width=5pt](blockStep4) -- (blockStep5); 
	
	\draw[-{Triangle[width=10pt,length=4pt]}, line width=5pt](blockStep5) -- (blockStep6);

	\end{tikzpicture}
	\caption{High-level steps of \textit{Mutate}. {\setlength{\fboxsep}{1pt}\colorbox{remove-line}{- expression}} (\resp, {\setlength{\fboxsep}{1pt}\colorbox{add-line}{+ expression}}) signals removed (\resp, added) lines.}
	\label{fig:mutateprocedure} 
\end{figure*}

\subsection{Mutate}
The features discovered in the \textit{Reduce} step lie at the level of statements (\ie, coarse-grained), thus they are likely to contain redundant elements. To pinpoint the fine-grained \whet, we mutate the program discovered in the \textit{Reduce} step in an attempt to identify the minimal features that keep the \textit{sufficient} and \textit{necessary} requirement satisfied.

Since the number of tokens in the \whet is not guaranteed to be so small as that of the statements in the minimum code fragment, the style of search adopted by the \textit{Reduce} step is likely to be inefficient for the \textit{Mutate} step.
For this reason, we only remove the irrelevant part of the code fragment that does not cause the violation of either the \textit{sufficient} or \textit{necessary} requirement.

\vspace*{3pt}
\noindent
\textbf{\textit{Mutation as well as Reduction.}}\, It is worth mentioning that a significant difference between the two steps is \textit{Mutate} does not adopt solely a program reduction approach to identifying the fine-grained \whet.
As an example, suppose a code fragment is already successfully reduced to a field access expression, \texttt{foo.bar}. We then attempt to further reduce the expression into a single identifier \texttt{foo}. If \texttt{foo} turns out to be invalid, we will mutate \texttt{foo} back into a field access expression, this time with an out-of-vocabulary field name instead of \texttt{bar} --- \texttt{foo.oov}\footnote{out-of-vocabulary words are those models have never encountered during training. Thus, replacing a token with an out-of-vocabulary word erases the influence of the replaced name on models. Throughout the paper, we use \texttt{oov} to represent out-of-vocabulary words for simplicity and clarity.}. 
Our rational is, even if \texttt{foo.bar} satisfies both the \textit{sufficient} and \textit{necessary} requirement, it is still possible that models use the \textit{type} of the syntactic structure --- field access expression --- associated with the object \texttt{foo} rather than the specific name of its field --- \texttt{bar} --- as the \whet. To deal with 
features of this kind that are not explicitly presented, 
we mix the deletion and modification operations in the \textit{Mutate} step to pinpoint both the explicit and implicit fine-grained features.

\begin{figure}[tb!]
	\centering
	\setlength\fboxrule{0.4pt}
    \captionsetup{skip=5pt}

	\begin{subfigure}{0.48\textwidth}
		\centering
		\fbox{\includegraphics[width=\textwidth]{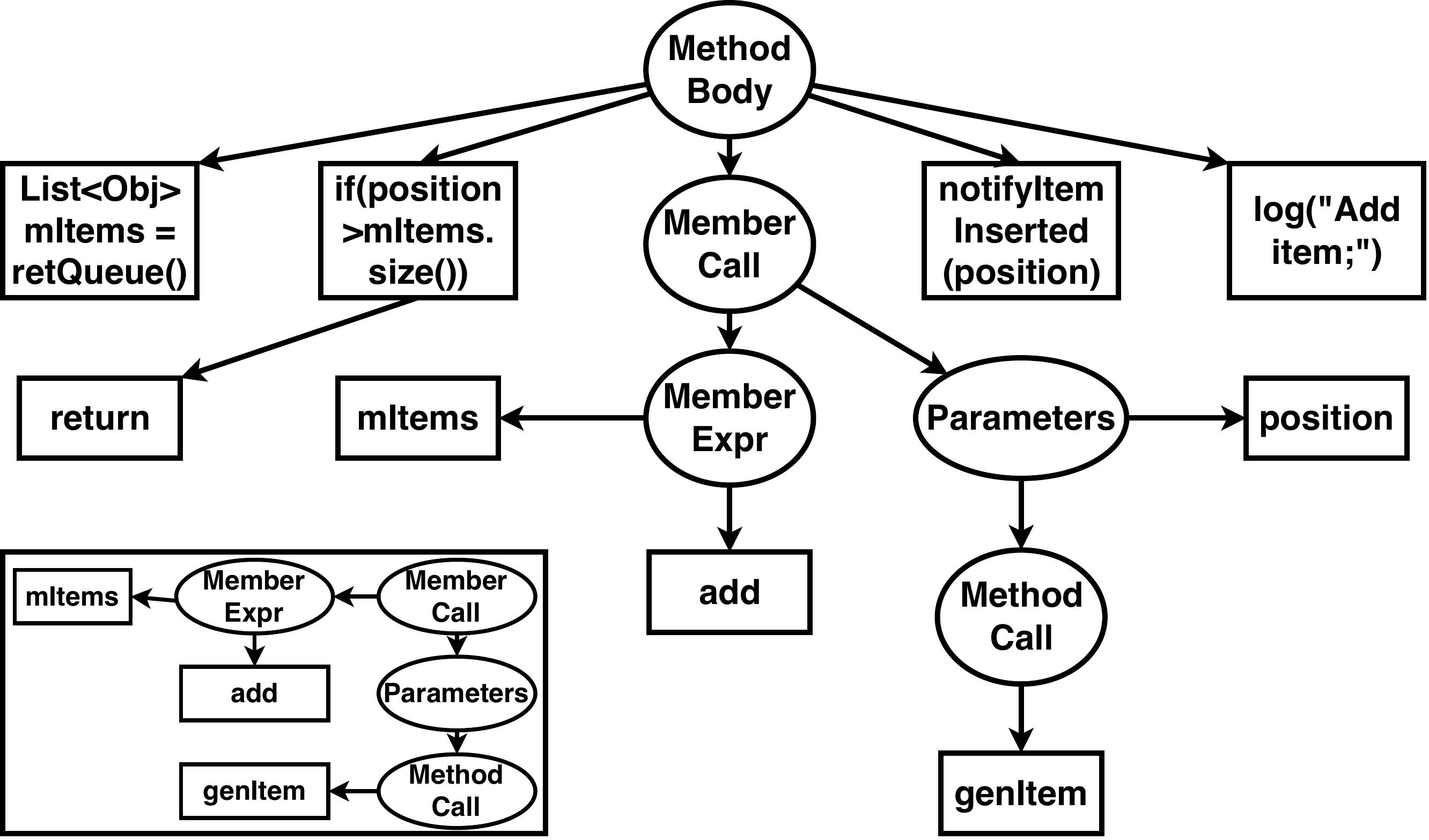}}
		
		\caption{}
		\label{fig:subtraction1}
	\end{subfigure}
	\,\,
	\begin{subfigure}{0.48\textwidth}
		\centering
		\fbox{\includegraphics[width=\textwidth]{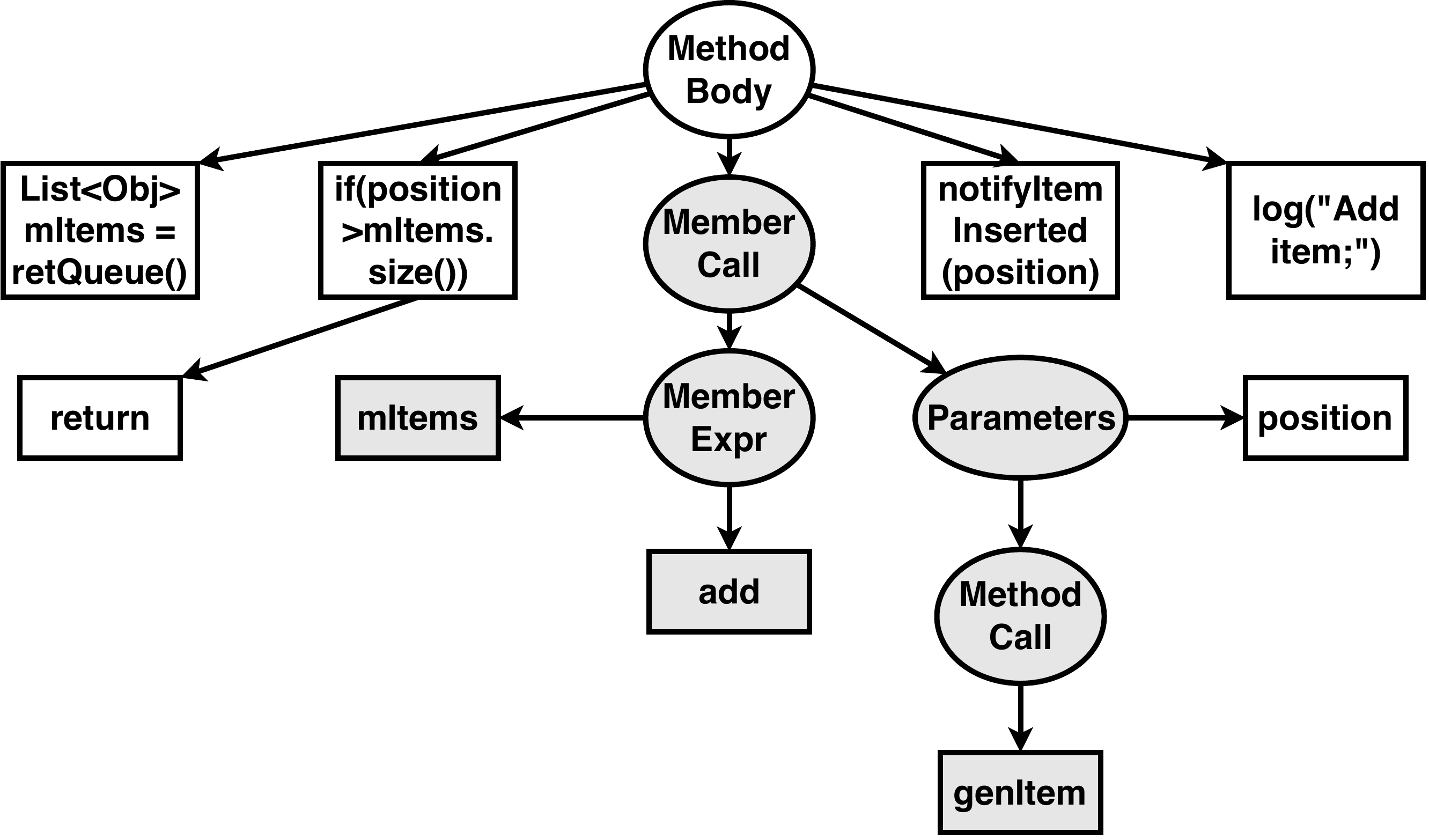}}
		
		\caption{}
		\label{fig:subtraction2}
	\end{subfigure}
	
	\begin{subfigure}{0.48\textwidth}
		\centering
		\fbox{\includegraphics[width=\textwidth]{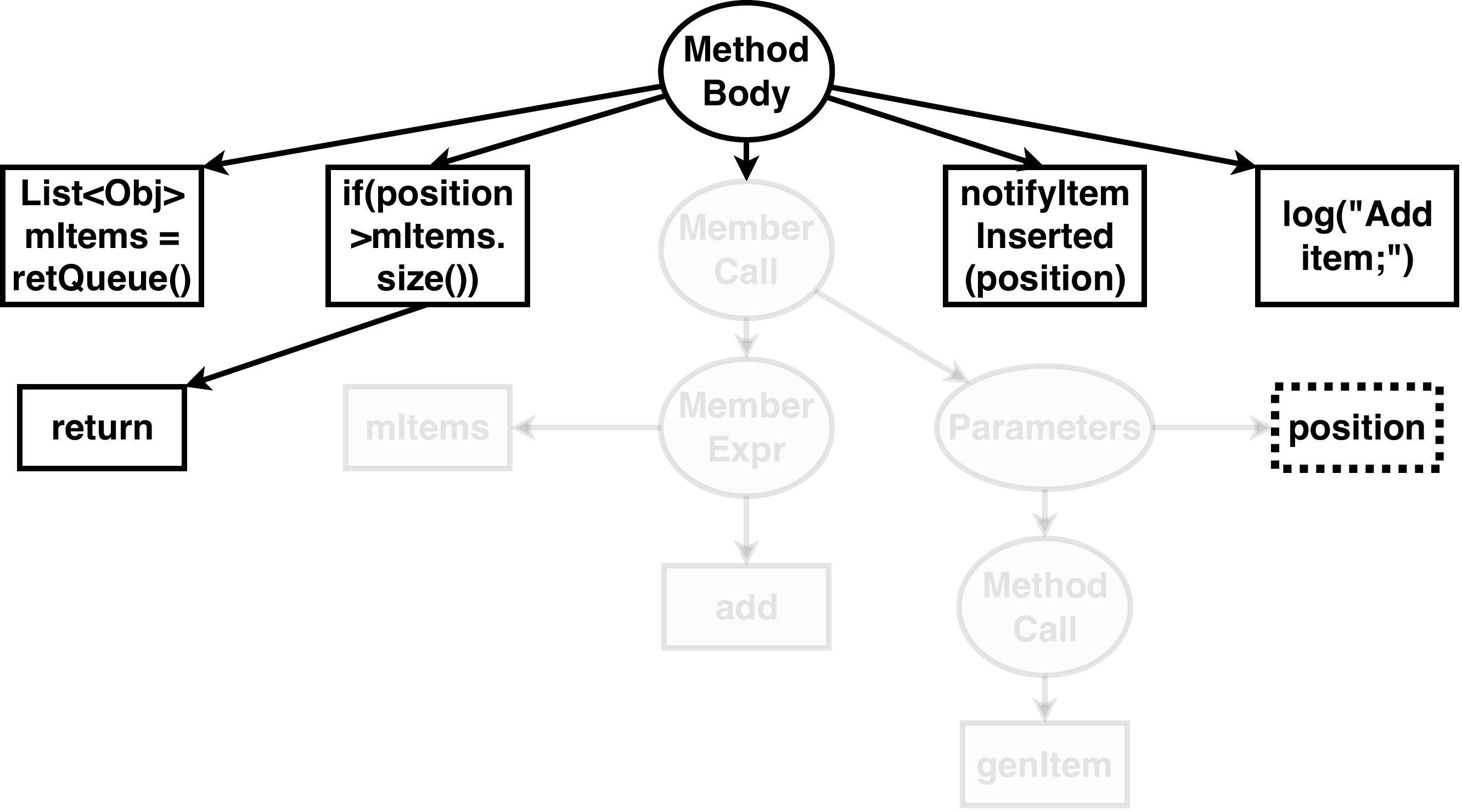}}
		
		\caption{}
		\label{fig:subtraction3}
	\end{subfigure}
	\,\,
	\begin{subfigure}{0.48\textwidth}
		\centering
		\fbox{\includegraphics[width=\textwidth]{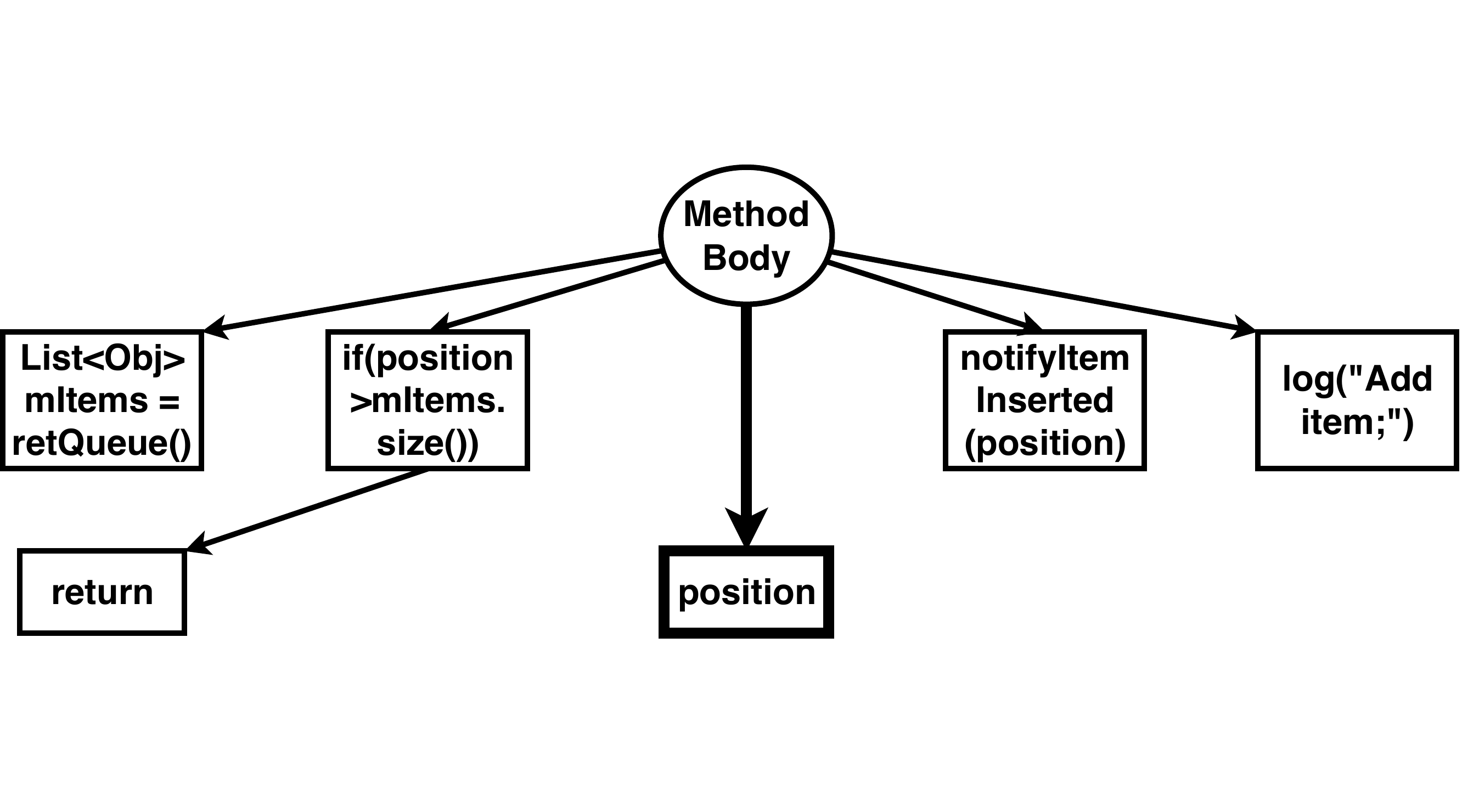}}
		
		\caption{}
		\label{fig:subtraction4}
	\end{subfigure} 
	
	\caption{A illustration of the four-step process of subtracting \texttt{mItems.add(genItem);} from the original program. (1) Figure~(\subref{fig:subtraction1}) presents the AST of the running example, and the AST of \texttt{mItems.add(genItem);} in the bottom left corner. For brevity, the ASTs are simplified. (2) Figure~(\subref{fig:subtraction2}) highlights the overlapping nodes between the two ASTs. (3) Figure~(\subref{fig:subtraction3}) emphasizes the resultant AST after the overlapping nodes are removed, as a result, \texttt{position} becomes a dangling node. (4) Figure~(\subref{fig:subtraction4}) connects \texttt{position} to the body of the method as the completion of the subtraction operation.}
	\label{fig:subtraction}
\vspace*{-10pt}
\end{figure}

Figure~\ref{fig:mutateprocedure} illustrates the procedure of the \textit{Mutate} step. As depicted in \textit{Step} \textbf{\circled{1}} and \textbf{\circled{2}}, it turns out that neither removing \texttt{mItems} nor mutating \texttt{mItems} into \texttt{oov} satisfies the \textit{sufficient} requirement. Similarly, we find that the identifier \texttt{add} is also integral to \motimodel's prediction that can not be changed. Next, we move our attention to the first parameter \texttt{position}. It turns out that \texttt{mItems.add(genItem());} satisfies both the \textit{sufficient} (\textit{Step} \textbf{\circled{3}}) and \textit{necessary} (\textit{Step} \textbf{\circled{4}}) requirement, thus, we remove \texttt{position} from the code fragment. Regarding the subtraction operation (Definition~\ref{def:subtra}) invoked for checking code against the \textit{necessary} requirement, directly removing a subtree from an entire AST may result in dangling nodes that do not connect to any other node in the AST (\eg, the identifier \texttt{position} marked in dashed box in Figure~\ref{fig:subtraction3} after we remove the AST of \texttt{mItems.add(genItem());} from that of the original program). In such cases, we simply connect the dangling nodes to the parent node of the root node for the deleted AST (\eg, connecting \texttt{position} to the method body as depicted in Figure~\ref{fig:subtraction4}). Figure~\ref{fig:subtraction} gives a detailed illustration of this procedure. 
Subsequently, we succeed in removing the second parameter \texttt{genItem()} because \texttt{mItems.add();} also satisfies the \textit{sufficient} (\textit{Step} \textbf{\circled{5}}) as well as the \textit{necessary} requirement (\textit{Step} \textbf{\circled{6}}). Note that the comma that separates the two then parameters \texttt{position} and \texttt{getItem()} becomes a semicolon at \textit{Step} \textbf{\circled{6}} because both \texttt{position} and \texttt{getItem()} will be considered as standalone statements in the resultant program of the subtraction of \texttt{mItems.add()} (Figure~\ref{app-fig:subtraction6} in the supplemental material gives a detailed illustration). 
Finally, we conclude that \texttt{mItems.add()} is the \whet that \motimodel uses for predicting the running example. 

\begin{algorithm}[!tb]
	\algnewcommand\Pd{prediction\xspace}
	\algnewcommand\Ml{model\xspace}
	\algnewcommand\OMT{program\xspace}
	\algnewcommand\K{i\xspace}	
	\algnewcommand\Mn{met.name\xspace}		
	\algnewcommand\Me{met\xspace}
	\algnewcommand\Mt{met\_ast\xspace}
	\algnewcommand\Md{fragments\xspace}
	\algnewcommand\Ma{min\_ast\xspace}
	\algnewcommand\Fe{features\xspace}
	\algnewcommand\Ne{node\xspace}
	\algnewcommand\Tn{target\_node\xspace}
	\algnewcommand\Fg{flag\xspace}
	\algnewcommand\Nt{node.children\xspace}
	\algnewcommand\Nz{node.children.size\xspace}	
	\algnewcommand\Ni{node.children[i]\xspace}	
	\algnewcommand\Na{node.children[node.value[1]]\xspace}	
	\algnewcommand\Nv{node.value\xspace}
	\algnewcommand\No{node.value[0]\xspace}	
	\algnewcommand\Nn{node.value[1]\xspace}	
	\algnewcommand\Ve{value\xspace}
	\algnewcommand\Pl{oov\xspace}		
	\algnewcommand\Rt{root\xspace}
	\algnewcommand\Ry{root.copy\xspace}	
	\algnewcommand\Ss{s\xspace}
	\algnewcommand\Ch{child\xspace}
	\algnewcommand\Cs{replace\xspace}
	\algnewcommand\Rv{replace.value\xspace}
	\algnewcommand\Cv{child.value\xspace}
	\algnewcommand\Cz{child.value[0]\xspace}	
	\algnewcommand\Co{count\xspace}
	\algnewcommand\Ps{pos\xspace}	
	\algnewcommand\Vo{oov\xspace}	
	\algnewcommand\Mr{min\_root\xspace}	
	\algnewcommand\NB{node\_bck\xspace}	
	\algnewcommand\FMN{MutateNode\xspace}	
	\algnewcommand\FDN{DeleteNode\xspace}	
	\algnewcommand\FIN{InheritNode\xspace}	
	\algnewcommand\NewV{new\_value\xspace}	
	\algnewcommand\Ncv{node.children[0].value\xspace}	
	\algnewcommand\IV{IsValid\xspace}	
	\algnewcommand\CR{CheckRequirements\xspace}	
	\algblockdefx[Foreach]{Foreach}{EndForeach}[1]{\textbf{foreach} #1 \textbf{do}}{\textbf{end foreach}}
	
	\caption{Find the \whet.}
	\label{alg:keyfeasimp}
	\begin{algorithmic}[1]
		\item[]
		
		\Procedure{FindFeatures}{$\Md, \OMT, \Ml$}
		\State \Mr $\gets$ \textit{Parse}\!\:(\OMT).\Rt   		\Comment{ \Mr is defined to track the minimal features}
		
		\LineComment{\hspace*{0.82em} \(\triangleright\)  {\Md} is a list of the minimum fragments returned 
			from the \textit{Reduce} step}
		\Foreach{min\_mth $\in$ \Md}
		\State \Rt $\gets$ \textit{Parse}\!\:(min\_mth).\Rt \label{line:parseast}		
		\State \Call{MutateToFixpoint}{\Rt, \OMT, \Ml}  \label{line:fixedpoint} 
		\State \Mr $\gets$  \textit{Min}\!\:({\Rt, \Mr}) \Comment{return the minimum of two code features}	
		\EndForeach
		\State	\Return{\Mr}
		\EndProcedure
		\item[]
		\Function{MutateToFixpoint}{$\Rt, \OMT, \Ml$}
		
		\State last\_root $\gets$ $\emptyset $
		\While{last\_root $\neq$ \Rt} \label{line:fixbegin} 
		\State last\_root $\gets$ \Rt 	
		\State \Call{Mutate}{\Rt, \Rt, \OMT, \Ml} \label{line:mutate}
		\EndWhile \label{line:fixend} 
		
		\EndFunction
		
		\item[]
		
		\Function{Mutate}{$\Ne, \Rt, \OMT, \Ml$} \label{line:fixedpointimp}
		\Foreach{\Ch $\in$ \Nt} \label{line:mutfor1} 
		\State \Call{Mutate}{\Ch, \Rt, \OMT, \Ml} \label{line:post}	
		\EndForeach \label{line:mutfor2}
		
		
		\If{\Nz == 0} \label{alg:mut:count0begsimp}
		\If{\textbf{not} \Call{\FDN}{\Ne, \Rt, \OMT, \Ml}}  \label{line:wrap1}\Comment{removing {node} }
		\LineComment{\ \ \ \  \ \ \ \  \ \ \ \   \(\triangleright\)  if fails, mutating {node}} 
		\State \Call{\FMN}{\Ne, \Rt, \OMT, \Ml}\label{line:wrap2} 
		
		\EndIf \label{alg:mut:count0endsimp}
		\EndIf  \label{alg:mut:count1end}
		
		\EndFunction
	\end{algorithmic}
\end{algorithm}

Algorithm~\ref{alg:keyfeasimp} gives the details about how to pinpoint the \whet within the minimum code fragment discovered by the \textit{Reduce} step. 
Technically,
we perform a postorder traversal on the AST (line~\ref{line:parseast}) of the minimum fragments. For each node operation in the AST, we consider the following two cases. If a node does not have any child node, then we first try deleting the node (line~\ref{line:wrap1}). If the resultant program no longer satisfies both the \textit{sufficient} and \textit{necessary} requirement, we will mutate the node into one with an out-of-vocabulary value (line~\ref{line:wrap2}). 
If a node has children nodes, it will be kept intact to preserve the status of its children (the \texttt{if} condition at line~\ref{alg:mut:count0begsimp} will be evaluated to \texttt{false}). 
If by any chance a node happens to be irremovable (\resp, immutable), we simply skip the deletion (\resp, mutation) operation. 
This entire process is repeated until the AST of the code fragment can not be reduced any further (line~\ref{line:mutfor1} to~\ref{line:mutfor2}). Considering some expressions, which can not be removed or mutated before, may become removable or mutable after others are dealt with first, thus, we compute a fixed point on the result of the \textit{Mutate} function to make certain the final output will be minimal. (lines~\ref{line:fixbegin} to~\ref{line:fixend}).\footnote{Interested readers may refer to Section~\ref{app-app:alg} in the supplemental material for the details of \texttt{DeleteNode} and \texttt{MutateNode} in Algorithm~\ref{alg:keyfeasimp}.}

\vspace*{3pt}
\noindent
\textbf{\textit{Minimality}}\, Although the result of Algorithm~\ref{alg:keyfeasimp} may not satisfy the 
\textit{minimum} requirement in Definition~\ref{def:keyfea}, it satisfies \textit{1-minimality}, or more precisely \textit{1-tree-minimality}~\cite{hddpro} that DD no longer satisfies due to the non-monotony of machine learning models. In addition, we show in Section~\ref{subsec:bfs} that even without the guarantees in theory, \textit{Reduce} and \textit{Mutate} is in fact remarkably effective in finding the global minimum features in practice.

\vspace*{3pt}
\noindent
\textbf{\textit{Compilability}}\, We can easily adjust \textit{Reduce} and \textit{Mutate} to accommodate code models that do require input programs to compile. For example, we can filter out uncompilable code with the help of existing compilers in both steps and only use those that compile to query the model. 

\subsection{The Validity of \method's Approach in Light of the Out-of-distribution Issues}
\label{sec:ood}

Since the programs that \method queries a model with at both \textit{Reduce} and \textit{Mutate} step differ from those that the model has seen in training, it seems that \method violates one of the key assumptions
in machine learning: the training and evaluation data come from the same distribution, in which case the predictions that the model makes for the query programs will be unsound, and the validity of \method's approach will be called into serious question. To address this issue, we draw from a seminal work in explainable ML, ROAR~\cite{ROAR}, which proposes a retraining approach for handling the distribution shift when evaluating the explainability methods. Their insight is that a commonly used strategy in estimating the feature importance --- removing the supposedly informative features from the input and observing how the classifier degrades --- comes at a significant
drawback because a performance degradation might be caused by a shift in distribution instead of removal of information. In contrast,~\citet{ROAR} first retrains the model on the modified data (\eg, in image classification domain, they replace the fraction of the pixels estimated to be most important by an attribution method with a fixed uninformative value for each training image) and then evaluates on test images, which are modified in the same manner, if a feature estimator can identify the important input pixels whose subsequent removal causes the sharpest degradation in accuracy. Since retraining makes certain that training and test data comes from a similar distribution,
the removal of important features becomes the only plausible explanation for accuracy degradation.
On the other hand, ROAR is not without limitations. For one, while the architecture is the same, the model used during evaluation is not the same as the model on which the feature importance estimates were originally obtained. 

Inspired by~\citet{ROAR}'s approach, we design a similar retraining procedure that not only tackles the out-of-distribution issue but also addresses the inherent limitation of ROAR. Denoting the model by $M$, its training set by $\mathcal{P}$ and test set by $\mathcal{Q}$, we describe the steps of our approach:

\begin{enumerate}[leftmargin=*]
	\item for each program $P \in \mathcal{P}$, we perform \textit{Reduce} and \textit{Mutate} to generate new programs $\mathbb{P}$ --- what \method would have queried the model with for finding the \whet of $P$.
	
	\item we manually label each program $P \in \mathbb{P}$ before training $M$ on $\mathcal{P} \cup \mathbb{P}$. 
	
	\item we deploy \method to find the \whet that retrained $M$ uses to predict each program $Q \in \mathcal{Q}$.
	
	\item we compare the \whet that \method finds for each program $Q \in \mathcal{Q}$ before and after the retraining of $M$.
\end{enumerate}

Using the \whet produced by step (3), albeit no longer out-of-distribution, creates the same problem: \method is essentially explaining the predictions of the retrained model rather than the original model. Instead, we check the equivalence of the two sets of \whet at step (4), if they turn out to be identical, we declare that the \whet \method finds for each test program $Q$ are indeed the \whet that $M$ uses for prediction; furthermore, the consistency between the two sets of \whet also indicates the retraining procedure does not fundamentally change the model behavior, therefore, \method can explain predictions of a model in its original form (without retraining).

\subsection{An Application of The \Whet}
\label{sec:exam}

We utilize the revealed features code models use to explain their predictions. The idea is as follows. Given a Model $M$, and a program $P$ from $M$'s test set for which $M$ predicts $L$. First, we discover the \whet $M$ uses to predict $L$; Second, we do the same thing for all programs in $M$'s training set with the label $L$. Note that in a generative setting (\eg, code2vec, seq-GNN) where training set might not contain a method of label $L$, we consider all methods whose name is a subset of the words in $L$. Finally, we rank these training programs based on the distance between their and $P$'s revealed features. Our rationale is to find the training programs from which models extracted the \whet at the first place. Thus, our method explains the root cause of a prediction rather than interprets the internal state of models.   
\section{Evaluation}
\label{sec:eva}

We realize our method \textit{Reduce} and \textit{Mutate} in a tool, called \tool, which extracts the \whet that code models use for prediction. In this evaluation, we first deploy \tool to explain several prominent code models. Then we examine the validity of \method's approach. Next, we evaluate if attribution methods can find the \whet that \tool finds. Finally, we evaluate the effectiveness of our \whet-based explanation to code models. 

\subsection{Evaluation Subjects}
\vspace*{1pt}
\noindent
\textbf{\textit{Models.}}\, We choose code2vec~\cite{Alon2019code2vec} because it is among the first to predict method names in large-scale, cross-project settings. We choose \motimodel since it's the state-of-the-art model in method name prediction. GGNN~\cite{allamanis2018learning} is selected because it is the first graph model applied in the programming domain for variable misuses detection. 
We also include CodeBERT~\cite{feng-etal-2020-codebert}, a bimodal pre-trained model for programming language and natural language. It achieves state-of-the-art performance on both natural language code search and code documentation generation, the latter of which is used in this evaluation.

\vspace*{3pt}
\noindent
\textbf{\textit{Datasets.}}\, For all models except code2vec, we use the datasets on which they are evaluated when they are first proposed. That is Java-small for \motimodel; a dataset of 2.9 million lines of real-world C\# code for GGNN; and CodeSearchNet~\cite{husain2019codesearchnet} for CodeBERT. To keep our engineering effort manageable, we only use Java programs in CodeSearchNet (496,688 in total), which should be sufficient for validating \tool's effectiveness.
Since Java-med, and Java-large share similar characteristics with Java-small except for their bigger size, and code2vec and \motimodel target the same downstream tasks, we use Java-small, Java-med, and Java-large to evaluate both models.
We have re-trained all models using their implementations open-sourced on GitHub. The performance of all re-trained models is either comparable or superior to the original (Table~\ref{app-tab:com1}-\ref{app-tab:com3} in the supplemental material). Note that we only use programs in test sets as the evaluation subjects in case models over-fit to training sets.

\vspace*{3pt}
\noindent
\textbf{\textit{Baseline.}}\, We use DD as a baseline method; in particular, we apply DD to remove code that is irrelevant to a prediction (determined by the \textit{sufficient} or \textit{necessary} requirement). In the end, we declare the program that can not be further removed as the \whet.  
%

\vspace*{3pt}
\noindent
\textbf{\textit{Hardware.}}\, All experiments are conducted on 5 Red Hat Linux servers each having 64 Intel(R) Xeon(R) 2.10GHz CPU, 755GB RAM and 4 Tesla V100 GPU (32GB GPU memory).

\subsection{Performance of \tool{}}
\label{subsec:effic}

We evaluate the performance of \method by the time \tool spends generating explanations end-to-end (\ie, including the time models spend on prediction). The results show that the two methods are neck and neck in practice even though DD is more efficient than \tool in terms of the worst-case complexity (\ie, quadratic time \textit{vs.} exponential time). This is due to the size of the minimum code fragment \tool finds, which never exceeds three statements for over a million programs used in the evaluation. Furthermore, since the minimal fragment of three statements rarely occurs, DD would not display a significant upgrade. To conclude, \tool is efficient: taking on average less than twenty seconds to compute the \whet. 

\begin{table}[t]
		\captionsetup{skip=1pt}
		\caption{Avg. end-to-end time for \tool and the baseline. Hereinafter, stronger results are marked in bold.}
		\centering
		\adjustbox{max width=1\textwidth}{
			\begin{tabular}{c|c|c|c|c|c|c|c|c}
				\hline
				\multirow{2}{*}{Methods} & \multicolumn{3}{c|}{code2vec} &  \multicolumn{3}{c|}{seq-GNN}  & GGNN & CodeBERT \\\cline{2-9}
				& Java-small & Java-med & Java-large & Java-small & Java-med & Java-large & C\# Datasets & CodeSearchNet  \\\hline
				Baseline & 26.45 & 18.24 &  23.28 & \textbf{7.28} 	&	12.78	& \textbf{5.34}  & \textbf{5.21} & 16.37 \\\hline
				\tool & \textbf{19.76}	 &	\textbf{14.37}	 &	\textbf{17.71}  &  10.57	 &	\textbf{12.07}	 &	9.79 & 7.54 &  \textbf{11.23}  \\\hline				
		\end{tabular}}
		\label{tab:etetpnew}
\vspace*{-5pt}
\end{table}


\subsection{Makeup of the \Whet}
\label{subsec:setsta}

\noindent
\textbf{\textit{Finding the \Whet}}\,  We follow Algorithm~\ref{alg:findregionsimp} and \ref{alg:keyfeasimp} to identify the minimum code fragments and pinpoint the fine-grained features within these fragments. Table~\ref{tab:sizemutnew} gives the average number of tokens the \whet is composed of. 
Overall, we find that none of the \whet that any evaluated model uses for prediction exceed
fifteen tokens. This result indicates that existing code models use simple program properties for prediction. Additionally, the \whet generated with those properties, albeit not guaranteed to be the global minimum, still help end users to know the \whet that models use for predictions. Table~\ref{tab:sizemutnew} also gives a comprehensive, head-to-head comparison between \tool and DD. Clearly, \tool is better across the board. A deeper analysis reveals that
for less than 9\% of all test programs, DD manages to find the same \whet as \tool, while \tool finds smaller \whet for the remaining over 91\%. 
In other words, DD never finds smaller \whet than \tool. 
This is a concrete piece of evidence that 
the degree to which non-monotony impacts \tool is significantly less than that impacts DD. 

\begin{table}[h]
	\begin{minipage}{1\textwidth}
		\captionsetup{skip=1pt}	
		\caption{The average number of tokens in the \whet.}
		\centering
		\adjustbox{max width=1\textwidth}{
			\begin{tabular}{c|c|c|c|c|c|c|c|c}
				\hline
				\multirow{2}{*}{Methods}  & \multicolumn{3}{c|}{code2vec} &  \multicolumn{3}{c|}{seq-GNN}  & GGNN & CodeBERT \\\cline{2-9}
				& Java-small & Java-med & Java-large & Java-small & Java-med & Java-large & C\# Datasets & CodeSearchNet  \\\hline
				Baseline & 14.21 & 17.13  & 16.24 & 12.30 & 14.84 & 10.04 & 14.22 & 13.63 \\\hline
				{\tool}  & \textbf{8.53} & \textbf{9.91}  & \textbf{9.58} & \textbf{6.74} & \textbf{6.79} & \textbf{5.66} & \textbf{6.81} & \textbf{6.78}\\\hline
		\end{tabular}} 
		\label{tab:sizemutnew}
	\end{minipage}
\vspace*{-5pt}
\end{table}

\begin{figure}[t!]
	\centering
	\captionsetup{skip=3pt}
	
	\begin{subfigure}{0.20\textwidth}
		\lstset{style=mystyle}
\begin{lstlisting}[morekeywords={Object, String}]
void init(
    CipherParameters 
    params) {


  reset();

  cipher.init(true, 
    params);



}
\end{lstlisting}
	\end{subfigure}
	\,
	\begin{subfigure}{0.35\textwidth}
		\lstset{style=mystyle}
\begin{lstlisting}[morekeywords={Object, String}]
void sort(int[] array) {
  boolean swapped = true;
  for (int i = 0; i < array.length 
      && swapped; i++) {
    swapped = false;
    for (int j = 0; j < array.length 
        - 1 - i; j++) {
      if (array[j] > array[j+1]) {
        int temp = array[j];
        array[j] = array[j+1];
        array[j+1]= temp;
        swapped = true;   
}}}}
\end{lstlisting} 
	\end{subfigure}    
	
	\vspace{-4.8pt}

	\begin{subfigure}{0.20\textwidth}
		\lstset{style=mystyle}
\begin{lstlisting}[morekeywords={Object, String}]
void init() { init; }
\end{lstlisting}
		\label{fig:explanlex1}
	\end{subfigure}
	\,  
	\begin{subfigure}{0.35\textwidth} 
		\lstset{style=mystyle}
\begin{lstlisting}[morekeywords={Object, String}]
void sort(int[] array) { swapped; }
\end{lstlisting}
		\label{fig:explanlex2}
	\end{subfigure} 
	
	\caption{Lexical \whet. Top (\resp, bottom) figures are the original methods (\resp, \whet).}	\label{fig:explanlex}
	
\vspace{-5pt}
	
\end{figure}

\begin{figure}[tb!]
	\centering
	\captionsetup{skip=4pt}
	
	\begin{subfigure}{0.26\textwidth}
		\lstset{style=mystyle}
\begin{lstlisting}[morekeywords={Object, String}]
void checkDone() {

  boolean done = false;

  while (!done) {
    if (remaining() <= 0) {
      done = true;

    }
  }
}
\end{lstlisting}
	\end{subfigure}   
	\,
	\begin{subfigure}{0.30\textwidth}
		\lstset{style=mystyle}
\begin{lstlisting}[morekeywords={Object, String}]
int nextInt() {
  while (st == null || 
    !st.hasMoreElements()) {
    try {
      st = new StringTokenizer(
          br.readLine());
    } catch (Exception e) {
      e.printStackTrace(); 
    }}
  return Integer.parseInt(
    st.nextToken());}
\end{lstlisting}
	\end{subfigure}

	\vspace{-4.8pt}
	
	\begin{subfigure}{0.26\textwidth}
		\lstset{style=mystyle}
\begin{lstlisting}[morekeywords={Object, String}]
void checkDone() {
  while (!done) {
  }}
\end{lstlisting}
		\label{fig:explansyn3} 
	\end{subfigure}
	\,
	\begin{subfigure}{0.30\textwidth} 
		\lstset{style=mystyle}
\begin{lstlisting}[morekeywords={Object, String}]
int nextInt() {
  return parseInt(
    nextToken());}
\end{lstlisting}
		\label{fig:explansyn4}
	\end{subfigure}
	\caption{Syntactic \whet.}
	\label{fig:explansynt}
	
	\vspace{-5pt}
	
\end{figure}

\vspace*{3pt}
\noindent
\textbf{\textit{Investigating the Substance of the \Whet.}}\, We classify the \whet into three categories according to their constituent program properties: lexical, syntactic, and semantic.

\begin{itemize}[leftmargin=*]
	\item \textbf{Lexical:} If each statement in the \whet consists of a single identifier. Figure~\ref{fig:explanlex} depicts two examples, in both cases, the \whet consist of only one identifier.
	\item \textbf{Syntactic:} If there is at least one statement in the \whet composed of a syntactic expression (Figure~\ref{fig:explansynt}).
	\item \textbf{Semantic:} \motimodel and GGNN are the two models which take in semantic properties as model inputs. In particular, they use nine kinds of manually designed edges (described in Section~\ref{app-app:edges} of the supplemental material), out of which seven can be deemed as semantic in nature, to enrich the original AST of input programs. Therefore, we define the \whet to be semantic if its AST is augmented with at least one semantic edge. Take the method \texttt{compare} in Figure~\ref{fig:explanSemantics} as an example, the three semantic edges --- one \texttt{LastWrite} edge and two \texttt{ComputeFrom} edges --- must be present in the AST of the \whet, otherwise the \textit{sufficient} and \textit{necessary} requirement are no longer satisfied at the same time. Thus, we classify this example as semantic. To determine whether \motimodel and GGNN used semantic properties in their \whet, we first run \textit{Reduce} and \textit{Mutate} as usual, then we remove all semantic edges from the augmented AST of the obtained \whet. If the resulted tree no longer satisfies both the \textit{sufficient} and \textit{necessary} requirement, we categorize the \whet to be semantic.
\end{itemize}

\begin{figure}[tb!]
	\centering
	
	\,\,
	\begin{subfigure}{5.4cm}
		\lstset{style=mystyle}
\begin{lstlisting}[morekeywords={Object, String}]
void writeSumToFile(String filename, 
    int[] x) {
  BufferedWriter outputWriter = null; 
  int sum = 0;
  outputWriter = new BufferedWriter(
        new FileWriter(filename));
  for (int i=0; i<x.length; i++) sum+=i;
  outputWriter.write(String.valueOf(sum));
  outputWriter.close();}
\end{lstlisting}
	\end{subfigure}
	\,
	\begin{subfigure}{5.4cm}
		\lstset{style=mystyle} 
\begin{lstlisting}[morekeywords={Object, String}]
int compare(Object obj1, Object obj2) {
  int intObj1 = (int) obj1;
  int intObj2 = (int) obj2;
  int difference = intObj1 - intObj2;

  if (difference < 0) return -1;
  else if (difference > 0) return 1;
  else return 0;
}
\end{lstlisting}
	\end{subfigure}
	
	\hspace*{-6pt}	
	\setlength\fboxrule{0.4pt}
	\begin{subfigure}{5.4cm}
		\centering
		\fbox{\includegraphics[height=5.517cm,width=5.4cm]{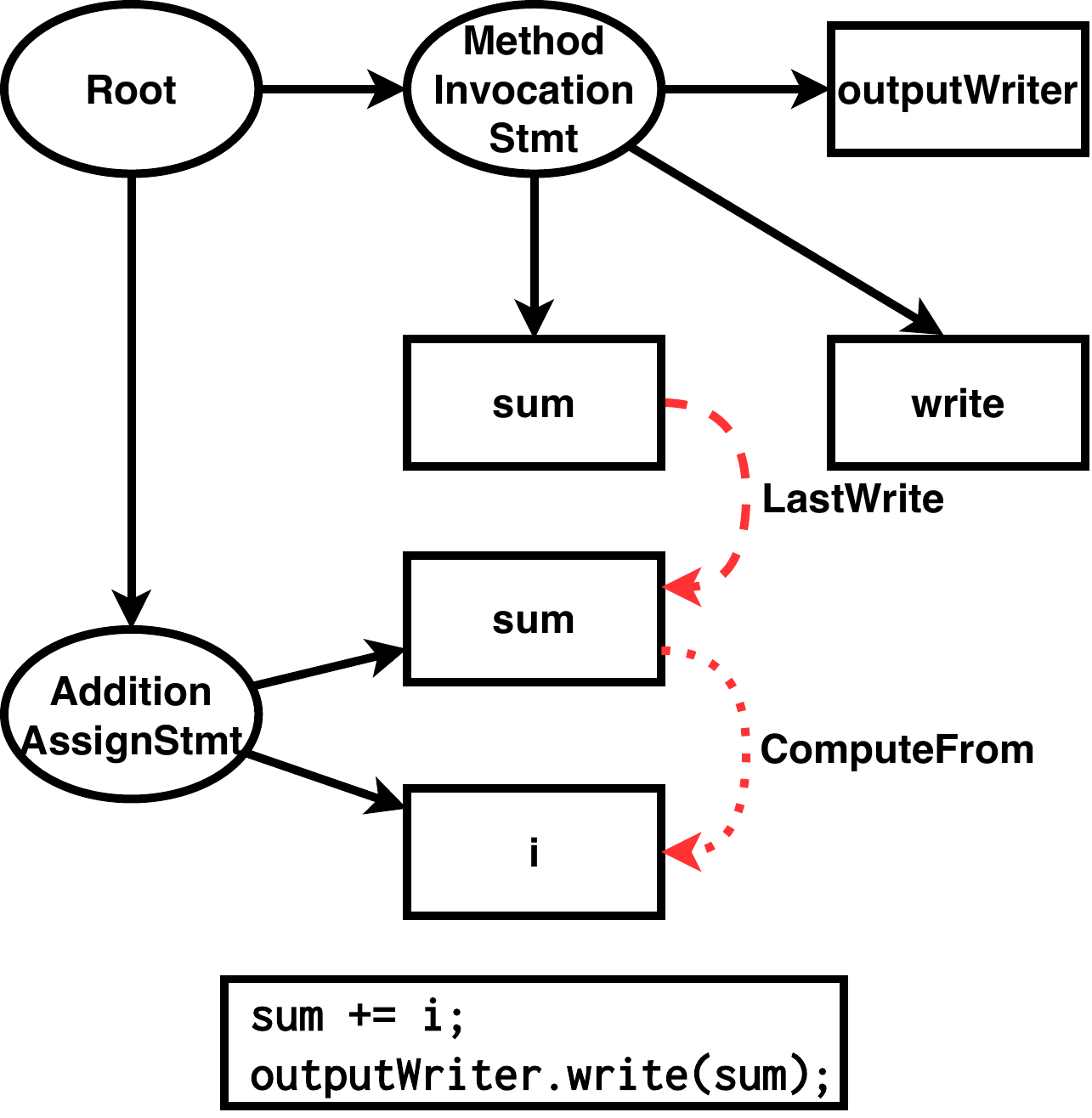}}
	\end{subfigure}
	\,
	\begin{subfigure}{5.4cm}
		\centering
		\fbox{\includegraphics[height=5.517cm,width=5.3875cm]{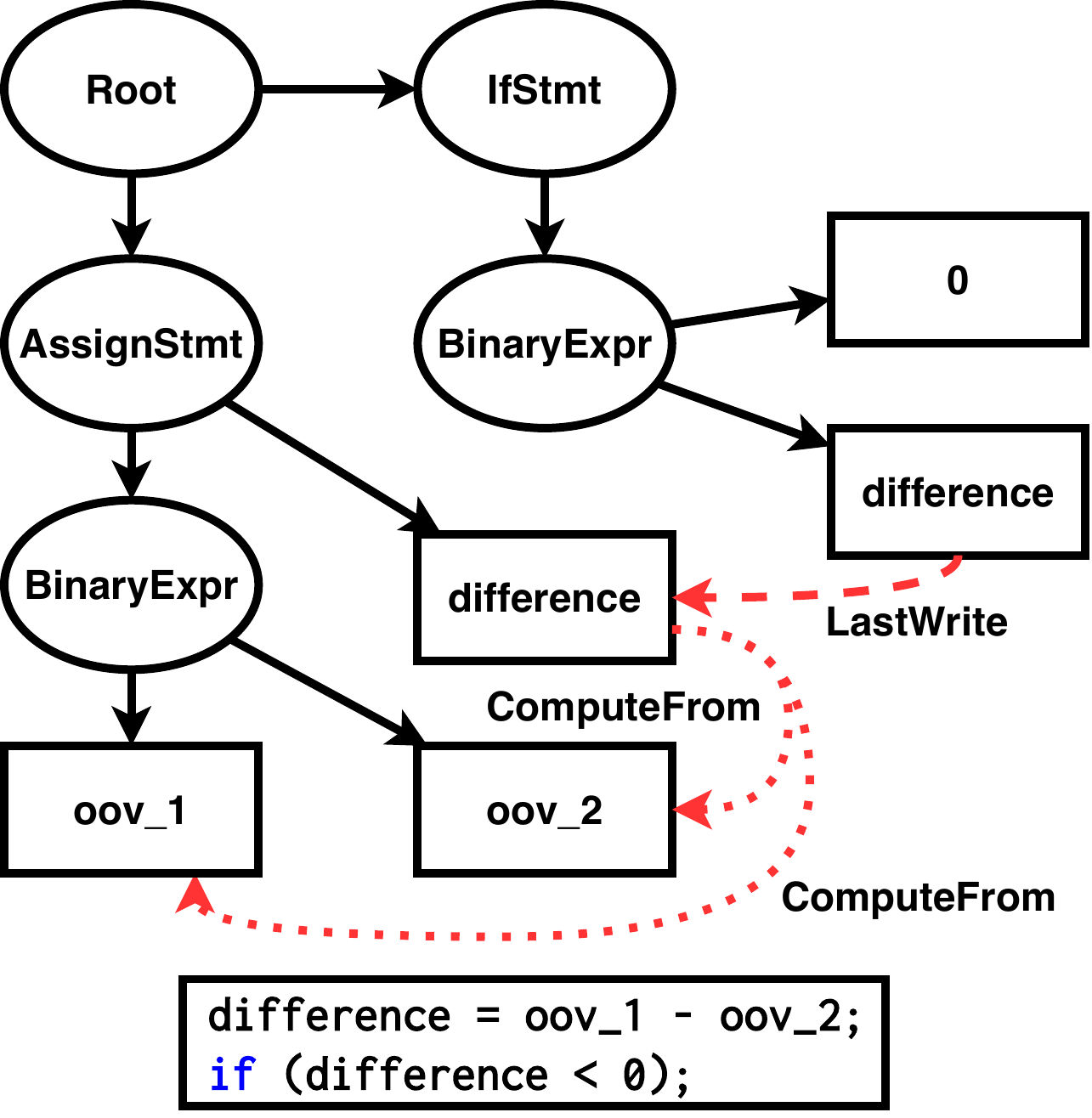}}
	\end{subfigure} 
	
	\caption{Semantic \whet. The bottom figure shows the AST of the \whet, in which dash arrows denote the semantic edges. For clarity, the AST is simplified.}
	\label{fig:explanSemantics}
	\vspace{-3pt}
\end{figure}

Table~\ref{tab:lssperc} gives detailed statistics on the classification for each model. Clearly, for all evaluated models, the \whet primarily consists of lexical properties.

\vspace{-3pt}
\begin{table*}[h]
	\captionsetup{skip=1pt}
	\caption{Types of the \whet that each model uses for prediction.}
	\centering
	\adjustbox{max width=0.95\textwidth}{
		\begin{tabular}{c|c|c|c|c|c|c|c|c}
			\hline
			\multirow{2}{*}{Types} & \multicolumn{3}{c|}{code2vec} &  \multicolumn{3}{c|}{seq-GNN}  & GGNN & CodeBERT \\\cline{2-9}
			& Java-small & Java-med & Java-large & Java-small & Java-med & Java-large & C\# Datasets & CodeSearchNet  \\\hline
			Lexical & 59.27\% &  51.34\% & 50.22\% & 77.60\% & 64.39\% & 71.25\%  & 53.39\%  & 49.17\%  \\\hline
			Syntactic & 40.73\%& 48.66\%   & 49.78\%& 21.73\%& 32.96\% & 26.24\% & 46.12\% & 50.83\%\\\hline
			Semantic & -- & --  & --& 0.67\% & 2.65\% & 2.51\% & 0.49\% & -- \\\hline 
	\end{tabular}}
	\label{tab:lssperc}
\end{table*}

\subsection{Confirm the Validity of \method in Light of the Out-of-distribution Issues}
Based on the high-level steps introduced in Section~\ref{sec:ood}, we present the technical details from two key aspects: data generation and labeling. 

\vspace*{3pt}
\noindent
\textbf{\textit{Data Generation.}}\, First, we perform the \textit{Reduce} operation to generate code fragments for each program in the training set.
Then, for each code fragment, we generate additional programs by randomly removing or mutating its constituent expressions. Such programs resemble the code produced out of the \textit{Mutate} step.

\vspace*{3pt}
\noindent
\textbf{\textit{Data Labeling.}}\, 
Given the amount of work, we attempt to recruit all undergraduate, master and PhD students from Math, Physics, Chemistry, Electronic Engineering, and Computer Science department at our university for the labeling task. In the end, we find 458 students in total and each student has at least one year of programming experience. We also hired 84 professional developers via Amazon Mechanical Turk. Here, we describe the set-up of this experiment in detail. First, for each generated program, we assign it the label of the program from which the \whet is derived. Second, we ask participants to confirm the assigned label for each generated program. Specifically, participants are required to assess based on their intuition if the assigned label correlates to the generated program. An example we gave to all participants is that the \whet of the running example correlates well to the label \texttt{addItem}, therefore, the label should be accepted for the \whet as a stand-alone program. To ensure the quality of labeling, we present every generated program with the assigned label to two participants, and we will only approve the program if both agree to accept the label, otherwise, the program is rejected. We also assign all rejected programs a new label that has not appeared in the original training set. 


We randomly select either training or validation set to add each approved program. We keep the same ratio between the added training programs and validation programs as the original dataset. To keep the added data balanced, we also randomly sample a rejected program to accompany each approved program. 
Table~\ref{tab:augnum} shows the number of programs we add including both the confirmed and rejected into Java-small, Java-med, C\# Datasets, and CodeSearchNet. We disregard Java-large because it demands excessive human effort to label enough programs given the size of its training set (more than 15 million programs). In addition, Java-small and Java-med should already contain enough data points to validate model behavior, in turn, \method's approach.

\begin{table*}[t]
		\captionsetup{skip=1pt}
		\caption{The number of programs that we add into each dataset for retraining. 
		}
		\centering
		\adjustbox{max width=1\textwidth}{
			\begin{tabular}{l|c|c|c|c}
				\hline
				& Java-small & Java-med &C\# Datasets & CodeSearchNet   \\\hline
				\#original programs -- training    &  665,115 & 3,004,536 & 130,101 & 454,451 \\
				\#added programs  -- training     &  336,938 & 1,106,288 & 46,862 & 267,922 \\\hline
				\#original programs -- validation    &  23,505 & 410,699 & 21,594 & 15,328 \\
				\#added programs -- validation     &  11,908 & 151,222 & 7,778 & 9,036 \\\hline
			\end{tabular}}
		\label{tab:augnum}
\vspace{-5pt}
\end{table*}

\vspace*{3pt}
\noindent
\textbf{\textit{Results.}}\, 
First, we validate the accuracy of all evaluated models on the augmented datasets. We find that all retrained models become more accurate even though the increase of accuracy is negligible (Table~\ref{app-tab:aug1}-\ref{app-tab:aug3}  in the supplemental material). We deploy \method to explain the retrained models, and then compare the two sets of \whet before and after the retraining. It is worth mentioning we also set out to validate Sivand's approach~\cite{rabin2021understanding} using the same original and retrained models. Considering the manually labeled programs also resemble those with which Sivand would have queried the model, thus, in the same way we can evaluate if 
Sivand's approach is still valid after the distribution shift between the training and test data is erased. Table~\ref{tab:augComp} presents the results. Regarding \method, we find that an overwhelming majority of the programs in every dataset has the same \whet before and after the retraining (the row designated by ``Matched''), confirming that those are indeed the \whet models use for prediction with or without the seemingly distribution shift. ``Related'' represents the \whet models use before and after the retraining come from the same code fragment. They differ at the level of expressions, for example, using retrained models, \method manages to reduce more features from the previous \whet at the \textit{Mutate} step or vice versa. ``Distinct'' means retrained models use entirely different \whet than the original model. Apparently, this raises issues despite the tiny number of programs that falls into this category. To dig deeper, we discover that all such \whet met the \textit{sufficient} and \textit{necessary} requirement before, only to get discarded due to the violation of the \textit{minimum} requirement. The reason that they become the \whet for retrained models is simple:
the then \whet that original models use gets bloated after the retraining (\ie, \textit{Mutate} operation could not reduce so many features as before), making \whet the minimum. Nevertheless, the ``Distinct'' (as well as the ``Related'') scenario does not invalidate \method's approach considering the change of model behavior, to a minor degree, is to be expected after retraining. Overall, we conclude \method's approach is valid in light of the out-of-distribution issues. 

On the contrary, Sivand's explanations change significantly after the retraining. In particular, only around 40\% of the \whet that retrained models use are the same as those that original models use. There is also a large number of programs on each dataset that falls into the ``Distinct'' category. Given that Sivand's approach does not have a similar notion of minimality, ``Distinct'' \whet suggests a drastic change in Sivand's explanations after retraining,
Overall, the results strongly indicate the distribution shift heavily influences Sivand.
This is yet another fundamental flaw of~\citet{rabin2021understanding} in addition to ignoring the \textit{necessary} requirement when producing model explanations. 

\vspace*{3pt}
\noindent
\textbf{\textit{Discussion.}}\, A natural question arises: for the original models (without retraining) how is it possible for the \whet, which is far from the data that models have seen during training, still within the learned distribution? We believe it is because of the very intrinsic limitation of models’ that they only learn local, lexical/syntactic features from training programs, In other words, even with a whole program to learn, models still extract patterns from simple parts of the program which the \whet is likely to resemble, as a result, they lie in the distribution that models learned.

\begin{table*}[t]
		\captionsetup{skip=1pt}
		\caption{The percentage of programs for which \whet are identical (designated by ``Matched''), come from the same code fragment (designated by ``Related''), or entirely different (designated by ``Distinct'') before and after retraining for Sivand and \method.}
		\centering
		\adjustbox{max width=1\textwidth}{
			\begin{tabular}{c|c|cc|cc|c|c}
				\hline			
				\multirow{2}{*}{Approach} & \multirow{2}{*}{Comparison} & \multicolumn{2}{c|}{code2vec} & \multicolumn{2}{c|}{seq-GNN}  & GGNN & CodeBERT \\
				& & Java-small & Java-med & Java-small & Java-med &C\# Datasets & CodeSearchNet   \\\hline
				\multirow{3}{*}{Sivand} & Matched & 46.2\% & 37.8\% & 44.9\% & 41.2\% & 38.4\% & 41.9\% \\
				 & Related & 23.9\% & 30.8\% & 32.5\% & 29.5\% & 26.4\% & 36.7\% \\
				 &  Distinct & 29.9\% & 31.4\% & 22.6\% & 29.3\% & 35.2\% & 21.4\% \\\hline
				\multirow{3}{*}{\method} & Matched & 92.5\% & 96.8\% & 94.8\% & 97.3\% & 93.9\% & 96.5\% \\
				&  Related & 3.6\% & 2.4\% & 2.1\% & 1.7\% & 4.3\% & 2.2\% \\
				 &  Distinct & 3.9\% & 0.8\% & 3.1\% & 1.0\% & 1.8\% & 1.3\% \\\hline
			\end{tabular}}
		\label{tab:augComp}
\vspace{-5pt}
\end{table*}


\subsection{Can \method Find the Global Minimum Features}
\label{subsec:bfs}

In this experiment, we set out to confirm the minimality of the \whet found by \tool. As explained earlier, finding the global minimum \whet requires exhausting all subsequences of the token sequence of the input method. Therefore, to lessen the computational burden, we limit the evaluation data for this experiment to those methods that have either a small set of tokens or small \whet found by \tool (Table~\ref{tab:me}). We also diversify the data \wrt the size of programs and \whet to make our results unbiased. To confirm the global minimality of the \whet for the selected methods, we exhaust all subsequences of the token sequence of the input method up to the size of the \whet, specifically, we start from the subsequences of size one, and gradually increase to the size of the \whet. Note that this is also the way in which we verify the \whet for our running example. To cope with such a heavy computational burden, we exhaustively generate the test candidates for each program beforehand in a parallel fashion; then we place them into separate batches to fully exploit the potential of our GPUs (\ie, $\approx$10K test candidates per second per GPU). The whole experiment takes almost three months to complete. 

We find that for 79,208 \whet that \tool identified, the brute-force search never finds a single instance where the \whet is smaller. Our results speak volume to the precision of our \textit{Reduce} and \textit{Mutate} method in finding the critical features that models use for prediction even if they are not guaranteed to be the global minimum.

\begin{table}[t]
	\captionsetup{skip=1pt}
	\caption{Statistics on the selected methods for the minimality experiment.}
	\centering
	\adjustbox{max width=1\textwidth}{
		\begin{tabular}{c|cccc|cccc|c}
			\hline
			\multirow{2}{*}{Models} & \multicolumn{4}{c|}{\#statements in methods} &  \multicolumn{4}{c|}{\#tokens in \whet}  & \multirow{2}{*}{\#methods in total}  \\
			& Min & Max & Mean & Median & Min & Max & Mean & Median &   \\\hline
			code2vec & 1.0 & 12.0 & 7.4 & 5.0 & 2.0 & 17.0 & 9.4 & 8.0 & 21,324 \\\hline
			code2seq & 1.0 & 15.0 & 7.2 & 6.0 & 2.0 & 16.0 & 10.1 & 9.0 & 18,105 \\\hline
			seq-GNN  & 1.0 & 12.0 & 5.8 & 5.0 & 2.0 & 16.0 & 8.2 & 7.0 & 19,751 \\\hline
			extreme  & 1.0 & 14.0 & 8.1 & 6.0 & 2.0 & 15.0 & 7.5 & 7.0 & 20,028 \\\hline			
	\end{tabular}}
	\label{tab:me}
\vspace{-5pt}
\end{table}

\subsection{Can Other Explainability Methods Find \Whet}

\begin{table*}[t]
	\captionsetup{skip=1pt}
	\caption{Percentage of programs whose \whet are fully covered by explainability methods.}
	\centering
	\adjustbox{max width=1\textwidth}{
		\begin{tabular}{c|ccccc|ccccc|ccccc}
			\hline
			\multirow{2}{*}{Models} & \multicolumn{5}{c|}{Integrated Gradients (Top-N\%)} & \multicolumn{5}{c|}{Attention-based (Top-N\%)}  & \multicolumn{5}{c}{SHAP (Top-N\%)}  \\
			& 10\% &  30\% & 50\% & 70\% & 90\% & 10\% &  30\% & 50\% & 70\% & 90\% & 10\% &  30\% & 50\% & 70\% & 90\%  \\\hline
			code2vec    & 8.23 & 11.42 & 13.85 & 21.69 & 27.93  & 52.50 & 60.00 & 62.50  & 64.91 & 69.22  & 4.38  & 5.02  & 6.37   & 15.24 & 55.19  \\\hline
			seq-GNN     & 1.74 & 6.32 & 10.17  & 17.90 & 25.54  & 41.36 & 45.81 & 51.27  & 56.92 & 60.14 & 6.12  & 9.33  & 12.02  & 38.24 & 60.58 \\\hline
			GGNN    & 1.95 & 4.56 & 9.23   & 15.68 & 23.18 & 38.23 & 42.32 & 52.53  & 57.21 & 61.21 & 4.72  & 7.64  & 11.72   & 21.45 & 65.23\\\hline
			CodeBERT     & 6.23 & 10.33 & 13.12  & 20.90 & 29.50  & 41.97 & 48.44 & 52.02 & 58.19 & 63.12 &6.20 & 11.54 & 15.67 & 27.30 & 73.56 \\\hline
	\end{tabular}}
	\label{tab:cpall}
\vspace{-5pt}
\end{table*}

In this experiment, we evaluate whether or not some of the most prominent attribution methods can also find the \whet that \tool finds. We choose Integrated Gradients~\cite{Sundararajan2017} and SHAP~\cite{LundbergNIPS2017}, both of which have been widely used for explaining the predictions of image and natural language models. 
Worth noting that the two methods are typically contrastive (\ie, they account for deviations from a baseline), we design a baseline in which the embedding for each token (or node) is set to 0. This is the conventional approach followed by the explainability literature. Additionally, we choose the attention mechanism~\cite{bahdanau2014neural}, which makes models pay greater attention to certain factors (\eg, elements in an input) when processing the input data. Therefore, the features that are heavily attended to can be deemed as an explanation naturally. In theory, the three explainability methods may find the \whet that \tool does not, therefore, we only use 
the 79,208 programs whose \whet are already verified in Section~\ref{subsec:bfs}.


First, we pick input features with top-$N$\% highest attribution scores computed by the method. Then, we show the percentage of the programs that are fully covered by those input features (Table~\ref{tab:cpall}).
Quite unexpectedly, the attention mechanism turns out to be the top performer in this experiment. Up to the features that receive the top-70\% attribution scores, attention beats the other two methods by a significant margin across all evaluated models. Another interesting observation we made about the attention is it can already cover a good amount of programs when only using the top-10\% of the input features and the increase of the coverage slows down when more features are selected. This indicates the important role the attention plays in helping models to learn the right features for prediction. As for the other two methods, SHAP's advantage over Integrated Gradients only stands out after the 50\% mark, otherwise, their numbers are in the same ballpark. Overall, it's evident that none of those explanation methods are good at finding the \whet. 
This is an expected outcome because the criteria by which \method evaluates the importance of input features is fundamentally different than the above explanation methods. Specifically, \method uses \textit{sufficient}, \textit{necessary} and \textit{Minimum} requirement whereas
the other three methods quantify the influence of features in mathematical terms (\eg, gradients, SHAP values, or attention scores).

\subsection{Explaining the Predictions of Code Models}
\label{subsec:ustu}

We conduct a user study to evaluate our technique for producing explanations for the prediction of all four models. As described in Section~\ref{sec:exam}, our technique finds training methods similar to a test method based on the AST distance between their \whet. We also adopt a simple baseline that searches for training methods based on the AST distance of the entire method. 

\begin{figure*}[thb!]
	\centering
	\captionsetup{justification=centering}
	
	\begin{subfigure}{0.308\textwidth}
		\lstset{style=mystyle}
\begin{lstlisting}[morekeywords={Object, String}]
String[] reverse(final String[] 
    array) {
  final String[] newArray = 
      new String[array.length];
  for(int index = 0; index < 
      array.length; index++) {
    @newArray[array.length -@ 
      @index - 1] = array[index];@
  }
  return newArray;
}
\end{lstlisting}
		\caption{A test method.}
		\label{fig:ori}
	\end{subfigure}
	\,\,
	\begin{subfigure}{0.35\textwidth}
		\lstset{style=mystyle}
\begin{lstlisting}[morekeywords={Object, String}]
FastStringBuffer reverse() {
  final int limit = count / 2;
  
  for(int i = 0; i < limit; ++i) {
    char c = value[i];
    value[i] = value[count - i - 1];
    @value[count - i - 1] = c;@
  }

  return this;
}
\end{lstlisting}
		\caption{The closet found by our technique.}
		\label{fig:pat}
	\end{subfigure}
	\,\,
	\begin{subfigure}{0.302\textwidth}
		\lstset{style=mystyle}
\begin{lstlisting}[morekeywords={Object, String}]
static <T> SList<T> reverse(
    SList<T> r) {
  SList<T> res = null;

  for(; r!=null; res=cons(r.car,
      res), r = r.cdr)
    System.out.println("");


  return res;
}
\end{lstlisting}
		\caption{The closet found by baseline.}
		\label{fig:base}
	\end{subfigure}	
	
	\caption{An example prediction produced by our technique and baseline.}
	\label{fig:explan}
\end{figure*}

\begin{figure}[th]
	\centering
	\begin{tcolorbox}[enhanced,width=\linewidth,size=fbox,colback=gray!5,
		fontupper=\small,drop shadow southwest,sharp corners]
		\begin{enumerate}[leftmargin=*]
			\item  The system finds training examples that have similar structure (or substructure) to the test example.
			\vspace{3pt}
			\item  Given the similarity of the structure (or substructure), explanations that the system produce are accurate.
			\vspace{3pt}
			\item  As a user of the model, I find the system helpful in provide an insight of how models work. 
			\vspace{3pt}
			\item  As a user of the model, I feel comfortable using this model along with the system providing the explanations.
			\vspace{3pt}
			\item  As a designer of the model, I find the system helpful in debugging models' mispredictions and improves their quality.			
		\end{enumerate}
		
	\end{tcolorbox}
	\caption{Questionnaire used in the study.}	
	\label{fig:questionnaire}
\end{figure}

\vspace*{4pt}
\noindent
\textbf{\textit{Procedure.}}\, Each participant is given two sets of methods. One set contains 4 methods randomly selected from the test set of the three datasets, and each method is accompanied with 5 methods found to the closet by our technique on the corresponding training set. 
As an example, we show a test method (Figure~\ref{fig:ori}) accompanied by a training method (Figure~\ref{fig:pat}) that is the closest found by our technique. We highlight their \whet in the shadow box. Similarly, the other set presents the same content except that the 5 methods are found by the baseline approach. An example is shown in Figure~\ref{fig:base}. After reviewing a set of methods, each participant is given a set of questions to answer. The questionnaire is depicted in Figure~\ref{fig:questionnaire}. For each question, participants are given 5 choices to make: strong disagree, disagree, neutral, agree, and strongly agree. Each choice is interpreted as a score starting from 1 for strongly disagree all the way to 5 for strong agree. We also conduct an interview after the study during which we encourage them to give any comments/suggestions that they may have.

\vspace*{4pt}
\noindent
\textbf{\textit{Participants.}}\, We have recruited 32 data scientists from a high-tech company through an internal email thread and paid them \$50 each for participating the study. Each participant has at least one year of Java programming experience. All of them have at least two years of experience in developing machine learning models at work. We have explained to each participant beforehand that our technique finds the training programs from which models learned their parameters for predicting the name of a particular test method, and they are asked to rate the quality of the explanations.

\vspace*{4pt}
\noindent
\textbf{\textit{Results.}}\, Table~\ref{tab:questin} shows the results of the questionnaire. In particular, each row contains the average score to a question for our technique and baseline approach. It is clear that participants have rated our technique consistently higher than the baseline approach by a notable margin. In terms of the actual score, participants have given more than 3 to all questions indicating an overall positive attitude toward our technique. Specifically, for Questions 1 and 2, participants are very positive about the accuracy of the predictions found by our technique. For Question 3 and 5, participants also agree that the system would benefit both users and designers of the model by providing the rationale of model predictions. One participant even raises the possibility of performing re-training with examples generated based on the \whet to improve the model accuracy. Question 4 is the only one to which participants reacted a bit more negatively. Therefore, we focus on this question in our post-study interview. 

\begin{table}[t]
	\caption{Participants' responses.}
	\centering
	\adjustbox{max width=1\textwidth}{
		\begin{tabular}{c|c|c}
			\hline
			Question & Our technique &   Baseline  \\\hline
			1 & 4.7  & 2.1  \\\hline
			2 & 4.5  & 1.8 \\\hline
			3 & 3.9  & 1.8  \\\hline
			4 & 3.1  & 1.2 \\\hline
			5 & 3.6  & 1.5 \\\hline
	\end{tabular}}
	\label{tab:questin}
\end{table}

The main message we received from participants who give particularly low scores to this question is that even though the system gives accurate explanations, they won't be very helpful when models only seem to use almost random textual properties to predict. For example, one participant explained that once is explanation is given, albeit having a similar substructure of the test method, users still have to manually examine the rest of the program in order to decide if the explanation can really match the prediction of the test method. Therefore, having a model that learns the right features is also necessary. 

As standard in user studies, we performed a t-test to evaluate the statistical significance of our results. The p-value of a two-tailed
t-test (assuming potentially unequal variance) comparing our technique against the baseline approach is $2 \times 10^{-6}$. This means
the probability that our technique has no influence on classification accuracy is less than 1 in 100,000, indicating our results are statistically significant.

\section{Related Work}
\label{sec:rel}

The work closest to ours is~\citet{rabin2021understanding}, which discovers that input programs can be reduced to significantly smaller code snippets for which models still make the correct predictions. Our work in fact surpasses~\citet{rabin2021understanding} at many levels. Conceptually, we propose \textit{necessary} requirement that \whet has to satisfy in addition to the \textit{sufficient} requirement. Technically, we show DD, which powers Sivand~\cite{rabin2021understanding}, suffers from models' non-monotony. In contrast, our technique \textit{Reduce} and \textit{Mutate} is more precise in identifying the \whet evidenced by our extensive evaluation. Empirically, we demonstrate the validity of \method's approach in light of out-of-the distribution issues, whereas, Sivand is fundamentally flawed in dealing with the distribution shift between the training and query programs. Finally, we also present a novel application of the \whet while~\citet{rabin2021understanding} have not demonstrated any utility of their method. For the remainder of this section, we survey two lines of related work: attribution methods and models of code.

\subsection{Attribution Methods}
In machine learning field, attribution methods are usually classified into two categories: Perturbation-based and backpropagation-based. The former generates explanations by iteratively probing a trained machine learning model with different variations of the inputs. As a few representatives,~\citet{zeiler2011adaptive} visualized the neural activations of individual layers of a deep convolutional network by occluding different segments of the input image and generating visualizations using a deconvolution network (DeConvNet).~\citet{Luisa2017} use a conditional sampling
based multi-variate approach to generate more targeted explanations on image classification CNNs. The Interpretability Randomization Test (IRT) and the One-Shot Feature Test
(OSFT) introduced by~\citet{Collin2020} focuses on discovering important features by replacing the features with uninformative counter-factuals. To derive a representation that is understandable by humans, LIME~\cite{Ribeiro2016} tries to find important contiguous superpixels (a patch of pixels) in a source image towards the output class.~\citet{LundbergNIPS2017} present a unified framework, SHAP, which computes individual feature contributions towards that output prediction. 

As for the backpropagation-based methods, Saliency maps~\cite{Simonyan14deepinside} construct attributions by taking the absolute value of the partial derivative of the target output with respect to the input features.
Gradient$^\ast$Input~\cite{Shrikumar16} was then proposed to improve the sharpness of the attribution maps. The attribution is computed by taking the (signed) partial derivatives of the output with respect to the input and multiplying them with the input itself. Integrated Gradients~\cite{Sundararajan2017}, similarly to Gradient$^\ast$Input, computes the partial derivatives of the output with respect to each input feature. However, while Gradient$^\ast$Input computes a single derivative, evaluated at the provided input $x$, Integrated Gradients computes the average gradient while the input varies along a linear path from a baseline $\hat{x}$ to $x$. Layer-wise Relevance Propagation (LRP) introduced by~\citet{Sebastian40} is used to find relevance scores for individual features in the input data by decomposing the output predictions of the DNN. DeepLIFT~\cite{Shrikumar6}, similar to LRP, assigns each unit $i$ an attribution that represents the relative effect of the unit activated at the original network input $x$ compared to some other reference input.

\subsection{Models of Code}
Machine learning methods have been applied to a variety of programming language tasks such as method name prediction~\cite{fernandes2018structured, alon2018code2seq, Alon2019code2vec, Wang101145, defreez2018poster, defreez2018path}, bug detection~\cite{wang2020learning, allamanis2018learning}, program repair~\cite{chen2019sequencer,Georgios, dinella2019hoppity}, and type inference~\cite{wei2019lambdanet, Allamanis3385997}. Below we survey a few notable representatives. GGNN is the first to learn program embeddings from graph representations of source code. code2vec and code2seq are among the first in predicting method names in large-scale, cross-projecting settings.~\citet{defreez2018path} is the first to use static program traces to learn function embeddings.~\citet{dy} is the first to embed programs with their executions.~\citet{chen2019sequencer} and~\citet{dinella2019hoppity} utilize sequence and graph models for program repair.~\citet{Jianan} and ~\citet{Xujie3} are the noteworthy efforts in inferring loop invariant with deep learning models. CodeBERT~\cite{feng-etal-2020-codebert} learns general-purpose representations that support downstream natural language-programming language applications.

\vspace{11pt}
\section{Conclusion}
\label{sec:conc}


In this paper, we present \method, an explanation method for models of code. Conceptually, we formalize the defining features that models use for predicting the label of input programs. Technically, we develop \textit{Reduce} and \textit{Mutate} and its implementation \tool, which we use to explain code2vec, \motimodel, GGNN, and CodeBERT.
We found that (1) \tool is efficient and effective in finding \whet; 
(2) through retraining, we confirm the validity of \method amid the distribution shift between training and queried programs;
(3) all models use simple syntactic or even lexical properties for prediction; 
(4) some of the most popular attribution methods routinely miss out on the \whet; 
(5) we present an example application of the revealed features: providing explanations for predictions of code models. Through a user study, we have shown the usefulness of our \whet-based explanation method. 

\bibliography{reference}


\begin{thebibliography}{38}


\ifx \showCODEN    \undefined \def \showCODEN     #1{\unskip}     \fi
\ifx \showDOI      \undefined \def \showDOI       #1{#1}\fi
\ifx \showISBNx    \undefined \def \showISBNx     #1{\unskip}     \fi
\ifx \showISBNxiii \undefined \def \showISBNxiii  #1{\unskip}     \fi
\ifx \showISSN     \undefined \def \showISSN      #1{\unskip}     \fi
\ifx \showLCCN     \undefined \def \showLCCN      #1{\unskip}     \fi
\ifx \shownote     \undefined \def \shownote      #1{#1}          \fi
\ifx \showarticletitle \undefined \def \showarticletitle #1{#1}   \fi
\ifx \showURL      \undefined \def \showURL       {\relax}        \fi
\providecommand\bibfield[2]{#2}
\providecommand\bibinfo[2]{#2}
\providecommand\natexlab[1]{#1}
\providecommand\showeprint[2][]{arXiv:#2}

\bibitem[\protect\citeauthoryear{Allamanis, Barr, Ducousso, and Gao}{Allamanis
  et~al\mbox{.}}{2020}]%
        {Allamanis3385997}
\bibfield{author}{\bibinfo{person}{Miltiadis Allamanis},
  \bibinfo{person}{Earl~T. Barr}, \bibinfo{person}{Soline Ducousso}, {and}
  \bibinfo{person}{Zheng Gao}.} \bibinfo{year}{2020}\natexlab{}.
\newblock \showarticletitle{Typilus: Neural Type Hints}
  \emph{(\bibinfo{series}{PLDI '20})}. \bibinfo{publisher}{Association for
  Computing Machinery}, \bibinfo{address}{New York, NY, USA},
  \bibinfo{pages}{91–105}.
\newblock


\bibitem[\protect\citeauthoryear{Allamanis, Brockschmidt, and
  Khademi}{Allamanis et~al\mbox{.}}{2018}]%
        {allamanis2018learning}
\bibfield{author}{\bibinfo{person}{Miltiadis Allamanis}, \bibinfo{person}{Marc
  Brockschmidt}, {and} \bibinfo{person}{Mahmoud Khademi}.}
  \bibinfo{year}{2018}\natexlab{}.
\newblock \showarticletitle{Learning to Represent Programs with Graphs}. In
  \bibinfo{booktitle}{\emph{International Conference on Learning
  Representations}} \emph{(\bibinfo{series}{ICLR '18})}.
\newblock


\bibitem[\protect\citeauthoryear{Alon, Levy, and Yahav}{Alon
  et~al\mbox{.}}{2019a}]%
        {alon2018code2seq}
\bibfield{author}{\bibinfo{person}{Uri Alon}, \bibinfo{person}{Omer Levy},
  {and} \bibinfo{person}{Eran Yahav}.} \bibinfo{year}{2019}\natexlab{a}.
\newblock \showarticletitle{code2seq: Generating Sequences from Structured
  Representations of Code}. In \bibinfo{booktitle}{\emph{International
  Conference on Learning Representations}} \emph{(\bibinfo{series}{ICLR '19})}.
\newblock


\bibitem[\protect\citeauthoryear{Alon, Zilberstein, Levy, and Yahav}{Alon
  et~al\mbox{.}}{2019b}]%
        {Alon2019code2vec}
\bibfield{author}{\bibinfo{person}{Uri Alon}, \bibinfo{person}{Meital
  Zilberstein}, \bibinfo{person}{Omer Levy}, {and} \bibinfo{person}{Eran
  Yahav}.} \bibinfo{year}{2019}\natexlab{b}.
\newblock \showarticletitle{code2vec: Learning Distributed Representations of
  Code}.
\newblock \bibinfo{journal}{\emph{Proceedings of the ACM on Programming
  Languages}} \bibinfo{volume}{3}, \bibinfo{number}{POPL}
  (\bibinfo{year}{2019}), \bibinfo{pages}{1--29}.
\newblock


\bibitem[\protect\citeauthoryear{Bach, Binder, Montavon, Klauschen, Müller,
  and Samek}{Bach et~al\mbox{.}}{2015}]%
        {Sebastian40}
\bibfield{author}{\bibinfo{person}{Sebastian Bach}, \bibinfo{person}{Alexander
  Binder}, \bibinfo{person}{Grégoire Montavon}, \bibinfo{person}{Frederick
  Klauschen}, \bibinfo{person}{Klaus-Robert Müller}, {and}
  \bibinfo{person}{Wojciech Samek}.} \bibinfo{year}{2015}\natexlab{}.
\newblock \showarticletitle{On Pixel-Wise Explanations for Non-Linear
  Classifier Decisions by Layer-Wise Relevance Propagation}.
\newblock \bibinfo{journal}{\emph{PLOS ONE}} \bibinfo{volume}{10},
  \bibinfo{number}{7} (\bibinfo{date}{07} \bibinfo{year}{2015}),
  \bibinfo{pages}{1--46}.
\newblock


\bibitem[\protect\citeauthoryear{Bahdanau, Cho, and Bengio}{Bahdanau
  et~al\mbox{.}}{2015}]%
        {bahdanau2014neural}
\bibfield{author}{\bibinfo{person}{Dzmitry Bahdanau},
  \bibinfo{person}{Kyunghyun Cho}, {and} \bibinfo{person}{Yoshua Bengio}.}
  \bibinfo{year}{2015}\natexlab{}.
\newblock \showarticletitle{Neural machine translation by jointly learning to
  align and translate}. In \bibinfo{booktitle}{\emph{International Conference
  on Learning Representations}} \emph{(\bibinfo{series}{ICLR '15})}.
\newblock


\bibitem[\protect\citeauthoryear{Burns, Thomason, and Tansey}{Burns
  et~al\mbox{.}}{2020}]%
        {Collin2020}
\bibfield{author}{\bibinfo{person}{Collin Burns}, \bibinfo{person}{Jesse
  Thomason}, {and} \bibinfo{person}{Wesley Tansey}.}
  \bibinfo{year}{2020}\natexlab{}.
\newblock \showarticletitle{Interpreting Black Box Models via Hypothesis
  Testing}. In \bibinfo{booktitle}{\emph{Proceedings of the 2020 ACM-IMS on
  Foundations of Data Science Conference}} (Virtual Event, USA)
  \emph{(\bibinfo{series}{FODS '20})}. \bibinfo{publisher}{Association for
  Computing Machinery}, \bibinfo{address}{New York, NY, USA},
  \bibinfo{pages}{47–57}.
\newblock
\showISBNx{9781450381031}


\bibitem[\protect\citeauthoryear{Chen, Kommrusch, Tufano, Pouchet, Poshyvanyk,
  and Monperrus}{Chen et~al\mbox{.}}{2019}]%
        {chen2019sequencer}
\bibfield{author}{\bibinfo{person}{Zimin Chen}, \bibinfo{person}{Steve~James
  Kommrusch}, \bibinfo{person}{Michele Tufano}, \bibinfo{person}{Louis-No{\"e}l
  Pouchet}, \bibinfo{person}{Denys Poshyvanyk}, {and} \bibinfo{person}{Martin
  Monperrus}.} \bibinfo{year}{2019}\natexlab{}.
\newblock \showarticletitle{Sequencer: Sequence-to-sequence learning for
  end-to-end program repair}.
\newblock \bibinfo{journal}{\emph{IEEE Transactions on Software Engineering}}
  (\bibinfo{year}{2019}).
\newblock


\bibitem[\protect\citeauthoryear{DeFreez, Thakur, and
  Rubio-Gonz{\'a}lez}{DeFreez et~al\mbox{.}}{2018a}]%
        {defreez2018poster}
\bibfield{author}{\bibinfo{person}{Daniel DeFreez}, \bibinfo{person}{Aditya
  Thakur}, {and} \bibinfo{person}{Cindy Rubio-Gonz{\'a}lez}.}
  \bibinfo{year}{2018}\natexlab{a}.
\newblock \showarticletitle{Poster: Path-Based Function Embeddings}. In
  \bibinfo{booktitle}{\emph{2018 IEEE/ACM 40th International Conference on
  Software Engineering: Companion}} \emph{(\bibinfo{series}{ICSE-Companion
  '18})}. IEEE, \bibinfo{pages}{430--431}.
\newblock


\bibitem[\protect\citeauthoryear{DeFreez, Thakur, and
  Rubio-Gonz{\'a}lez}{DeFreez et~al\mbox{.}}{2018b}]%
        {defreez2018path}
\bibfield{author}{\bibinfo{person}{Daniel DeFreez}, \bibinfo{person}{Aditya~V
  Thakur}, {and} \bibinfo{person}{Cindy Rubio-Gonz{\'a}lez}.}
  \bibinfo{year}{2018}\natexlab{b}.
\newblock \showarticletitle{Path-based function embedding and its application
  to error-handling specification mining}. In
  \bibinfo{booktitle}{\emph{Proceedings of the 2018 26th ACM Joint Meeting on
  European Software Engineering Conference and Symposium on the Foundations of
  Software Engineering}} \emph{(\bibinfo{series}{ESEC/FSE '18})}.
  \bibinfo{pages}{423--433}.
\newblock


\bibitem[\protect\citeauthoryear{Dinella, Dai, Li, Naik, Song, and
  Wang}{Dinella et~al\mbox{.}}{2019}]%
        {dinella2019hoppity}
\bibfield{author}{\bibinfo{person}{Elizabeth Dinella}, \bibinfo{person}{Hanjun
  Dai}, \bibinfo{person}{Ziyang Li}, \bibinfo{person}{Mayur Naik},
  \bibinfo{person}{Le Song}, {and} \bibinfo{person}{Ke Wang}.}
  \bibinfo{year}{2019}\natexlab{}.
\newblock \showarticletitle{Hoppity: Learning Graph Transformations to Detect
  and Fix Bugs in Programs}. In \bibinfo{booktitle}{\emph{International
  Conference on Learning Representations}} \emph{(\bibinfo{series}{ICLR '19})}.
\newblock


\bibitem[\protect\citeauthoryear{Feng, Guo, Tang, Duan, Feng, Gong, Shou, Qin,
  Liu, Jiang, and Zhou}{Feng et~al\mbox{.}}{2020}]%
        {feng-etal-2020-codebert}
\bibfield{author}{\bibinfo{person}{Zhangyin Feng}, \bibinfo{person}{Daya Guo},
  \bibinfo{person}{Duyu Tang}, \bibinfo{person}{Nan Duan},
  \bibinfo{person}{Xiaocheng Feng}, \bibinfo{person}{Ming Gong},
  \bibinfo{person}{Linjun Shou}, \bibinfo{person}{Bing Qin},
  \bibinfo{person}{Ting Liu}, \bibinfo{person}{Daxin Jiang}, {and}
  \bibinfo{person}{Ming Zhou}.} \bibinfo{year}{2020}\natexlab{}.
\newblock \showarticletitle{{C}ode{BERT}: A Pre-Trained Model for Programming
  and Natural Languages}. In \bibinfo{booktitle}{\emph{Findings of the
  Association for Computational Linguistics: EMNLP 2020}}.
  \bibinfo{publisher}{Association for Computational Linguistics},
  \bibinfo{address}{Online}, \bibinfo{pages}{1536--1547}.
\newblock


\bibitem[\protect\citeauthoryear{Fernandes, Allamanis, and
  Brockschmidt}{Fernandes et~al\mbox{.}}{2019}]%
        {fernandes2018structured}
\bibfield{author}{\bibinfo{person}{Patrick Fernandes},
  \bibinfo{person}{Miltiadis Allamanis}, {and} \bibinfo{person}{Marc
  Brockschmidt}.} \bibinfo{year}{2019}\natexlab{}.
\newblock \showarticletitle{Structured Neural Summarization}. In
  \bibinfo{booktitle}{\emph{International Conference on Learning
  Representations}} \emph{(\bibinfo{series}{ICLR '19})}.
\newblock


\bibitem[\protect\citeauthoryear{Glorot and Bengio}{Glorot and Bengio}{2010}]%
        {pmlrv9glorot10a}
\bibfield{author}{\bibinfo{person}{Xavier Glorot} {and} \bibinfo{person}{Yoshua
  Bengio}.} \bibinfo{year}{2010}\natexlab{}.
\newblock \showarticletitle{Understanding the difficulty of training deep
  feedforward neural networks}. In \bibinfo{booktitle}{\emph{Proceedings of the
  Thirteenth International Conference on Artificial Intelligence and
  Statistics}} \emph{(\bibinfo{series}{Proceedings of Machine Learning
  Research})}, \bibfield{editor}{\bibinfo{person}{Yee~Whye Teh} {and}
  \bibinfo{person}{Mike Titterington}} (Eds.), Vol.~\bibinfo{volume}{9}.
  \bibinfo{publisher}{PMLR}, \bibinfo{address}{Chia Laguna Resort, Sardinia,
  Italy}, \bibinfo{pages}{249--256}.
\newblock


\bibitem[\protect\citeauthoryear{Hooker, Erhan, Kindermans, and Kim}{Hooker
  et~al\mbox{.}}{2019}]%
        {ROAR}
\bibfield{author}{\bibinfo{person}{Sara Hooker}, \bibinfo{person}{Dumitru
  Erhan}, \bibinfo{person}{Pieter-Jan Kindermans}, {and} \bibinfo{person}{Been
  Kim}.} \bibinfo{year}{2019}\natexlab{}.
\newblock \showarticletitle{A Benchmark for Interpretability Methods in Deep
  Neural Networks}. In \bibinfo{booktitle}{\emph{Proceedings of the 33rd
  International Conference on Neural Information Processing Systems}}.
  \bibinfo{publisher}{Curran Associates Inc.}, \bibinfo{address}{Red Hook, NY,
  USA}, Article \bibinfo{articleno}{873}, \bibinfo{numpages}{12}~pages.
\newblock


\bibitem[\protect\citeauthoryear{Husain, Wu, Gazit, Allamanis, and
  Brockschmidt}{Husain et~al\mbox{.}}{2019}]%
        {husain2019codesearchnet}
\bibfield{author}{\bibinfo{person}{Hamel Husain}, \bibinfo{person}{Ho-Hsiang
  Wu}, \bibinfo{person}{Tiferet Gazit}, \bibinfo{person}{Miltiadis Allamanis},
  {and} \bibinfo{person}{Marc Brockschmidt}.} \bibinfo{year}{2019}\natexlab{}.
\newblock \showarticletitle{{CodeSearchNet} challenge: Evaluating the state of
  semantic code search}.
\newblock \bibinfo{journal}{\emph{arXiv preprint arXiv:1909.09436}}
  (\bibinfo{year}{2019}).
\newblock


\bibitem[\protect\citeauthoryear{John~Philip and West}{John~Philip and
  West}{1997}]%
        {john1997mathematical}
\bibfield{author}{\bibinfo{person}{D'Angelo John~Philip} {and}
  \bibinfo{person}{Douglas~Brent West}.} \bibinfo{year}{1997}\natexlab{}.
\newblock \bibinfo{booktitle}{\emph{Mathematical thinking: problem-solving and
  proofs}}.
\newblock \bibinfo{publisher}{Prentice-Hall}.
\newblock


\bibitem[\protect\citeauthoryear{Lundberg and Lee}{Lundberg and Lee}{2017}]%
        {LundbergNIPS2017}
\bibfield{author}{\bibinfo{person}{Scott~M Lundberg} {and}
  \bibinfo{person}{Su-In Lee}.} \bibinfo{year}{2017}\natexlab{}.
\newblock \showarticletitle{A Unified Approach to Interpreting Model
  Predictions}.
\newblock In \bibinfo{booktitle}{\emph{Advances in Neural Information
  Processing Systems 30}}, \bibfield{editor}{\bibinfo{person}{I.~Guyon},
  \bibinfo{person}{U.~V. Luxburg}, \bibinfo{person}{S.~Bengio},
  \bibinfo{person}{H.~Wallach}, \bibinfo{person}{R.~Fergus},
  \bibinfo{person}{S.~Vishwanathan}, {and} \bibinfo{person}{R.~Garnett}}
  (Eds.). \bibinfo{publisher}{Curran Associates, Inc.},
  \bibinfo{pages}{4765--4774}.
\newblock


\bibitem[\protect\citeauthoryear{Misherghi and Su}{Misherghi and Su}{2006}]%
        {hddpro}
\bibfield{author}{\bibinfo{person}{Ghassan Misherghi} {and}
  \bibinfo{person}{Zhendong Su}.} \bibinfo{year}{2006}\natexlab{}.
\newblock \showarticletitle{HDD: Hierarchical Delta Debugging}. In
  \bibinfo{booktitle}{\emph{Proceedings of the 28th International Conference on
  Software Engineering}} \emph{(\bibinfo{series}{ICSE ’06})}.
  \bibinfo{publisher}{Association for Computing Machinery},
  \bibinfo{address}{New York, NY, USA}, \bibinfo{pages}{142–151}.
\newblock


\bibitem[\protect\citeauthoryear{Rabin, Bui, Wang, Yu, Jiang, and
  Alipour}{Rabin et~al\mbox{.}}{2021a}]%
        {RABIN2021106552}
\bibfield{author}{\bibinfo{person}{Md~Rafiqul~Islam Rabin},
  \bibinfo{person}{Nghi~D.Q. Bui}, \bibinfo{person}{Ke Wang},
  \bibinfo{person}{Yijun Yu}, \bibinfo{person}{Lingxiao Jiang}, {and}
  \bibinfo{person}{Mohammad~Amin Alipour}.} \bibinfo{year}{2021}\natexlab{a}.
\newblock \showarticletitle{On the generalizability of Neural Program Models
  with respect to semantic-preserving program transformations}.
\newblock \bibinfo{journal}{\emph{Information and Software Technology}}
  \bibinfo{volume}{135} (\bibinfo{year}{2021}), \bibinfo{pages}{106552}.
\newblock
\showISSN{0950-5849}
\urldef\tempurl%
\url{https://doi.org/10.1016/j.infsof.2021.106552}
\showDOI{\tempurl}


\bibitem[\protect\citeauthoryear{Rabin, Hellendoorn, and Alipour}{Rabin
  et~al\mbox{.}}{2021b}]%
        {rabin2021understanding}
\bibfield{author}{\bibinfo{person}{Md~Rafiqul~Islam Rabin},
  \bibinfo{person}{Vincent~J Hellendoorn}, {and} \bibinfo{person}{Mohammad~Amin
  Alipour}.} \bibinfo{year}{2021}\natexlab{b}.
\newblock \showarticletitle{Understanding Neural Code Intelligence Through
  Program Simplification}.
\newblock \bibinfo{journal}{\emph{arXiv preprint arXiv:2106.03353}}
  (\bibinfo{year}{2021}).
\newblock


\bibitem[\protect\citeauthoryear{Ribeiro, Singh, and Guestrin}{Ribeiro
  et~al\mbox{.}}{2016}]%
        {Ribeiro2016}
\bibfield{author}{\bibinfo{person}{Marco~Tulio Ribeiro},
  \bibinfo{person}{Sameer Singh}, {and} \bibinfo{person}{Carlos Guestrin}.}
  \bibinfo{year}{2016}\natexlab{}.
\newblock \showarticletitle{"Why Should I Trust You?": Explaining the
  Predictions of Any Classifier}. In \bibinfo{booktitle}{\emph{Proceedings of
  the 22nd ACM SIGKDD International Conference on Knowledge Discovery and Data
  Mining}} (San Francisco, California, USA) \emph{(\bibinfo{series}{KDD '16})}.
  \bibinfo{publisher}{Association for Computing Machinery},
  \bibinfo{address}{New York, NY, USA}, \bibinfo{pages}{1135–1144}.
\newblock
\showISBNx{9781450342322}


\bibitem[\protect\citeauthoryear{Sakkas, Endres, Cosman, Weimer, and
  Jhala}{Sakkas et~al\mbox{.}}{2020}]%
        {Georgios}
\bibfield{author}{\bibinfo{person}{Georgios Sakkas}, \bibinfo{person}{Madeline
  Endres}, \bibinfo{person}{Benjamin Cosman}, \bibinfo{person}{Westley Weimer},
  {and} \bibinfo{person}{Ranjit Jhala}.} \bibinfo{year}{2020}\natexlab{}.
\newblock \showarticletitle{Type Error Feedback via Analytic Program Repair}.
  In \bibinfo{booktitle}{\emph{Proceedings of the 41st ACM SIGPLAN Conference
  on Programming Language Design and Implementation}} (London, UK)
  \emph{(\bibinfo{series}{PLDI '20})}. \bibinfo{publisher}{Association for
  Computing Machinery}, \bibinfo{address}{New York, NY, USA},
  \bibinfo{pages}{16–30}.
\newblock
\showISBNx{9781450376136}


\bibitem[\protect\citeauthoryear{Shrikumar, Greenside, and Kundaje}{Shrikumar
  et~al\mbox{.}}{2017a}]%
        {pmlrv70shrikumar17a}
\bibfield{author}{\bibinfo{person}{Avanti Shrikumar}, \bibinfo{person}{Peyton
  Greenside}, {and} \bibinfo{person}{Anshul Kundaje}.}
  \bibinfo{year}{2017}\natexlab{a}.
\newblock \showarticletitle{Learning Important Features Through Propagating
  Activation Differences}. In \bibinfo{booktitle}{\emph{Proceedings of the 34th
  International Conference on Machine Learning}}
  \emph{(\bibinfo{series}{Proceedings of Machine Learning Research})},
  \bibfield{editor}{\bibinfo{person}{Doina Precup} {and}
  \bibinfo{person}{Yee~Whye Teh}} (Eds.), Vol.~\bibinfo{volume}{70}.
  \bibinfo{publisher}{PMLR}, \bibinfo{pages}{3145--3153}.
\newblock


\bibitem[\protect\citeauthoryear{Shrikumar, Greenside, and Kundaje}{Shrikumar
  et~al\mbox{.}}{2017b}]%
        {Shrikumar6}
\bibfield{author}{\bibinfo{person}{Avanti Shrikumar}, \bibinfo{person}{Peyton
  Greenside}, {and} \bibinfo{person}{Anshul Kundaje}.}
  \bibinfo{year}{2017}\natexlab{b}.
\newblock \showarticletitle{Learning Important Features through Propagating
  Activation Differences}. In \bibinfo{booktitle}{\emph{Proceedings of the 34th
  International Conference on Machine Learning - Volume 70}} (Sydney, NSW,
  Australia) \emph{(\bibinfo{series}{ICML '17})}.
  \bibinfo{publisher}{JMLR.org}, \bibinfo{pages}{3145–3153}.
\newblock


\bibitem[\protect\citeauthoryear{Shrikumar, Greenside, Shcherbina, and
  Kundaje}{Shrikumar et~al\mbox{.}}{2016}]%
        {Shrikumar16}
\bibfield{author}{\bibinfo{person}{Avanti Shrikumar}, \bibinfo{person}{Peyton
  Greenside}, \bibinfo{person}{Anna Shcherbina}, {and} \bibinfo{person}{Anshul
  Kundaje}.} \bibinfo{year}{2016}\natexlab{}.
\newblock \showarticletitle{Not just a black box: Learning important features
  through propagating activation differences}.
\newblock \bibinfo{journal}{\emph{arXiv preprint arXiv:1605.01713}}
  (\bibinfo{year}{2016}).
\newblock


\bibitem[\protect\citeauthoryear{Si, Dai, Raghothaman, Naik, and Song}{Si
  et~al\mbox{.}}{2018}]%
        {Xujie3}
\bibfield{author}{\bibinfo{person}{Xujie Si}, \bibinfo{person}{Hanjun Dai},
  \bibinfo{person}{Mukund Raghothaman}, \bibinfo{person}{Mayur Naik}, {and}
  \bibinfo{person}{Le Song}.} \bibinfo{year}{2018}\natexlab{}.
\newblock \showarticletitle{Learning Loop Invariants for Program Verification}.
  In \bibinfo{booktitle}{\emph{Proceedings of the 32nd International Conference
  on Neural Information Processing Systems}} (Montr\'{e}al, Canada)
  \emph{(\bibinfo{series}{NIPS '18})}. \bibinfo{pages}{7762–7773}.
\newblock


\bibitem[\protect\citeauthoryear{Simonyan, Vedaldi, and Zisserman}{Simonyan
  et~al\mbox{.}}{2014}]%
        {Simonyan14deepinside}
\bibfield{author}{\bibinfo{person}{Karen Simonyan}, \bibinfo{person}{Andrea
  Vedaldi}, {and} \bibinfo{person}{Andrew Zisserman}.}
  \bibinfo{year}{2014}\natexlab{}.
\newblock \showarticletitle{Deep inside convolutional networks: Visualising
  image classification models and saliency maps}. In
  \bibinfo{booktitle}{\emph{In Workshop at International Conference on Learning
  Representations}}.
\newblock


\bibitem[\protect\citeauthoryear{Sundararajan, Taly, and Yan}{Sundararajan
  et~al\mbox{.}}{2017}]%
        {Sundararajan2017}
\bibfield{author}{\bibinfo{person}{Mukund Sundararajan}, \bibinfo{person}{Ankur
  Taly}, {and} \bibinfo{person}{Qiqi Yan}.} \bibinfo{year}{2017}\natexlab{}.
\newblock \showarticletitle{Axiomatic Attribution for Deep Networks}. In
  \bibinfo{booktitle}{\emph{Proceedings of the 34th International Conference on
  Machine Learning - Volume 70}} (Sydney, NSW, Australia)
  \emph{(\bibinfo{series}{ICML '17})}. \bibinfo{publisher}{JMLR.org},
  \bibinfo{pages}{3319–3328}.
\newblock


\bibitem[\protect\citeauthoryear{Wang, Singh, and Su}{Wang
  et~al\mbox{.}}{2018}]%
        {dy}
\bibfield{author}{\bibinfo{person}{Ke Wang}, \bibinfo{person}{Rishabh Singh},
  {and} \bibinfo{person}{Zhendong Su}.} \bibinfo{year}{2018}\natexlab{}.
\newblock \showarticletitle{Dynamic Neural Program Embedding for Program
  Repair}. In \bibinfo{booktitle}{\emph{International Conference on Learning
  Representations}} \emph{(\bibinfo{series}{ICLR '18})}.
\newblock


\bibitem[\protect\citeauthoryear{Wang and Su}{Wang and Su}{2020}]%
        {Wang101145}
\bibfield{author}{\bibinfo{person}{Ke Wang} {and} \bibinfo{person}{Zhendong
  Su}.} \bibinfo{year}{2020}\natexlab{}.
\newblock \showarticletitle{Blended, Precise Semantic Program Embeddings}. In
  \bibinfo{booktitle}{\emph{Proceedings of the 41st ACM SIGPLAN International
  Conference on Programming Language Design and Implementation}}
  \emph{(\bibinfo{series}{PLDI '20})}.
\newblock


\bibitem[\protect\citeauthoryear{Wang, Wang, Gao, and Wang}{Wang
  et~al\mbox{.}}{2020}]%
        {wang2020learning}
\bibfield{author}{\bibinfo{person}{Yu Wang}, \bibinfo{person}{Ke Wang},
  \bibinfo{person}{Fengjuan Gao}, {and} \bibinfo{person}{Linzhang Wang}.}
  \bibinfo{year}{2020}\natexlab{}.
\newblock \showarticletitle{Learning Semantic Program Embeddings with Graph
  Interval Neural Network}.
\newblock \bibinfo{journal}{\emph{Proceedings of the ACM on Programming
  Languages}} \bibinfo{volume}{4}, \bibinfo{number}{OOPSLA}, Article
  \bibinfo{articleno}{137} (\bibinfo{date}{Nov.} \bibinfo{year}{2020}),
  \bibinfo{numpages}{27}~pages.
\newblock


\bibitem[\protect\citeauthoryear{Wei, Goyal, Durrett, and Dillig}{Wei
  et~al\mbox{.}}{2019}]%
        {wei2019lambdanet}
\bibfield{author}{\bibinfo{person}{Jiayi Wei}, \bibinfo{person}{Maruth Goyal},
  \bibinfo{person}{Greg Durrett}, {and} \bibinfo{person}{Isil Dillig}.}
  \bibinfo{year}{2019}\natexlab{}.
\newblock \showarticletitle{LambdaNet: Probabilistic Type Inference using Graph
  Neural Networks}. In \bibinfo{booktitle}{\emph{International Conference on
  Learning Representations}} \emph{(\bibinfo{series}{ICLR '19})}.
\newblock


\bibitem[\protect\citeauthoryear{Yao, Ryan, Wong, Jana, and Gu}{Yao
  et~al\mbox{.}}{2020}]%
        {Jianan}
\bibfield{author}{\bibinfo{person}{Jianan Yao}, \bibinfo{person}{Gabriel Ryan},
  \bibinfo{person}{Justin Wong}, \bibinfo{person}{Suman Jana}, {and}
  \bibinfo{person}{Ronghui Gu}.} \bibinfo{year}{2020}\natexlab{}.
\newblock \showarticletitle{Learning Nonlinear Loop Invariants with Gated
  Continuous Logic Networks}. In \bibinfo{booktitle}{\emph{Proceedings of the
  41st ACM SIGPLAN Conference on Programming Language Design and
  Implementation}} (London, UK) \emph{(\bibinfo{series}{PLDI '20})}.
  \bibinfo{publisher}{Association for Computing Machinery},
  \bibinfo{address}{New York, NY, USA}, \bibinfo{pages}{106–120}.
\newblock
\showISBNx{9781450376136}


\bibitem[\protect\citeauthoryear{Zeiler, Taylor, and Fergus}{Zeiler
  et~al\mbox{.}}{2011}]%
        {zeiler2011adaptive}
\bibfield{author}{\bibinfo{person}{Matthew~D Zeiler}, \bibinfo{person}{Graham~W
  Taylor}, {and} \bibinfo{person}{Rob Fergus}.}
  \bibinfo{year}{2011}\natexlab{}.
\newblock \showarticletitle{Adaptive deconvolutional networks for mid and high
  level feature learning}. In \bibinfo{booktitle}{\emph{2011 International
  Conference on Computer Vision}}. IEEE, \bibinfo{pages}{2018--2025}.
\newblock


\bibitem[\protect\citeauthoryear{Zeller}{Zeller}{1999}]%
        {zeller1999yesterday}
\bibfield{author}{\bibinfo{person}{Andreas Zeller}.}
  \bibinfo{year}{1999}\natexlab{}.
\newblock \showarticletitle{Yesterday, my program worked. Today, it does not.
  Why?}
\newblock \bibinfo{journal}{\emph{ACM SIGSOFT Software engineering notes}}
  \bibinfo{volume}{24}, \bibinfo{number}{6} (\bibinfo{year}{1999}),
  \bibinfo{pages}{253--267}.
\newblock


\bibitem[\protect\citeauthoryear{Zeller and Hildebrandt}{Zeller and
  Hildebrandt}{2002}]%
        {zeller2002simplifying}
\bibfield{author}{\bibinfo{person}{Andreas Zeller} {and} \bibinfo{person}{Ralf
  Hildebrandt}.} \bibinfo{year}{2002}\natexlab{}.
\newblock \showarticletitle{Simplifying and isolating failure-inducing input}.
\newblock \bibinfo{journal}{\emph{IEEE Transactions on Software Engineering}}
  \bibinfo{volume}{28}, \bibinfo{number}{2} (\bibinfo{year}{2002}),
  \bibinfo{pages}{183--200}.
\newblock


\bibitem[\protect\citeauthoryear{Zintgraf, Cohen, Adel, and Welling}{Zintgraf
  et~al\mbox{.}}{2017}]%
        {Luisa2017}
\bibfield{author}{\bibinfo{person}{Luisa~M Zintgraf}, \bibinfo{person}{Taco~S
  Cohen}, \bibinfo{person}{Tameem Adel}, {and} \bibinfo{person}{Max Welling}.}
  \bibinfo{year}{2017}\natexlab{}.
\newblock \showarticletitle{Visualizing Deep Neural Network Decisions:
  Prediction Difference Analysis}. In \bibinfo{booktitle}{\emph{International
  Conference on Learning Representations}} \emph{(\bibinfo{series}{ICLR '17})}.
\newblock


\end{thebibliography}
	
\appendix









\newpage

\section{Definition of Subsequence}
A subsequence of <$a$> is a sequence <$b$> defined by $b_k=a_{n_k}$, where $n_1<n_2<...$ is an increasing sequence of indices~\cite{john1997mathematical}.

For example, if $a_n = 2n - 1$ and $n_k = k^2$, then $b_k = 2k^2 - 1$~\cite{john1997mathematical}.

\begin{table}[H]
	\centering
	\begin{tabular}{c|ccccccccc}
		$n$ & 1 & 2 & 3 & 4 & 5 & 6 & 7 & 8 & 9 \\
		$a_n$ & 1 & 3 & 5 & 7 & 9 & 11 & 13 & 15 & 17 \\
		$k$ & 1 &  &  & 2 &   &  &  &  & 3 \\
		$b_k$ & 1 &  &  & 7 &   &  &  &  & 17 \\
	\end{tabular}
\end{table}

\newpage

\section{The Existence of \Whet}\label{app-app:proof}

\begin{theorem}[Existence of \Whet]\label{app-the:exis}
	Given a prediction $L$ that $M$ makes for an input program $P$, the \whet $\tilde{P}$ that models use to predict the label of $P$ always exists.
\end{theorem}

\begin{proof}
	Assume otherwise, so that the \whet $\tilde{P}$ does not exist for $P$. \par
	Because the body of $P$, $B_P$, satisfies \textit{constituent}, \textit{sufficient}, and \textit{necessary} requirement in Definition~2.1. It has to be the \textit{minimum} requirement that $B_P$ violates, meaning, there exists a set of statements/expressions $P^\prime$ that also satisfies all but the \textit{minimum} requirement, and $\lvert (t_{n}^{P^\prime})_{n\in\mathbb{N}} \rvert < \lvert(t_{n}^{B_P})_{n\in\mathbb{N}}\rvert$. Since $P^\prime$ is not the \whet either, we can infer that $P^\prime$ is also not \textit{minimum}.\par
	Because the domain that contains all sets of statements/expressions whose token sequence is a subsequence of $P$'s is finite, and the size of the candidate programs will monotonically decreases (\ie,
	$\lvert (t_{n}^{P^\prime})_{n\in\mathbb{N}} \rvert < \lvert(t_{n}^{B_P})_{n\in\mathbb{N}}\rvert$,
	$\lvert (t_{n}^{P^{\prime\prime}})_{n\in\mathbb{N}} \rvert < \lvert(t_{n}^{P^{\prime}})_{n\in\mathbb{N}}\rvert$, and so on), there will be a global minimum set of statements/expressions $\tilde{P}$ that satisfies all requirements in Definition~\ref{def:keyfea}, which implies that $\tilde{P}$ is the \whet for $P$. This contradicts the assumption that $\tilde{P}$ does not exist for $P$.
\end{proof}
\vspace{-1pt}

\newpage

\section{ Examples of Delta Debugging}\label{app-app:dd}

The motivating example includes nineteen tokens. At the first step, we split the program into two partitions: the first partition ($\Delta_1$) contains the first nine tokens (\texttt{List, Object, mItems, =, retQueue, if, position, >, mItems}) and the second one ($\Delta_2$) has the last ten tokens (\texttt{size}, \texttt{return}, \texttt{mItems}, \texttt{add, position, genItem}, \texttt{notifyItemInserted}, \- \texttt{position}, \texttt{log}, \texttt{"add item"}). We proceed with $\Delta_2$ since it satisfies the \textit{sufficient} and \textit{necessary} requirement. Then, we split $\Delta_2$ into two partitions, and demonstrate the subsequent steps in Table~\ref{app-tab:dd}. 

\begin{table*}[htbp]
	\captionsetup{skip=1pt}	
	\caption{An example from sequence GNN model without the monotonicity. $\Delta_i$ denotes partitions and $\nabla_i$ is the complement of $\Delta_i$. For simplicity, we use tokens to represent programs that are tested against the \textit{sufficient} and \textit{necessary} requirement at each step. The last column shows the requirements that partitions do not satisfy, but \textit{pass} means the partitions satisfy both requirements.}
	\centering
	\adjustbox{max width=\linewidth}{
		\begin{tabular}{c|c|ccccccccc|c}
			\hline
			\multirow{2}{*}{\tabincell{c}{Step}} & \multirow{2}{*}{\tabincell{c}{Partition}} & \multicolumn{9}{c|}{Tokens} & \multirow{2}{*}{\tabincell{c}{Results}}  \\ 
			 & & size  &  return &  mItems &  add &  position &  genItem &  log &  "add item" &  position \\\hline		
			1 & $\Delta_3$  & $\checkmark$  & $\checkmark$  & $\checkmark$ & $\checkmark$ &  $\checkmark$ & & & & &  Pass  \\\cline{1-12}
			2 & $\Delta_4$ & & & & &  &$\checkmark$ & $\checkmark$  & $\checkmark$ & $\checkmark$   & \textit{Necessary} \\\hline\hline
			3 & $\Delta_5$  = $\nabla_6$ & $\checkmark$  & $\checkmark$ & $\checkmark$&  &   & & & & &  Both  \\\cline{1-12}
			4 & $\Delta_6$  = $\nabla_5$  & & &  & $\checkmark$ & $\checkmark$ &  &   &  &  & Both \\\hline\hline
			5 & $\Delta_7$  & $\checkmark$  &  &  &  &   & & & & &  Both  \\\cline{1-12}	
			6 & $\Delta_8$  &   & $\checkmark$ &  &  &   & & & & &  Both  \\\cline{1-12}	
			7 & $\Delta_9$  &   &  & $\checkmark$ &  &   & & & & &  Both  \\\cline{1-12}
			8 & $\Delta_{10}$  &   &   &   & $\checkmark$ & $\checkmark$  & & & & &  Both  \\\cline{1-12}
			9 & $\nabla_7$  &   &  $\checkmark$  &  $\checkmark$  & $\checkmark$ & $\checkmark$  & & & & &  \textit{Sufficient}  \\\cline{1-12}
			10 & $\nabla_8$  &  $\checkmark$   &  &  $\checkmark$  & $\checkmark$ & $\checkmark$  & & & & &  \textit{Sufficient}  \\\cline{1-12}
			11 & $\nabla_9$  &  $\checkmark$   &  $\checkmark$  &  & $\checkmark$ & $\checkmark$  & & & & &  \textit{Sufficient}  \\\cline{1-12}
			12 & $\nabla_{10}$  &   $\checkmark$  & $\checkmark$  & $\checkmark$ &  &  & & & & &  \textit{Sufficient}  \\\hline\hline
			13 & $\Delta_{11}$  & $\checkmark$  &  &  &  &   & & & & &  Both \\\cline{1-12}	
			14 & $\Delta_{12}$  &   & $\checkmark$ &  &  &   & & & & &  Both  \\\cline{1-12}	
			15 & $\Delta_{13}$  &   &  & $\checkmark$ &  &   & & & & &  Both \\\cline{1-12}
			16 & $\Delta_{14}$  &   &   &   & $\checkmark$ &  & & & & &  Both  \\\cline{1-12}
			17 & $\Delta_{15}$  &   &   &   &  & $\checkmark$  & & & & &  Both \\\cline{1-12}
			18 & $\nabla_{11}$  &   &  $\checkmark$  &  $\checkmark$  & $\checkmark$ & $\checkmark$  & & & & &  \textit{Sufficient}  \\\cline{1-12}
			19 & $\nabla_{12}$  &  $\checkmark$   &  &  $\checkmark$  & $\checkmark$ & $\checkmark$  & & & & &  \textit{Sufficient}  \\\cline{1-12}
			20 & $\nabla_{13}$  &  $\checkmark$   &  $\checkmark$  &  & $\checkmark$ & $\checkmark$  & & & & &  \textit{Sufficient}  \\\cline{1-12}
			21 & $\nabla_{14}$  &   $\checkmark$  & $\checkmark$  & $\checkmark$ &  & $\checkmark$ &  & & & &  \textit{Sufficient}  \\\cline{1-12}
			22 & $\nabla_{15}$  &   $\checkmark$  & $\checkmark$  & $\checkmark$ &   $\checkmark$ &   & & & & &  \textit{Sufficient}  \\\hline
		\end{tabular}
	}
	\label{app-tab:dd}
\end{table*}

\newpage

\section{An example of the subtraction operation in the \textit{Mutate} step} \label{app-app:step6fig}
Figure~\ref{app-fig:subtraction6} gives a detailed illustration for subtracting \texttt{mItems.add();} from the original program.

\begin{figure}[hb]
	\centering
	\setlength\fboxrule{0.4pt}
	
	\begin{subfigure}{0.47\textwidth}
		\centering
		\fbox{\includegraphics[width=\textwidth]{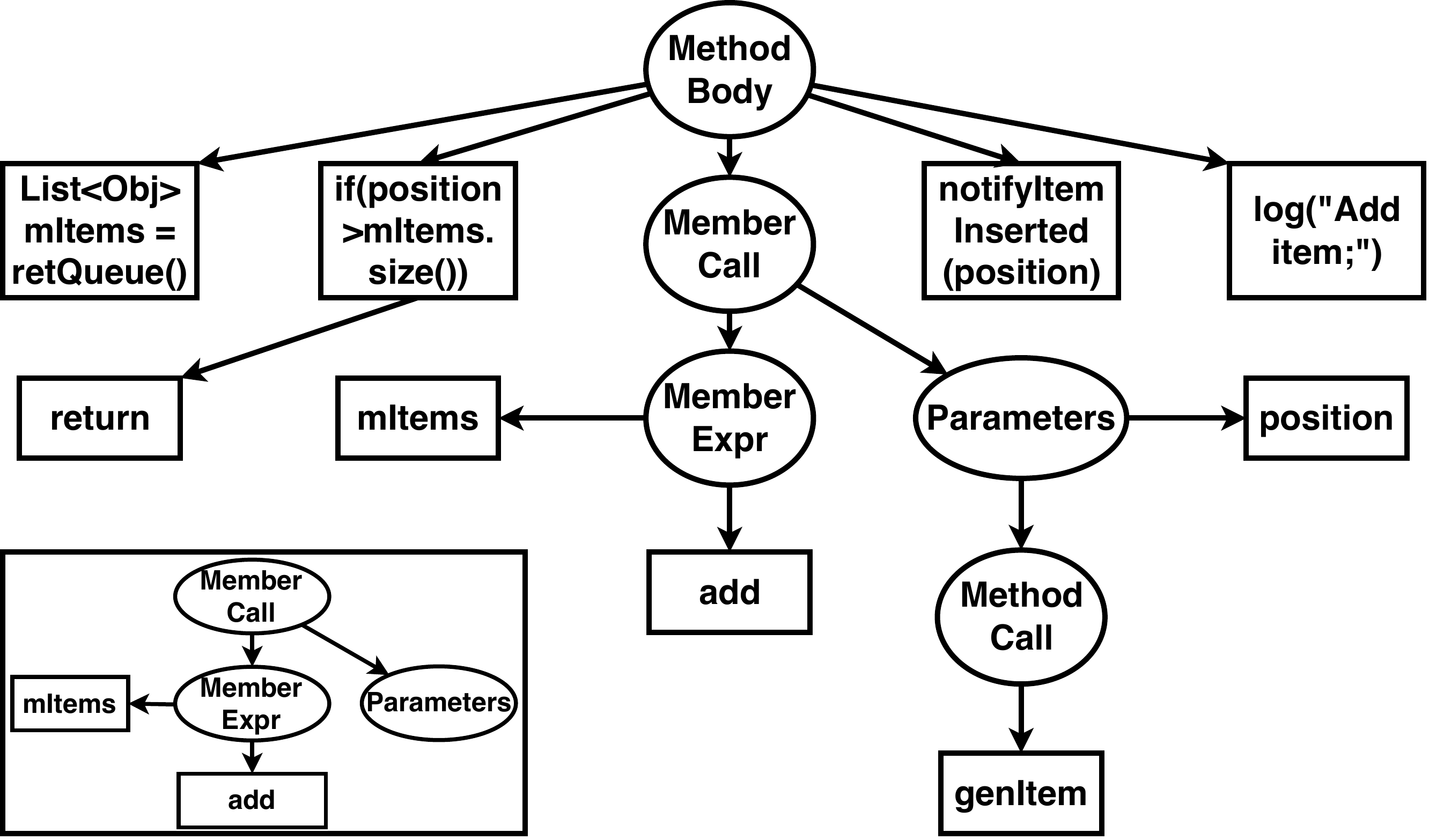}}
		
		\caption{}
		\label{app-fig:subtraction1}
	\end{subfigure}
	\,\,
	\begin{subfigure}{0.47\textwidth}
		\centering
		\fbox{\includegraphics[width=\textwidth]{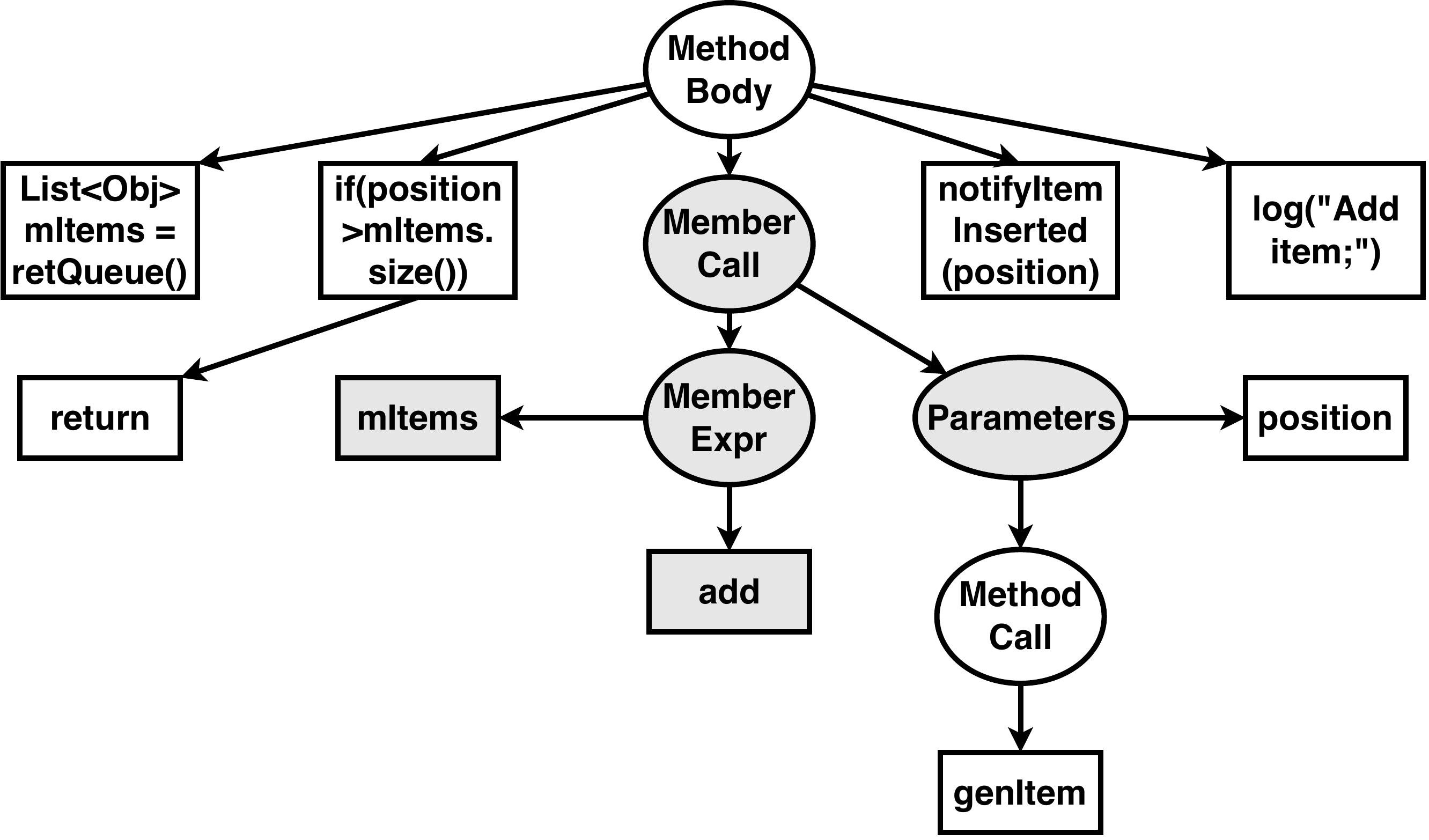}}
		
		\caption{}
		\label{app-fig:subtraction2}
	\end{subfigure}
	
	\begin{subfigure}{0.47\textwidth}
		\centering
		\fbox{\includegraphics[width=\textwidth]{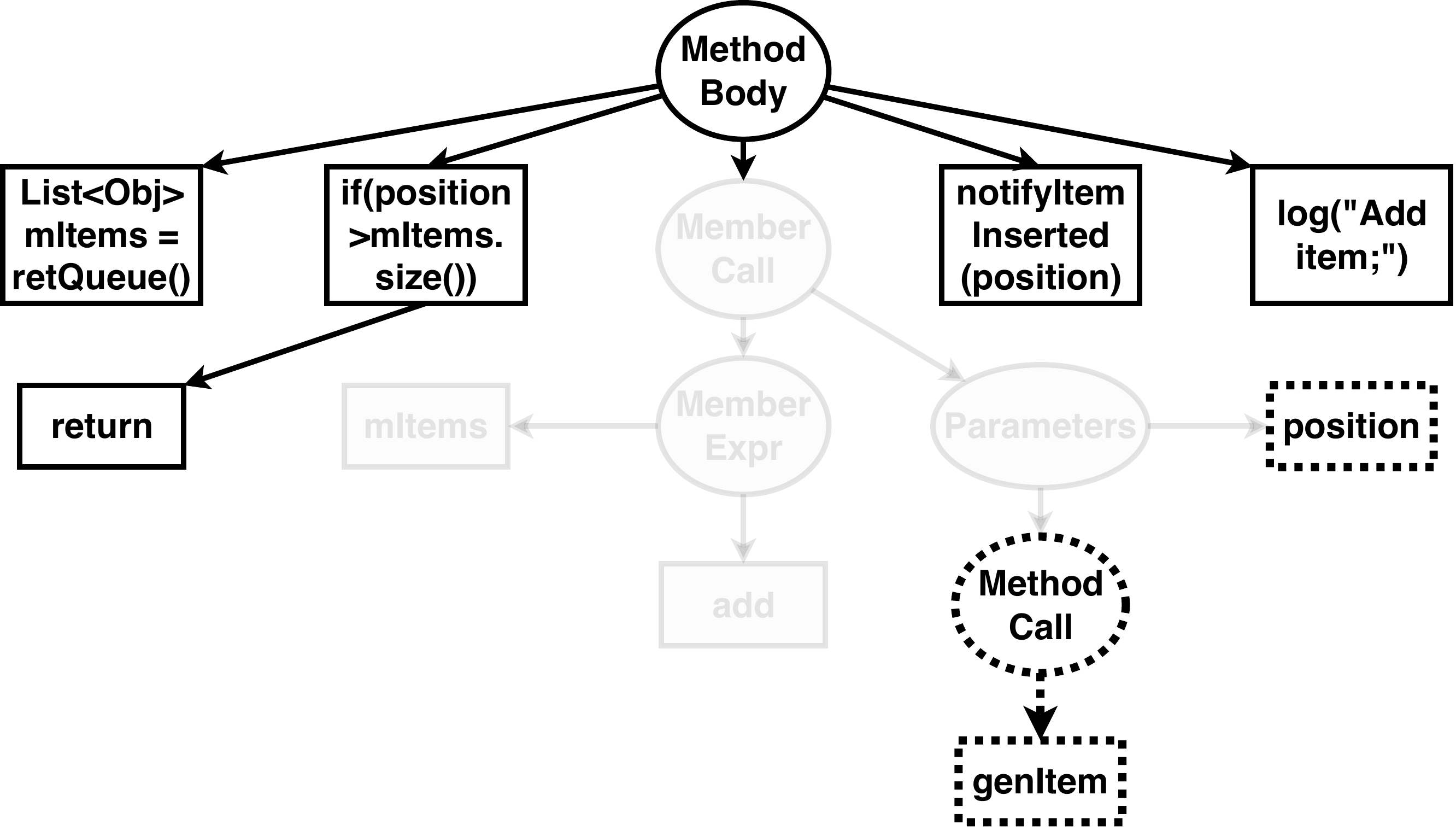}}
		
		\caption{}
		\label{app-fig:subtraction3}
	\end{subfigure}
	\,\,
	\begin{subfigure}{0.47\textwidth}
		\centering
		\fbox{\includegraphics[width=\textwidth]{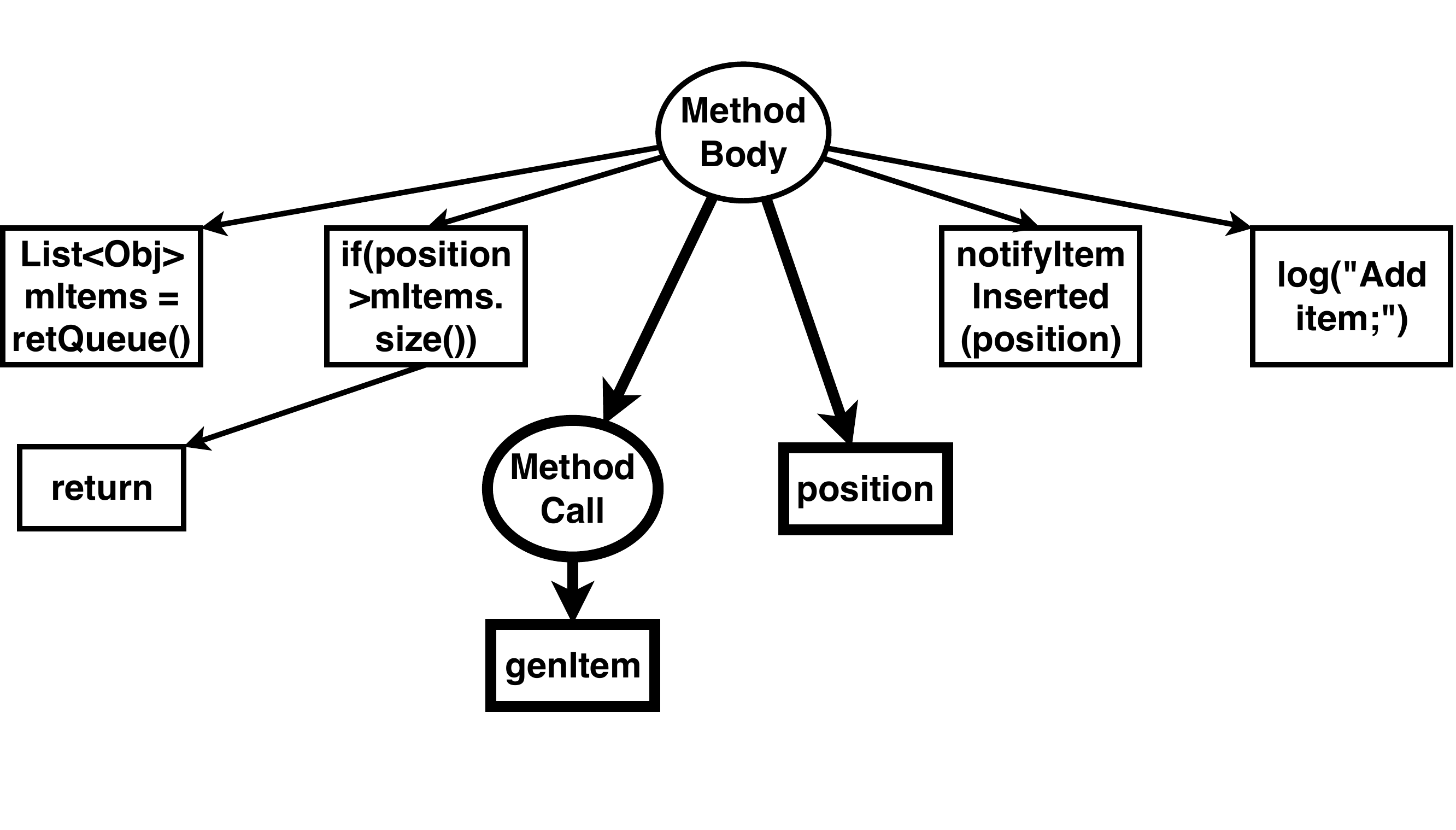}}
		
		\caption{}
		\label{app-fig:subtraction4}
	\end{subfigure} 
	
	\caption{A illustration of the four-step process of subtracting \texttt{mItems.add();} from our example program. (1) Figure~(\subref{app-fig:subtraction1}) presents the AST of the running example, and the AST of \texttt{mItems.add();} in the bottom left corner. For brevity, the ASTs are simplified. (2) Figure~(\subref{app-fig:subtraction2}) highlights the overlapping nodes between the two ASTs. (3) Figure~(\subref{fig:subtraction3}) emphasizes the resultant AST after the overlapping nodes are removed, as a result, \texttt{position} and \texttt{getItem()} become dangling nodes. (4) Figure~(\subref{app-fig:subtraction4}) connects \texttt{position} and \texttt{getItem()} to the body of the method.}
	\label{app-fig:subtraction6}
\end{figure}

\newpage
\section{The functionality of \texttt{DeleteNode} and \texttt{MutateNode}} \label{app-app:alg}

Algorithm~\ref{app-alg:keyfeasimp} gives a detailed illustration for \texttt{DeleteNode} and \texttt{MutateNode} functions. 
In  the \texttt{DeleteNode} function, we first try removing the node, which comes from its first parameter, then check whether the resultant program satisfies both the \textit{sufficient} and \textit{necessary} requirement (Line~\ref{app-line:checkdel}). If it satisfies, we update the current program to the new one whose node is deleted (Line~\ref{app-line:update1}). 
Similarly, in the  \texttt{MutateNode} function, we will mutate the node into one with an out-of-vocabulary value (Line~\ref{app-line:wrap2}) and update the program if the resultant program satisfies the two requirements (Lines~\ref{app-line:checkmut} and~\ref{app-line:checkmut2}). 
Both \texttt{DeleteNode} and \texttt{MutateNode} functions invoke   the \texttt{VerifyWheat} function, which constructs two programs for checking against the \textit{sufficient} and \textit{necessary} requirements, respectively. The function returns \textit{True} only when the resultant program satisfy both requirements, otherwise it returns \textit{False}.

\begin{algorithm}[!h]
	\algnewcommand\Pd{prediction\xspace}
	\algnewcommand\Ml{model\xspace}
	\algnewcommand\OMT{program\xspace}
	\algnewcommand\K{i\xspace}	
	\algnewcommand\Mn{met.name\xspace}		
	\algnewcommand\Me{met\xspace}
	\algnewcommand\Mt{met\_ast\xspace}
	\algnewcommand\Md{fragments\xspace}
	\algnewcommand\Ma{min\_ast\xspace}
	\algnewcommand\Fe{features\xspace}
	\algnewcommand\Ne{node\xspace}
	\algnewcommand\Tn{target\_node\xspace}
	\algnewcommand\Fg{flag\xspace}
	\algnewcommand\Nt{node.children\xspace}
	\algnewcommand\Nz{node.children.size\xspace}	
	\algnewcommand\Ni{node.children[i]\xspace}	
	\algnewcommand\Na{node.children[node.value[1]]\xspace}	
	\algnewcommand\Nv{node.value\xspace}
	\algnewcommand\No{node.value[0]\xspace}	
	\algnewcommand\Nn{node.value[1]\xspace}	
	\algnewcommand\Ve{value\xspace}
	\algnewcommand\Pl{oov\xspace}		
	\algnewcommand\Rt{root\xspace}
	\algnewcommand\Ry{root.copy\xspace}	
	\algnewcommand\Ss{s\xspace}
	\algnewcommand\Ch{child\xspace}
	\algnewcommand\Cs{replace\xspace}
	\algnewcommand\Rv{replace.value\xspace}
	\algnewcommand\Cv{child.value\xspace}
	\algnewcommand\Cz{child.value[0]\xspace}	
	\algnewcommand\Co{count\xspace}
	\algnewcommand\Ps{pos\xspace}	
	\algnewcommand\Vo{oov\xspace}	
	\algnewcommand\Mr{min\_root\xspace}	
	\algnewcommand\NB{node\_bck\xspace}	
	\algnewcommand\FMN{MutateNode\xspace}	
	\algnewcommand\FDN{DeleteNode\xspace}	
	\algnewcommand\FIN{InheritNode\xspace}	
	\algnewcommand\NewV{new\_value\xspace}	
	\algnewcommand\Ncv{node.children[0].value\xspace}	
	\algnewcommand\IV{IsValid\xspace}	
	\algnewcommand\CR{VerifyWheat\xspace}	
	\algblockdefx[Foreach]{Foreach}{EndForeach}[1]{\textbf{foreach} #1 \textbf{do}}{\textbf{end foreach}}
	
	\caption{Delete and mutate AST nodes. }
	\label{app-alg:keyfeasimp}
	\begin{algorithmic}[1]
		\item[]
		
				\Function{\FDN}{$\Ne, \Rt, \OMT, \Ml$} \label{app-line:reduce}
				\State new\_root\_del $\gets$ \Rt.\textit{Delete}\;\!(\Ne) \label{app-line:wrap1} \Comment{remove {node} }
				\If{\Call{\CR}{new\_root\_del , \Rt, \OMT, \Ml}\label{app-line:checkdel}}  
				\State  \textit{Update}\;\!(\Rt, new\_root\_del)\label{app-line:update1} \Comment{update to the new program}
					\EndIf
				\EndFunction

		\item[]				
				\Function{\FMN}{$\Ne, \Rt, \OMT, \Ml$} \label{app-line:mutate}
              \State  new\_root\_mut $\gets$  \Rt.\textit{Replace}\;\!(\Ne, \Pl)\label{app-line:wrap2} \Comment{replace {node} with oov}
				\If{\Call{\CR}{new\_root\_mut , \Rt, \OMT, \Ml}\label{app-line:checkmut}}
				\State  \textit{Update}\;\!(\Rt, new\_root\_mut)\label{app-line:checkmut2} 
				\EndIf  \label{app-alg:mut:count1end}
				
				\EndFunction
		
				\item[]
				
				\Function{\CR}{$\Ne, \Rt, \OMT, \Ml$} \label{app-line:checkrequirements}
		
				\State program\_suff $\gets$ \textit{Serialize}\;\!(\Ne)\Comment{checking against the \textit{sufficient} requirement}
			
				\If{\Ml.\textit{Predict}\;\!(program\_suff) != \Ml.\textit{Predict}\;\!(\OMT)}
				\State	\Return{False}		\Comment{does not satisfy the \textit{sufficient} requirement}
				\EndIf 
				
				\State program\_necc $\gets$ \textit{Subtract}\;\!(\Rt, \Ne)\Comment{checking against the \textit{necessary} requirement}
				
				\If{\Ml.\textit{Predict}\;\!(program\_necc) ==\Ml.\textit{Predict}\;\!(\OMT)}
				\State	\Return{False}		\Comment{does not satisfy the \textit{necessary} requirement}
				\EndIf 
				\State	\Return{True}		
		
				\EndFunction
	\end{algorithmic}
\end{algorithm}

\newpage

\section{Performance of Re-implemented Models }
We have re-implemented code2vec, Sequence GNN, GGNN and
CodeBERT.
Table~\ref{app-tab:com1},~\ref{app-tab:com2}, and~\ref{app-tab:com3} show the performance of all re-implemented models is either comparable or superior to the originals. 

\begin{table}[hb]
		\begin{minipage}{1\textwidth}
	\captionsetup{skip=1pt}	
\caption{Compare reimplementations (\textbf{bolded}) to originals for code2vec and Seq-GNN. }
\centering
		\adjustbox{max width=\linewidth}{

			\begin{tabular}{c|ccc|ccc|ccc}
				\hline
				\multirow{2}{*}{\tabincell{c}{Models}} & \multicolumn{3}{c|}{Java-small} & \multicolumn{3}{c|}{Java-med} & \multicolumn{3}{c}{Java-large}\\
			\cline{2-10}
			& Precision & Recall & F1 & Precision & Recall & F1 & Precision & Recall & F1  \\\hline
			code2vec    & 18.51 & 18.74 & 18.62 & 38.12 & 28.31 &32.49 & 48.15 & 38.40 & 42.73    \\\hline		
			\textbf{code2vec}& \textbf{19.23} & \textbf{17.72} & \textbf{18.44} & \textbf{40.32} & \textbf{28.89} &\textbf{33.66} & \textbf{48.90} & \textbf{37.26} & \textbf{42.29} \\\hline\hline

			Seq-GNN    & \textemdash & \textemdash & 51.4 & \textemdash & \textemdash &\textemdash & \textemdash & \textemdash & \textemdash    \\\hline
			\textbf{Seq-GNN}& \textbf{49.94} & \textbf{47.35} & \textbf{48.61} & \textbf{58.46} & \textbf{45.73} &\textbf{51.32} & \textbf{61.82} & \textbf{50.32} & \textbf{55.48} \\\hline\hline
			\end{tabular}
		}
       \label{app-tab:com1}
		\end{minipage}%
	
		\begin{minipage}{0.48\textwidth}
			\raggedleft
			
			\caption{Compare reimplementations (\textbf{bolded}) to originals for GGNN.}
\begin{tabular}{>{\centering\arraybackslash}p{2cm} | >{\centering\arraybackslash}p{4cm} }
     			\hline
     			 \multirow{2}{*}{Models}  &  C\# Datasets \\\cline{2-2}		
				&  Accuracy  \\\hline
			GGNN &  78.2     \\\hline
			\textbf{GGNN}& \textbf{78.0}    \\\hline
			\end{tabular} 
       \label{app-tab:com2}
   		\end{minipage}
	\,\,
		\begin{minipage}{.48\textwidth}
			\raggedright
	\caption{Compare reimplementations (\textbf{bolded}) to originals for CodeBERT.}
\begin{tabular}{>{\centering\arraybackslash}p{2cm} | >{\centering\arraybackslash}p{4cm} }
		\hline
		\multirow{2}{*}{Models}  & CodeSearchNet \\\cline{2-2}		
		& Smoothed BLEU score  \\\hline
		CodeBERT & 17.65  \\\hline
\textbf{CodeBERT} & \textbf{17.66} \\\hline
	\end{tabular} 
       \label{app-tab:com3}
   \end{minipage}
\end{table}

\newpage

\section{Addition Edges Used in Fernandes \etal} 
\label{app-app:edges}

Below we give the list of edges~\citet{fernandes2018structured} incorporate into ASTs.

\begin{itemize}
    \item \textbf{NextToken} connects each terminal node (syntax token) to its successor.
    \item \textbf{LastRead} connects a terminal node of a variable to all elements of the set of terminal nodes at which the variable could have been read last.
     \item \textbf{LastWrite} connects a terminal node of a variable to all elements of the set of syntax tokens at which the variable was could have been last written to.
     \item \textbf{ComputedFrom} connects a terminal node of a variable $v$ to all variable tokens occurring in $expr$ when $expr$ is assigned to $v$. 
     \item \textbf{LastLexicalUse} chains all uses of the same variable.
     \item \textbf{ReturnsTo} connects \texttt{return} tokens to the method declaration.
     \item \textbf{FormalArgName} connects arguments in method calls to the formal parameters that they are matched to.
     \item \textbf{GuardedBy} connects every token corresponding to a variable (in the true branch of a \texttt{if} statement) to the enclosing guard expression that uses the variable.
     \item \textbf{GuardedByNegation} connects every token corresponding to a variable (in the false branch of a \texttt{if} statement) to the enclosing guard expression that uses the variable.    
\end{itemize}

We exclude \textbf{NextToken} and \textbf{ReturnsTo} since they do not represent any semantic properties.

\newpage

\section{Performance of the re-trained models on the augmented datasets.}
We have re-trained code2vec, Sequence GNN, GGNN and
CodeBERT on the manually labeled datasets.
Table~\ref{app-tab:aug1}, \ref{app-tab:aug2}, and \ref{app-tab:aug3} show the performance of all re-trained models is negligibly higher than  original models.

\begin{table}[hb]
	\begin{minipage}{1\textwidth}
		\captionsetup{skip=1pt}	
		\caption{Compare code2vec and Seq-GNN before and after the retraining (\textbf{bolded} numbers denotes the results of retrained models). }
		\centering
		\adjustbox{max width=\linewidth}{
			
			\begin{tabular}{c|ccc|ccc}
				\hline
				\multirow{2}{*}{\tabincell{c}{Models}} & \multicolumn{3}{c|}{Java-small} & \multicolumn{3}{c}{Java-med} \\
				\cline{2-7}
				& Precision & Recall & F1 & Precision & Recall & F1  \\\hline
				code2vec& 19.23 & 17.72 & 18.44 & 40.32 & 28.89 &33.66  \\\hline
				\textbf{code2vec}& \textbf{19.33} & \textbf{18.01} & \textbf{18.74} & \textbf{40.38} & \textbf{28.97} &\textbf{33.96}  \\\hline\hline
				
				Seq-GNN& 49.94 & 47.35 & 48.61 & 58.46 & 45.73 &51.32  \\\hline
				\textbf{Seq-GNN}& \textbf{50.12} & \textbf{48.39} & \textbf{48.62} & \textbf{58.51} & \textbf{46.09} &\textbf{52.00}  \\\hline\hline
			\end{tabular}
		}
		\label{app-tab:aug1}
	\end{minipage}%
	
	\begin{minipage}{0.48\textwidth}
		\raggedleft
		
		\caption{Compare GGNN before and after the retraining (\textbf{bolded} numbers denotes the results of the retrained model).}
		\begin{tabular}{>{\centering\arraybackslash}p{2cm} | >{\centering\arraybackslash}p{4cm} }
			\hline
			\multirow{2}{*}{Models}  &  C\# Datasets \\\cline{2-2}		
			&  Accuracy  \\\hline
			GGNN &  78.0     \\\hline
			\textbf{GGNN}& \textbf{78.8}    \\\hline
		\end{tabular} 
		\label{app-tab:aug2}
	\end{minipage}
	\,\,
	\begin{minipage}{.48\textwidth}
		\raggedright
		\caption{Compare CodeBERT before and after the retraining (\textbf{bolded} numbers denotes the results of the retrained model).}
		\begin{tabular}{>{\centering\arraybackslash}p{2cm} | >{\centering\arraybackslash}p{4cm} }
			\hline
			\multirow{2}{*}{Models}  & CodeSearchNet \\\cline{2-2}		
			& Smoothed BLEU score  \\\hline
			CodeBERT & 17.66  \\\hline
			\textbf{CodeBERT} & \textbf{17.91} \\\hline
		\end{tabular} 
		\label{app-tab:aug3}
	\end{minipage}
\end{table}

\end{document}